\documentclass[twoside,11pt]{article}

\usepackage{graphicx,amsmath,amsfonts,amscd,amssymb,bm,url,color}
\usepackage[T1]{fontenc}
\usepackage[small,bf]{caption}
\usepackage{subcaption}
\usepackage{bbm}
\usepackage{algorithm}
\usepackage{algorithmicx}
\usepackage{algpseudocode}
\usepackage{comment}
\usepackage[toc, page]{appendix}

\usepackage[abbrvbib, preprint]{jmlr2e}
\usepackage{lastpage}

\DeclareMathOperator{\st}{subject\; to}
\DeclareMathOperator{\trace}{tr}
\DeclareMathOperator{\grad}{grad}

\renewcommand{\mathbf}{\boldsymbol}

\newcommand{\mb}{\mathbf}
\newcommand{\mc}{\mathcal}

\newcommand{\bb}{\mathbb}
\newcommand{\msf}{\mathsf}

\newcommand{\eps}{\varepsilon}

\newcommand{\indicator}[1]{\mathbbm 1_{#1}}

\newcommand{\norm}[2]{\left\| #1 \right\|_{#2}}
\newcommand{\abs}[1]{\left| #1 \right|}

\newcommand{\innerprod}[2]{\left\langle #1,  #2 \right\rangle}

\newtheorem{claim}[theorem]{Claim}

\ShortHeadings{Complete Dictionary Learning via $\ell^4$-Norm Maximization over the Orthogonal Group}{Zhai, Yang, Liao, Wright, and Ma}
\firstpageno{1}

\begin{document}

\jmlrheading{21}{2020}{1-\pageref{LastPage}}{9/19; Revised
8/20}{8/20}{19-755}{Yuexiang Zhai, Zitong Yang, Zhenyu Liao, John Wright, and Yi Ma}
\ShortHeadings{Complete Dictionary Learning over the Orthogonal Group}{Zhai, Yang, Liao, Wright, and Ma}

\title{Complete Dictionary Learning via $\ell^4$-Norm Maximization over the Orthogonal Group}

\author{\name Yuexiang Zhai$^{\dagger}$ \email ysz@berkeley.edu
       \AND
       \name Zitong Yang$^{\dagger}$ \email zitong@berkeley.edu
       \AND
       \name Zhenyu Liao$^{\ddagger}$ \email liaozhenyu2004@gmail.com
       \AND
       \name John Wright$^{\diamond}$ \email jw2966@columbia.edu
       \AND
       \name Yi Ma$^{\dagger}$ \email yima@eecs.berkeley.edu
       \AND
       \addr $^\dagger$Department of Electrical Engineering and Computer Science\\
       University of California, Berkeley, CA 94720-1770\\
       \addr $^\ddagger$Kuaishou Technology.\\
       \addr $^\diamond$Department of Electrical Engineering\\ 
       Columbia University, New York, NY, 10027
       }

\editor{Julien Mairal}

\maketitle

\begin{abstract}
This paper considers the fundamental problem of learning a complete (orthogonal) dictionary from samples of sparsely generated signals. Most existing methods solve the dictionary (and sparse representations) based on heuristic algorithms, usually without theoretical guarantees for either optimality or complexity. The recent $\ell^1$-minimization based methods do provide such guarantees but the associated algorithms recover the dictionary one column at a time. In this work, we propose a new formulation that {\em maximizes} the $\ell^4$-norm over the orthogonal group, to learn the entire dictionary. We prove that under a random data model, with nearly minimum sample complexity, the global optima of the $\ell^4$-norm are very close to signed permutations of the ground truth. Inspired by this observation, we give a conceptually simple and yet effective algorithm based on ``{\em matching, stretching, and projection}'' (MSP). The algorithm provably converges locally and cost per iteration is merely an SVD. In addition to strong theoretical guarantees, experiments show that the new algorithm is significantly more efficient and effective than existing methods, including KSVD and $\ell^1$-based methods. Preliminary experimental results on mixed real imagery data clearly demonstrate advantages of so learned dictionary over classic PCA bases.
\end{abstract}

\begin{keywords}
  sparse dictionary learning, $\ell^4$-norm maximization, orthogonal group, measure concentration, fixed point algorithm
\end{keywords}

\section{Introduction and Overview}
\subsection{Motivation}
One of the most fundamental problems in signal processing or data analysis is that given an observed signal $\mb y$, either continuous or discrete, we would like to find a transform $\mathcal{F}$ such that after applying the transform, the resulting signal $\mb x = \mathcal{F}(\mb y)$ becomes much more compact, sparse, or compressible. We believe such a compact representation $\mb x$ can help reveal intrinsic structures of the observed signal $\mb y$ and is also more amenable to storage, processing, and transmission. 

For computational purposes, the transforms considered are typically orthogonal linear transforms so that both $\mathcal{F}$ and $\mathcal{F}^{-1}$ are easy to represent and compute: In this case, we have $\mb y = \mb D \mb x$ or $\mb x= \mb D^* \mb y$ for some orthogonal matrix (or linear operator) $\mb D$. Examples include the classical Fourier transform \citep{oppenheim1999discrete,vetterli2014foundations} or various wavelets \citep{vetterli1995wavelets}. Conventionally, the best transform to use is typically by ``{\em design}'': By assuming the signals of interest have certain physical or mathematical properties (e.g. band-limited, piece-wise smooth, or scale-invariant), one may derive or design the optimal transforms associated with different classes of functions or signals. 
This classical approach has found its deep mathematical roots in functional analysis \citep{kreyszig1978introductory} and harmonic analysis \citep{stanton1981l4,katznelson2004introduction} and has seen great empirical successes in digital signal processing \citep{oppenheim1999discrete}.

Nevertheless, in the modern big data era, both science and engineering are inundated with tremendous high-dimensional data, such as images, audios, languages, and genetics etc. Many of such data may or may not belong to the classes of functions or signals for which we know the optimal transforms. Assumptions about their intrinsic (low-dim) structures are not clear enough for us to derive any new transform either. Therefore, we are compelled to change our practice and ask whether we can ``{\em learn}'' an optimal transform (if exists) directly from the observed data. That is, if we observe many, say $p$, samples $\mb y_i \in \bb R^n$ from a model:
\begin{equation*}
    \mb y_i = \mb D_o \mb x_i, \quad i = 1, \ldots, p    
\end{equation*}
where $\mb x_i \in\bb R^m$ is presumably much more compact or sparse, can we both learn $\mb D_o$ and recover the associated $\mb x_i$ from such $\mb y_i$? Here $\mb D_o$ is called the ground truth ``dictionary'' and the problem is known as ``Dictionary Learning.'' Dictionary learning is a fundamental problem in data science since finding a sparse representation of data appears in different applications such as computational neural science \citep{olshausen1996emergence,olshausen1997sparse}, machine learning \citep{argyriou2008convex,ranzato-nips-07}, and computer vision \citep{elad2006image,yang2010image,mairal2014sparse}. Note that here we know neither the dictionary $\mb D_o$ nor the hidden state $\mb x$. So in machine learning, dictionary learning belongs to the category of {\em unsupervised learning} problems, and a fundamental one that is.

In this work, we consider the problem of {\em learning a complete dictionary}\footnote{In this work, we only consider the case the dictionary is complete. The more general setting in which the dictionary $\mb D_o$ is over-complete, $m > n$, is beyond the scope of this paper.} from sparsely generated sample signals. More precisely, an $n$-dimensional sample $\mb y \in \bb R^n$ is assumed to be a sparse superposition of columns of a non-singular complete dictionary $\mb D_o \in \bb R^{n\times n}$:
$\mb y = \mb D_o \mb x,$
where $\mb x \in \bb R^n$ is a sparse (coefficient) vector. A typical statistical model for the sparse coefficient is that entries of $\mb x$ are i.i.d. Bernoulli-Gaussian $\{x_i\} \sim_{iid}\text{BG}(\theta)$\footnote{I.e., each entry $x_i$ is a product of independent Bernoulli and standard normal random variables: $x_{i}=\Omega_{i} V_{i}$, where $\Omega_{i}\sim_{iid} \text{Ber}(\theta)$ and $V_{i}\sim_{iid} \mc N(0,1)$.} \citep{spielman2012exact,sun2015complete,bai2018subgradient}. 

Suppose we are given a collection of sample signals $\mb Y=[\mb y_1,\mb y_2,\dots,\mb y_p]\in\bb R^{n\times p}$, each of which is generated as $\mb y_i = \mb D_o \mb x_i$ for a nonsingular matrix $\mb D_o$. Write $\mb X_o = [ \mb x_1, \mb x_2, \dots, \mb x_p] \in \bb R^{n \times p}$. In this notation, we have:
\begin{equation}
    \label{eq:Y=DXAssumption}
    \mb Y = \mb D_o \mb X_o.   
\end{equation}
{\em Dictionary learning} is the problem of recovering both the dictionary $\mb D_o$ and the sparse coefficients $\mb X_o$, given only the samples $\mb Y$. Equivalently, we wish to factorize $\mb Y$ as $\mb Y = \mb D \mb X$, where $\mb D$ is an estimate of the true dictionary $\mb D_o$ and $\mb X$ is the sparsest possible. 

Under the Bernoulli-Gaussian assumption, the problem of learning an arbitrary complete dictionary can be reduced to that of learning an {\em orthogonal} dictionary: As shown in \cite{sun2015complete}, when the objective is smooth, the problem can be converted to the orthogonal case through a preconditioning: \begin{equation*}{\mb Y} \leftarrow \left(\frac{1}{p\theta} \mb Y\mb Y^*\right)^{-\frac{1}{2}} \mb Y.
\end{equation*}
So without loss of generality, we can assume that $\mb D_o$ is an orthogonal matrix: $\mb D_o \in \msf O(n; \bb R)$.

Because $\mb Y$ is sparsely generated, the optimal estimate $\mb D_\star$ should make the associated coefficients $\mb X_\star$ maximally sparse. In other words, $\ell^0$-norm, the number of non-zero entries, of $\mb X_\star$ should be as small as possible, therefore, one may formulate the following optimization program to find $\mb X_\star$:
\begin{equation}
    \label{eq:L0Min}
    \min_{\mb X,\mb D}\norm{\mb X}{0},\quad \st \quad \mb {Y} = \mb {D}\mb {X}, \;\mb D \in \msf O(n; \bb R).
\end{equation}
Under fairly mild conditions, the global minimizer of the $\ell^0$-norm recovers the true dictionary $\mb D_o$ \cite{spielman2012exact}.  However, global minimization of the $\ell^0$-norm is a challenging NP-hard problem \citep{donoho2006most,candes2005decoding,natarajan1995sparse}. Traditionally, one resorts to local heuristics such as orthogonal matching pursuit, as in the KSVD algorithm \citep{aharon2006k,rubinstein2010dictionaries}.\footnote{\citep{peyre2010best,ravishankar2015l0} also provide algorithms for learning orthogonal sparsifying transformations.} This approach has been widely practiced but does not give any strong guarantees for optimality of the algorithm nor correctness of the solution, as we will see through experiments compared with our method.

Since $\ell^1$-norm minimization promotes sparsity \citep{candes2014mathematics} and the $\ell^1$-norm is convex and continuous, one may reformulate dictionary learning as an $\ell^1$-minimization problem:
\begin{equation}
    \label{eq:L1Min}
    \min_{\mb X,\mb D}\norm{\mb X}{1},\quad \st \quad \mb {Y} = \mb {D}\mb {X}, \;\mb D \in \msf O(n; \bb R).
\end{equation}
This reformulation \eqref{eq:L1Min} remains a nonsmooth optimization with nonconvex constraints, which is in general still NP-hard \citep{murty1987some}. Nevertheless, many heuristic algorithms  \citep{mairal2008discriminative,mairal2009supervised,mairal2012task} have attempted to reformulate optimization  \eqref{eq:L1Min} as an unconstrained optimization problem with $\ell^1$-regularization by removing the constraint $\mb Y = \mb {DX}$.

Although the $\ell^0$- or $\ell^1$-minimization has been widely practiced in dictionary learning, rigorous justification for optimality and correctness is only provided recently.
\cite{geng2014local} is the first to show the local optimality of the $\ell^1$-minimization. \cite{spielman2012exact} further proves that a complete (square and invertible) $\mb D$ can be recovered from $\mb Y$, when each column of $\mb X$ contains no more than $O(\sqrt n)$ of nonzero entries. Subsequent works \citep{agarwal2013exact,agarwal2014learning,arora2014new,arora2015simple} have provided provable algorithms for overcomplete dictionary learning, under the assumption that each column of $X$ has $\tilde{O}(\sqrt n)$ nonzero entries.\footnote{$\tilde{O}$ suppresses logarithm factors.}

In this work, we mainly focus on {\it complete dictionary learning}, which implies $\text{row}(\mb Y) = \text{row}(\mb X)$. \cite{spielman2012exact} has proposed to find the sparsest vector $\mb d^*\mb Y$ in $\text{row}(\mb Y)$ one by one via solving the following optimization: 
\begin{equation}
    \label{eq:L1MinByColumn}
    \min_{\mb d\in\bb R^n} \norm{\mb d^*\mb Y}{1},\quad \st \quad \mb d\neq \mb 0
\end{equation}
$n$ times, instead of solving the hard nonconvex optimization \eqref{eq:L1Min} directly. \eqref{eq:L1MinByColumn} is easier to solve, since it can be further reduced to a linear programming. \cite{sun2015complete} has proposed a new formulation with spherical constraint that finds one column of the dictionary via solving:
\begin{equation}
    \label{eq:L1MinSphereByColumn}
    \min_{\mb d\in\bb R^n} \norm{\mb d^*\mb Y}{1},\quad \st \quad \norm{\mb d}{2}=1.
\end{equation}

For escaping saddle points, \cite{sun2015complete} has proposed a provably correct second-order Riemannian Trust Region method \citep{absil2009optimization} to solve \eqref{eq:L1MinSphereByColumn} and improved the sparsity level of $\mb X$ to constant.\footnote{Each column of $\mb X$ can contain $O(n)$ non zero entries.} However, as addressed by \cite{gilboa2018efficient}, the computational complexity of a second-order algorithm is high, not to mention one needs to solve \eqref{eq:L1MinSphereByColumn} $n$ times! To mitigate the computation complexity, \cite{gilboa2018efficient} suggests that the first-order gradient descent method with random initialization has the same performance as a second-order one, and \cite{bai2018subgradient} shows that a randomly initialized first-order projected subgradient descent is able to solve \eqref{eq:L1MinSphereByColumn}. But these approaches fail to overcome the main cause for the high complexity -- one needs to break down the 
complete dictionary learning problem \eqref{eq:L1Min} into solving $n$ optimization programs like \eqref{eq:L1MinByColumn} (or \eqref{eq:L1MinSphereByColumn}).

Besides the $\ell^1$-minimization based dictionary learning framework, works based on the sum-of-square (SoS) SDP hierarchy \citep{Barak-2015,ma2016polynomial,schramm2017fast} also provide guarantee for exact recovery of (overcomplete) dictionary learning problem in polynomial time, under some specific statistical assumption of the data model. But one still needs to solve the SoS based SDP Problem $n$ times to recover a dictionary with $n$ components,\footnote{See Theorem 1.1 in \cite{schramm2017fast}} let alone the high computational complexity for solving a high-dimensional SDP programming each time \citep{qu2014finding,bai2018subgradient}.

\subsection{Our Approach and Connection to Prior Works}
In this paper, we show that one can actually efficiently learn a complete (orthogonal) dictionary holistically via solving one $\ell^4$-norm optimization over the entire orthogonal group $\msf O(n;\bb R)$:
\begin{equation}
    \label{eq:L4MaxOrthogonalGroupDY}
    \max_{\mb D} \norm{\mb D^*\mb Y}{4}^4,\quad \st \quad  \mb D\in \msf O(n;\bb R),
\end{equation}
where the $\ell^4$-norm of a matrix means the sum of $4^\text{th}$ powers of all entries: $\forall \mb A\in \bb R^{n\times m},\norm{\mb A}{4}^4=\sum_{i,j}a_{i,j}^4$. The intuition for \eqref{eq:L4MaxOrthogonalGroupDY} comes from:
\begin{equation}
    \label{eq:L4MaxOrthogonalGroupX}
    \max_{\mb X,\mb D}\norm{\mb X}{4}^4,\quad \st \quad \mb {Y} = \mb {D}\mb {X},\mb D\in \msf O(n;\bb R),
\end{equation}
where maximizing the $\ell^4$-norm of $\mb X$ (over a sphere) promotes ``spikiness'' or ``sparsity'' of $\mb X$ \citep{zhang2018structured} . It is easy to see this as the sparsest points on a unit $\ell^2$-sphere (points $(0,1)$, $(0,-1)$, $(1,0)$, and $(-1,0)$) have the smallest $\ell^1$-norm and largest $\ell^4$-norm, as shown in Figure \ref{fig:2dLpUnitBall}. Since the columns of orthogonal matrices have unit norm, the constraint $\mb D\in \msf O(n;\bb R)$ can be viewed as simultaneously enforcing orthonormal constraints on $n$ vectors on the unit $\ell^2$-sphere $\bb S^1$. Moreover, comparing to the $\ell^1$-norm, the $\ell^4$-norm objective is everywhere smooth, so we expect it is amenable to better optimization. 

\begin{figure}[H]
    \centering
    \includegraphics[width=0.5\textwidth]{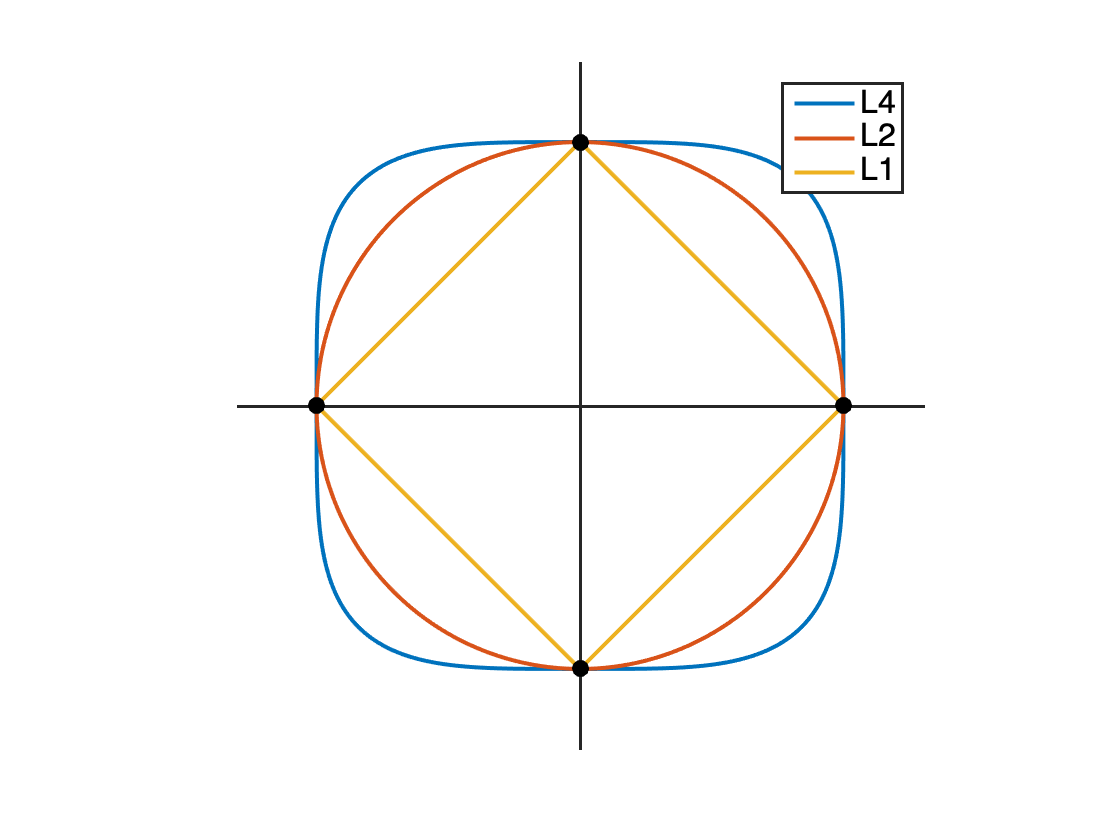}
    \caption{Unit $\ell^1$-,$\ell^2$-, and $\ell^4$- spheres in $\bb R^{2}$, a similar picture can be found in Figure 1 of \cite{li2018global}.}
    \label{fig:2dLpUnitBall}
\end{figure}

\subsubsection{Spherical Harmonic Analysis} The property of the $\ell^4$-norm has long been realized and used in seeking (orthonormal) functions with similar properties since 1970's, if not any earlier. For instance for spherical harmonics, the Stanton-Weinstein Theorem \citep{stanton1981l4,Lu-1987} have shown that ``among all the $\ell^2$-normalized spherical harmonics of a given degree, the $\ell^4$-norm is locally maximized by the `highest-weight' function,'' which among all the eigenfunctions of the Laplacian on the sphere, is the ``most concentrated'' in measure (see Theorem 1 of \cite{stanton1981l4}). In quantum mechanics, such functions represent trajectories that are the most probable and most closely approximate classical trajectories. 

\subsubsection{Independent Component Analysis} We should note that $4^\text{th}$-order statistical cumulants have been widely used in blind source separation or independent component analysis (ICA) since the 1990's, see \cite{Hyvarinen1997a,Hyvarinen1997-NC} and references therein. So if $\mb x$ are $n$ independent components, by finding extrema of the so-called {\em kurtosis}:
$\text{kurt}(\mb d^* \mb y) \doteq \bb E[(\mb d^* \mb y)^4] - 3\bb E[(\mb d^* \mb y)^2]^2,$
one can identify one independent (non-Gaussian) component $x_i$ at a time. Algorithm wise, this is similar to using the $\ell^1$-minimization \eqref{eq:L1MinSphereByColumn} to identify one column $\mb d_i$ at a time for $\mb D$. Fast fixed-point like algorithms have been developed for this purpose \citep{Hyvarinen1997a,Hyvarinen1997-NC}. If $\mb x$ are indeed i.i.d. Bernoulli-Gaussian, with $\|\mb d\|_2^2 = 1$, the second term in $\text{kurt}(\mb d^* \mb y)$ would become a constant and the objective of ICA coincides with maximizing the sparsity-promoting $\ell^4$-norm of a vector over a sphere.

\subsubsection{Sum of Squares} The use of $\ell^4$-norm can also be justified from the perspective of sum of squares (SoS). The works of \cite{Barak-2015,ma2016polynomial,schramm2017fast} show that in theory, when $\mb x$ is sufficiently sparse, one can utilize properties of higher order sum of squares polynomials (such as the fourth order polynomials) to correctly recover $\mb D$. Although \cite{schramm2017fast} has improved proposed faster tensor decomposition based on SoS method, again the algorithm only recovers one column $\mb d_i$ at a time. 

\subsubsection{Blind Deconvolution} Recent works in blind deconvolution \citep{zhang2018structured,li2018global} have also explored the sparsity promoting property of the $\ell^4$-norm in their objective function. \cite{zhang2018structured,li2018global} have shown that, for any filter on a unit sphere, all local maxima of the $\ell^4$-objective are close to the inverse of the ground truth filter (up to the intrinsic sign and shift ambiguity). Moreover, the global geometry of the $\ell^4$-norm over the sphere is good -- all saddle points have negative curvatures. Such nice global geometry guarantees a randomly initialized first-order Riemannian gradient descent algorithm to escape saddle points and find the ground truth, for both single channel \citep{zhang2018structured} and multi-channel \citep{li2018global,qu2019nonconvex} tasks.

\subsection{Main Results}

We shall first note that there is an intrinsic ``signed permutation'' ambiguity in all dictionary learning formulation \eqref{eq:L0Min}, \eqref{eq:L1Min}, and \eqref{eq:L4MaxOrthogonalGroupX}. For any input data matrix $\mb Y$, suppose $\mb D_\star\in \msf O(n;\bb R)$ is the optimal orthogonal dictionary that gives the sparsest coefficient matrix $\mb X_\star$ satisfying $\mb Y = \mb D_\star\mb X_\star$. Then for any matrix $\mb P$ in the signed permutation group $\msf{SP}(n)$, the group of orthogonal matrices that only contain $0, \pm1$, we have:
$$\mb Y = \mb D_\star \mb X_\star = \mb D_\star\mb P \mb P^* \mb X_\star,$$
where $\mb P^* \mb X_\star$ is equally sparse as $\mb X_\star$ and $\mb D_\star\mb P\in \msf O(n;\bb R)$. So we can only expect to recover the correct dictionary (and sparse coefficient matrix) {\em up to an arbitrary signed permutation.} Therefore, we say the ground truth dictionary $\mb D_o$ is successfully recovered, if any {\em signed permuted} version $\mb D_o\mb P$ is found.\footnote{In mathematical terms, we are looking for a solution in the quotient space between the orthogonal group and the signed permutation group: $\msf O(n;\bb R)/\msf{SP}(n)$.} Unlike approaches \citep{spielman2012exact,sun2015complete,bai2018subgradient} that solve one column at a time for $\mb D$, we here attempt recover the entire dictionary $\mb D$ from solving the problem \eqref{eq:L4MaxOrthogonalGroupDY}. The signed permutation ambiguity would create numerous equivalent global maximizers, which poses a serious challenge to analysis and optimization. 
 
In this paper, we adopt the Bernoulli-Gaussian model as in prior works \citep{spielman2012exact,sun2015complete,bai2018subgradient}. We assume our observation matrix $\mb Y\in \bb R^{n\times p}$ is produced by the product of a ground truth orthogonal dictionary $\mb D_o$, and a Bernoulli-Gaussian matrix $\mb X_o\in \bb R^{n\times p}$:
\begin{equation}
    \label{eq:BGDictionaryLearningFormulation}
    \mb Y = \mb D_o\mb X_o,\quad \mb D_o \in \msf O(n;\bb R),\; \{\mb X_o\}_{i,j} \sim_{iid} \text{BG}(\theta). 
\end{equation}
The Bernoulli-Gaussian model can be considered as a prototype for dictionary learning because one may adjust $\theta$ to control the sparsity level of the ground truth $\mb X_o$. With the Bernoulli-Gaussian assumption, we now state our main result.

\subsubsection{Correctness of the Proposed Objective Function}
\begin{theorem}[Correctness of Global Optima, informal version of Theorem \ref{Thm:MainResult}]
        \label{Thm:InformalMainResult}
        $\forall \theta\in(0,1)$, let $\mb X_o\in \bb R^{n\times p}$, $ x_{i,j}\sim_{iid}\text{BG}(\theta)$, $\mb D_o\in \msf O(n;\bb R)$ an arbitrary orthogonal matrix, and $\mb Y = \mb D_o\mb X_o$. Suppose $\hat{\mb A}_\star$ is a global maximizer of the optimization problem:
    \begin{equation}
        \label{eq:L4OrthObjective}
        \max_{\mb A} \norm{\mb A\mb Y}{4}^4,\quad \st \quad \mb A \in \msf O(n;\bb R), 
    \end{equation}
    then for any $\eps\in[0,1]$, there exists a signed permutation matrix $\mb P \in \msf{SP}(n)$, such that
    \begin{equation}
        \label{eqn:main-result-perturb}
        \frac{1}{n}\norm{\hat{\mb A}_\star^*-\mb D_o\mb P}{F}^2\leq C\eps,
     \end{equation}
    holds with high probability, when $p$ is large enough, and $C$ is a constant depends on $\theta$.
    
\end{theorem}
Theorem \ref{Thm:InformalMainResult} is obtained through the following line of reasoning: 1) $\forall \mb A \in \msf O(n;\bb R)$, we can view  $\frac{1}{p}\norm{\mb A\mb Y}{4}^4 = \frac{1}{p}\norm{\mb A\mb D_o\mb X_o}{4}^4=\frac{1}{p}\sum_{j=1}^p \norm{\mb A\mb D_o\mb x_j}{4}^4$ as the mean of $p$ i.i.d. random variables, so it will concentrate to its expectation $\bb E_{\mb x_j} \norm{\mb A \mb D_o \mb x_j}{4}^4$; 2) the value of $\bb E_{\mb x_j} \norm{\mb A \mb D_o \mb x_j}{4}^4$ is largely characterized by $\norm{\mb A\mb D_o}{4}^4$, a deterministic function (with respect to $\mb A$) on the orthogonal group, whose global optima $\mb A_\star$ satisfy $\mb A_\star = \mb P^*\mb D^*_o,\forall \mb P \in \msf{SP}(n)$. Therefore, when $p$ is large enough, maximizing $\norm{\mb A\mb Y}{4}^4$ over the orthogonal
group is equivalent to (with high probability) maximizing the deterministic objective $\norm{\mb A\mb D_o}{4}^4$ over the orthogonal group, which yields the desire result. Formal statements and proofs are given in Section \ref{sec:StatCharacter} and the Appendices.

\subsubsection{A Fast Optimization Algorithm}
\label{sec:ResultAlgorithm}
Unlike almost all previous algorithms that find the correct dictionary one column $\mb d_i$ at a time, in Section \ref{sec:Algorithm} we introduce a novel {\em matching, stretching, and projection} (MSP) algorithm that solves the program \eqref{eq:L4OrthObjective} directly for the entire $\mb D \in \msf O(n;\bb R)$. The MSP algorithm directly computes (transpose of) the optimal dictionary $\mb A_\star$ as the ``fixed point'' to the following iteration:
\begin{equation}
    \mb A_{t+1} = \mathcal{P}_{\mathsf O(n;\bb R)} \big[(\mb A_t \mb Y)^{\circ 3} \mb Y^*\big],
    \label{eqn:fixed-point-iteration}
\end{equation}
where $\mathcal{P}_{\mathsf O(n;\bb R)}(\cdot)$ is the projection onto the orthogonal group $\mathsf O(n;\bb R)$, which can be easily calculated from SVD (see Lemma \ref{lemma:OrthProj}). Meanwhile, the MSP algorithm also efficiently maximizes the deterministic objective $\norm{\mb A\mb D_o}{4}^4$ over $\msf O(n;\bb R)$ by the following iteration:
\begin{equation}
    \mb A_{t+1} = \mathcal{P}_{\mathsf O(n;\bb R)} \big[(\mb A_t \mb D_o)^{\circ 3} \mb D_o^*\big].
    \label{eqn:fixed-point-iteration-det}
\end{equation}
Statistical analysis for proving Theorem \ref{Thm:InformalMainResult} suggests that estimates from random samples converge to their expectation. For the deterministic objective $\norm{\mb A \mb D_o}{4}^4$, we further show that the proposed algorithm converges fast with a {\em cubic rate} around each global maximizer. 

Essentially, the update of the MSP algorithm \eqref{eqn:fixed-point-iteration}, \eqref{eqn:fixed-point-iteration-det} are performing projected gradient ascend with {\it infinite} step size with respect to objective function $\norm{\mb{AY}}{4}^4$, $\norm{\mb{AD}_o}{4}^4$, over $\msf O(n;\bb R)$ respectively. The MSP algorithm is similar to the FastICA algorithm \citep{Hyvarinen1997-NC} for independent component analysis (ICA) and such similarity is extensively discussed in a follow-up work \citep{zhai2019understanding}.

\begin{theorem}[Cubic Convergence Rate, informal version of Theorem \ref{Thm:MSPLocalConvergence}]
\label{Thm:MSPLocalConvergenceInformal}
Given an orthogonal matrix $\mb A\in \msf O(n;\bb R)$, if $\norm{\mb A-\mb I}{F}^2=\eps$ for a small $\eps<0.579$, and let $\mb A^\prime$ denote the result after one iteration of our proposed MSP algorithm \eqref{eqn:fixed-point-iteration-det}, then we have $\norm{\mb A^\prime - \mb I}{F}^2 \leq O(\eps^3)$.
\end{theorem}

In Theorem \ref{Thm:MSPLocalConvergenceInformal}, showing the local convergence of to identity matrix a general orthogonal matrix $\mb A$ to $\mb I$  suffices to characterize the local convergence of $\mb {AD}_o$ to a signed permutation matrix $\mb P$, because of the signed permutation symmetry in the $\ell^4$ norm and the orthogonal invariant of SVD. More details are provided in Section \ref{sec:StatCharacter} and Section \ref{sec:Algorithm}.  

Although Theorem \ref{Thm:InformalMainResult} characterizes the properties of the randomized objective $\norm{\mb {AY}}{4}^4$ while Theorem \ref{Thm:MSPLocalConvergenceInformal} describes a deterministic result, they are highly related with each other. We use iteration \eqref{eqn:fixed-point-iteration} to maximize the randomized objective $\norm{\mb {AY}}{4}^4$ of Theorem \ref{Thm:InformalMainResult} and Theorem \ref{Thm:MSPLocalConvergenceInformal} characterizes the local convergence of iteration \eqref{eqn:fixed-point-iteration-det}. $(\mb {AY})^{\circ3}\mb Y^*$ in \eqref{eqn:fixed-point-iteration} concentrates to its expectation when there is enough samples and the expectation $\bb E[(\mb {AY})^{\circ3}\mb Y^*]$ highly depends on $(\mb {AD}_o)^{\circ3}\mb D_o$ of \eqref{eqn:fixed-point-iteration-det}. Such algorithmic relationship between \eqref{eqn:fixed-point-iteration} and \eqref{eqn:fixed-point-iteration-det} is discussed in Proposition \ref{prop:L4MSPExp} and Proposition \ref{prop:GradHatfUnionConcentrationBound}
of Section \ref{sec:Algorithm}.

As our algorithm is very efficient and scalable, we can test it over very large range of dimensions and settings. Extensive simulations suggest that the MSP algorithm converges globally to the correct solution under broad conditions. We give a global convergence proof for the case $n=2$ (on $\msf O(2;\bb R)$) and conjecture that similar results hold for arbitrary $n$ (under mild conditions). Extensive experiments show that the algorithm is far more efficient than existing heuristic algorithms and (Riemannian) gradient or subgradient based algorithms. With this efficient algorithm, we characterize empirically the range of success for the program \eqref{eq:L4OrthObjective}, which goes well beyond any existing theoretical guarantees \citep{sun2015complete,bai2018subgradient} for the complete dictionary case.

\subsubsection{Observations and Implications} 
Notice that at first sight, the optimization problem associated with dictionary learning is highly nonconvex, with nontrivial orthogonal constraints, and with numerous critical points and ambiguities. Hence, understandably, many recent approaches focus on introducing regularization to the objective function so as to relax the constraints  \citep{aharon2006k,mairal2008discriminative,mairal2009supervised,mairal2012task,wu2015local,wang2019unique} or analyzing and utilizing local information and design heuristic or gradient descent schemes for such objective functions (with or without regularization) \citep{agarwal2014learning,arora2015simple,ge2015escaping,ma2017implicit,li2018learning,du2018gradient,allen2018convergence,davis2018stochastic}.

However, in our work, we have observed a surprising phenomenon that is rather contrary to conventional views: The discrete signed permutation symmetry $\msf{SP}(n)$ associated with the orthogonal group $\msf O(n;\bb R)$ are beneficial. In our work, they both play an important role in making the algorithm efficient and effective, instead of being nuisances or difficulties to be dealt with. As we will see, such discrete symmetry of the orthogonal group makes the global landscape of the objective function amenable to global convergence and enables a fixed-point type algorithm that converges at a super-linear rate to a correct solution in the quotient space $\msf O(n;\bb R)/\msf{SP}(n)$. Similar phenomena that symmetry facilitates global convergence of non-convex programs have been also been observed and reported in recent works \citep{ge2015escaping,sun2015complete,zhang2018structured,li2018global,chi2018nonconvex,kuo2019geometry}. This is a direction that deserves better and deeper study in the future which encourages significant confluence of geometry, algebra, and statistics.

\subsection{Notations}
We use a bold uppercase and a bold lowercase letter to denote a matrix and a vector, respectively: $\mb X\in \bb R^{n\times p}, \mb x\in \bb R^n$. Moreover, for a matrix $\mb X\in \bb R^{n\times p}$, we use $\mb x_j\in \bb R^n,\forall j\in [p]$ to denote its $j^{\text{th}}$ column vector as default. We use $\mb X^*$ or $\mb x^*$ to denoted the (conjugate) transpose of a matrix or a vector, respectively. We reserve lower-case letter for scalar: $x\in\bb R$. We use $\norm{\mb X}{4}$ to denote the element-wise $\ell^4$-norm of a matrix $\mb X$ ($\norm{\mb X}{4}^4 = \sum_{i,j}x_{i,j}^4$). We use $\mb D_o$ to denote the ground truth orthogonal dictionary, and $\mb A$ is an estimate of $\mb D^*_o$ from solving \eqref{eq:L4OrthObjective}. $\theta\in (0,1)$ to is sparsity level of the ground truth Bernoulli-Gaussian sparse coefficient $\mb X_o$: $\mb X_o\sim \text{BG}(\theta)$. We use $\circ$ to denote the Hadamard product: $\forall \mb A,\mb B\in \bb R^{n\times m}$, $\{\mb A\circ\mb B\}_{i,j} = a_{i,j}b_{i,j}$, and $\{\mb A^{\circ r}\}_{i,j}=a_{i,j}^r$ is the element-wise $r^\text{th}$ power of $\mb A$.

Given an input data matrix $\mb Y$ randomly generated from $\mb Y = \mb D_o\mb X_o$, $\mb X_o\sim_{iid} \text{BG}(\theta)$, for any orthogonal matrix $\mb A\in\msf O(n;\bb R)$, we define $\hat{f}:\msf O(n;\bb R)\times \bb R^{n\times p}\mapsto \bb R$ as the $4^\text{th}$ power of $\ell^4$-norm of $\mb {AY}$: 
    \begin{equation}
        \hat{f}(\mb A,\mb Y) \doteq \norm{\mb {AY}}{4}^4.
    \end{equation}
We define $f:\msf O(n;\bb R)\mapsto \bb R$ as the expectation of $\hat{f}$ over $\mb X_o$:
    \begin{equation}
        f(\mb A) \doteq \bb E_{\mb X_o}[\hat{f}(\mb A,\mb Y)] = \bb E_{\mb X_o}\big[\norm{\mb {AY}}{4}^4\big].
    \end{equation}
For any orthogonal matrix $\mb W\in \msf O(n;\bb R)$, we define $g:\msf O(n;\bb R)\mapsto \bb R$ as $4^\text{th}$ power of its $\ell^4$-norm:
\begin{equation}
    g(\mb W) \doteq \norm{\mb W}{4}^4.    
\end{equation}

\subsection{Organization of this Paper}
Rest of the paper is organized as follows. In Section \ref{sec:StatCharacter}, we characterize the global maximizers of \eqref{eq:L4OrthObjective} statistically via measure concentration. In Section \ref{sec:Algorithm}, we describe the proposed MSP algorithm. We characterize fixed points of the algorithm and show its convergence results in Section \ref{sec:Analysis}, All related proofs can be found in the appendices. Finally, in Section \ref{sec:Experiments}, we conduct extensive experiments to show effectiveness and efficiency of our method, by comparing with the state of the art. 

\section{Key Analysis and Main Result}
\label{sec:StatCharacter}
\subsection{Expectation and Concentration of the $\ell^4$-Objective}

In this section, we statistically justify that one can recover the ground truth dictionary $\mb D_o$ by solving
\begin{equation}
    \label{eq:maxHatf}
    \max_{\mb A} \hat{f}(\mb A,\mb Y)=\norm{\mb {AY}}{4}^4,\quad \st \quad \mb A\in \msf O(n;\bb R).
\end{equation}
Notice that the solution to the above $\ell^4$-norm optimization problem: 
\begin{equation*}
    \hat{\mb A}_\star = \underset{\mb A\in\msf O(n;\bb R)}{\arg\max} \hat{f}(\mb A,\mb Y)
\end{equation*}
is a random variable that depends on the random samples $\mb Y$. We need to characterize how ``close'' an estimate $\hat{\mb A}_\star$ is to the ground truth $\mb D_o$. A key technique is to show that the random objective function actually concentrates on its expectation (a deterministic function) as the number of observations $p$ increases. We first calculate $f(\mb A)$, the expectation of $\hat{f}(\mb A,\mb Y)$ and provide its concentration bound in Lemma \ref{lemma:orthproperty} and Lemma \ref{lemma:HatfUnionConcentrationBound} respectively.


\begin{lemma} [Properties of $f(\mb A)$]
    \label{lemma:orthproperty}
    $\forall \theta \in (0,1)$, let $\mb X_o\in \bb R^{n\times p}$, $x_{i,j}\sim_{iid}\text{BG}(\theta)$, $\mb D_o\in \msf O(n;\bb R)$ is an orthogonal matrix, and $\mb Y = \mb D_o\mb X_o$. Then, $\forall \mb A\in \msf O(n;\bb R)$, we have
    \begin{equation}
        \label{eq:orthproperty}
        \frac{1}{3p\theta}f(\mb A) = (1-\theta)g(\mb {AD}_o) + \theta n.
    \end{equation}
\end{lemma}
\begin{proof}
    See \ref{proof:OrthProperty}
\end{proof}

\begin{lemma}[Concentration Bound of $\ell^4$-Norm]
    \label{lemma:HatfUnionConcentrationBound}
    $\forall \theta\in(0,1)$, if $\mb X\in \bb R^{n\times p}$, $x_{i,j}\sim_{iid}\text{BG}(\theta)$, $\forall \delta>0$, the following inequality holds:
    \begin{equation}
         \begin{split}
            &\bb P\left(\sup_{\mb W\in \msf O(n;\bb R)}\frac{1}{np}\abs{\norm{\mb W \mb X}{4}^4-\bb E \norm{\mb W \mb X}{4}^4}\geq \delta\right)< \frac{1}{p},
        \end{split}
    \end{equation}
    when $p=\Omega(\theta n^2\ln n/\delta^2)$.
\end{lemma}
\begin{proof}
    See \ref{proof:HatfUnionConcentrationBound}.
\end{proof}

Lemma \ref{lemma:HatfUnionConcentrationBound} implies that, for any orthogonal transformation $\mb {WX},\mb W\in \msf O(n;\bb R)$ of a Bernoulli-Gaussian matrix $\mb X\in \bb R^{n\times p}$, $\frac{1}{np}\norm{\mb {WX}}{4}^4$ concentrates onto its expectation as long as the sample size $p$ is large enough -- in the order $\Omega(\theta n^2\ln n/\delta^2)$. By our definition, $\hat{f}(\mb A,\mb Y) = \norm{\mb {AY}}{4}^4 = \norm{\mb {W}\mb {X}_o}{4}^4$ ($\mb W=\mb {AD}_o$ is an orthogonal matrix) satisfies the concentration inequality in Lemma \ref{lemma:HatfUnionConcentrationBound}. Therefore, including designing optimization algorithms, $f(\mb A)$ can be considered as a good proxy to the original objective $\hat{f}(\mb A,\mb Y)$ and we can consider \eqref{eq:maxHatf} as maximizing its expectation:
\begin{equation}
\label{eq:maxf}
    \max_{\mb A} f(\mb A) = \bb E \norm{\mb {AY}}{4}^4 \quad\st\mb A \in \msf O(n;\bb R).
\end{equation}

The function $f(\mb A)$ is a deterministic function and its geometric property would help us understand the landscape of the original objective. Moreover, Lemma \ref{lemma:orthproperty} states that the global maximizers of the mean $f(\mb A)$ are exactly the global maximizers of $g(\mb {AD}_o) = \|\mb A\mb D_o\|_4^4$. To understand how such a function can be effectively optimized, we need to study the extrema of $\ell^4$-norm $g(\cdot)$ over the orthogonal group $\msf O(n;\bb R)$.

\subsection{Property of the $\ell^4$-Norm over $\msf O(n;\bb R)$}



\begin{lemma} [Extrema of $\ell^4$-Norm over Orthogonal Group]
    \label{lemma:G(A)bound}
     For any orthogonal matrix $\mb A\in \msf O(n;\bb R)$, $g(\mb A) = \norm{\mb A}{4}^4\in[1, n]$ and $g(\mb A)$ reaches maximum if and only if $\mb A\in \msf{SP}(n)$.
\end{lemma}
\begin{proof}
    See \ref{proof:G(A)bound}.
\end{proof}
This lemma implies that if $\mb A_\star$ is a global maximizer of $g(\mb {AD}_o)$, i.e. $\norm{\mb {A}_\star\mb{D}_o}{4}^4 = n$, then it differs from $\mb D_o^*$ by a signed permutation. However, this lemma does not say the function may or may not have other local minima or maxima. Nevertheless, our experiments in Section \ref{sec:experiment-DL} will show that even if such critical points exist, they are unlikely to be stable (or attractive). As a direct corollary to Lemma \ref{lemma:orthproperty} and Lemma \ref{lemma:G(A)bound}, we know that $\forall \mb A\in \msf O(n;\bb R),\theta\in(0,1)$, the maximum value of $f(\mb A)$ satisfies:
\begin{equation}
    \frac{1}{3p\theta}f(\mb A)\leq n,
\end{equation}
and the equality holds if and only if $\mb {AD}_o\in \msf{SP}(n)$. Although we know that in the stochastic setting, \eqref{eq:maxHatf} concentrates to the following $\ell^4$-norm maximization over $\msf O(n;\bb R)$:
\begin{equation}
    \label{eq:maxg}
    \max_{\mb A} g(\mb A \mb D_o) = \norm{\mb {AD}_o}{4}^4,\quad \st \quad \mb A\in \msf O(n;\bb R),
\end{equation}
we cannot hope that the estimate from $\hat{\mb A}_\star = \arg\max_{\mb A\in \msf O(n;\bb R)} \hat{f}(\mb A,\mb Y)$ would achieve the maximal value of $g(\mb A \mb D_o)$ precisely. But if the value is close to the maximal, how close would $\hat{\mb A}_\star$ be to a global maximizer? The following lemma shows that when the $\ell^4$-norm of an orthogonal matrix $\mb A$ is close to the maximum value $n$, it is also close to a signed permutation matrix in Frobenius norm. 
\begin{lemma}[Approximate Maxima of $\ell^4$-Norm over Orthogonal Group]
    \label{lemma:L4ExtremaBoundOrthGroup}
    Suppose $\mb W$ is an orthogonal matrix: $\mb W\in \msf O(n;\bb R)$. $\forall \eps\in[0,1]$, if $\frac{1}{n}\norm{\mb W}{4}^4\geq1-\eps$, then $\exists \mb P\in \msf{SP}(n)$, such that 
    \begin{equation}
        \frac{1}{n}\norm{\mb W - \mb P}{F}^2 \leq 2\eps.
    \end{equation}
\end{lemma}

\begin{proof}
    See \ref{proof:L4ExtremaBoundOrthGroup}.
\end{proof}
This result is useful whenever we evaluate how close a solution given by an algorithm is to the optimal solution, in terms of value of the objective function. With all the above results, we are now ready to characterize in what sense a global maximizer of $\hat{\mb A}_\star = \arg\max_{\mb A\in \msf O(n;\bb R)} \hat{f}(\mb A,\mb Y)$ gives a ``correct'' estimate of the ground truth dictionary $\mb D_o$. 

\subsection{Main Statistical Result}

\begin{theorem}[Correctness of the Global Optima]
    \label{Thm:MainResult}
    $\forall \theta\in(0,1)$, let $\mb X_o\in \bb R^{n\times p}$, $x_{i,j}\sim_{iid}\text{BG}(\theta)$, $\mb D_o\in \msf O(n;\bb R)$ is any orthogonal matrix, and $\mb Y = \mb D_o\mb X_o$. Suppose $\hat{\mb A}_\star$ is a global maximizer of the optimization problem:
    \begin{equation*}
        \max_{\mb A} \hat{f}(\mb A,\mb Y)=\norm{\mb A\mb Y}{4}^4,\quad \st \quad \mb A \in \msf O(n;\bb R), 
    \end{equation*}
    then for any $\eps\in[0,1]$, there exists a signed permutation matrix $\mb P \in \msf{SP}(n)$, such that
    \begin{equation}
        \frac{1}{n}\norm{\hat{\mb A}_\star^*-\mb D_o\mb P}{F}^2\leq C\eps,
     \end{equation}
     with probability at least $1-\frac{1}{p}$, when $p=\Omega(\theta n^2\ln n/\eps^2)$, for a constant $C>\frac{4}{3\theta(1-\theta)}$. 
\end{theorem}
\begin{proof}
    See \ref{proof:MainResult}.
\end{proof}
Theorem \ref{Thm:MainResult} states that with high probability, the global optimal solution to the $\ell^4$-norm optimization \eqref{eq:maxHatf} is close to the true solution (up to a signed permutation) as the sample size $p$ is of the order $\Omega(\theta n^2\ln n/\eps^2)$, where $\eps$ is the desired accuracy. This result quantifies the sample size needed to achieve certain accuracy of recovery and it matches the intuition that more observations lead to better recovering result. Moreover, this result corresponds to our phase transition curves in Figure \ref{fig:SPPR} of Section \ref{sec:Experiments}, and to the best of our knowledge, a sample complexity of $\Omega(n^2\ln n)$ is currently the best result for dictionary learning task.

\begin{remark}[Maximizing $\ell^{2k}$-Norm]

If one were to choose maximizing $\ell^{2k}$-norm to promoting sparsity, similar analysis of concentration bounds would reveal that for the same error bound, it requires much larger number $p$ of samples for the (random) objective function $\hat f(\mb A,\mb Y)$ to concentrate on its (deterministic) expectation $f(\mb A)$. We will discuss the choice of $k$ in more details in Section \ref{subsec:General2kNorm}. Experiments in Figure \ref{fig:SPOrder2k} also corroborate with the findings.
\end{remark}

\section{Algorithm: Matching, Stretching, and Projection (MSP)}
\label{sec:Algorithm}
In this section, we introduce an algorithm, based on a simple iterative {\em matching, stretching, and projection} (MSP) process, which efficiently solves the two related programs \eqref{eq:maxHatf} and \eqref{eq:maxg}.

\subsection{Algorithmic Challenges and Related Optimization Methods}
Although \eqref{eq:maxHatf} is everywhere smooth, the associated optimization is non-trivial in several ways. First, one needs to deal with the signed permutation ambiguity. The problem has exponentially many global maximizers. Furthermore, we are maximizing a convex function (or minimizing a concave function) over a constraint set. So conventional methods such as augmented Lagrangian \citep{bertsekas1997nonlinear} barely works. This is because the Lagrangian:
\begin{equation*}
    \mc L(\mb A,\mb \Lambda) \doteq -\norm{\mb A\mb Y}{4}^4 + \innerprod{\mb A \mb A - \mb I}{\mb \Lambda}
\end{equation*}
will go to negative infinity due to the concavity of the objective function $-\norm{\mb A\mb Y}{4}^4$. Notice that all of its global maximizers are on the constraint set, an ideal iterative algorithm should converge to a solution that {\em exactly lies on constraint set} $\msf O(n;\bb R)$.

Another natural way to optimize  \eqref{eq:maxHatf} (or \eqref{eq:maxg}) is to apply Riemannian gradient (or projected gradient) type methods \citep{edelman1998geometry,absil2009optimization} on $\msf O(n;\bb R)$. One can take small gradient steps to ensure convergence and such methods converge at best with a linear rate (say, if the objective function is strongly convex). Nevertheless, by better utilizing the special global geometry of the parameter space (the orthogonal group), we can choose an arbitrary large step size and the process converges much more rapidly (with a superlinear rate). 

We next introduce a very effective and efficient algorithm to solve problems \eqref{eq:maxHatf} and \eqref{eq:maxg}, that is based on simple iterative matching, stretching and projection (MSP) operators. Mathematically, our algorithm for solving \eqref{eq:maxHatf} or \eqref{eq:maxg} are essentially doing projected gradient ascent on objective $\norm{\mb {AD}_o}{4}^4$ or $\norm{\mb {AY}}{4}^4$ respectively, but with {\em infinite} step size. Such infinite step size gradient method is able to find global maximum of our proposed objective \eqref{eq:maxHatf} and \eqref{eq:maxg} because it exploits the coincidence between the signed permutation symmetry in our formulation and the symmetry of $\msf{SP}(n)$ over $\msf O(n;\bb R)$.

\subsection{$\ell^4$-Norm Maximization over the Orthogonal Group}
Since the objective function of dictionary learning  \eqref{eq:maxHatf} concentrates to the $\ell^4$-norm maximization problem \eqref{eq:maxg} w.h.p., we first introduce our algorithm for the simpler (deterministic) case:
\begin{equation}
    \max_{\mb A} g(\mb {AD}_o)=\norm{\mb {AD}_o}{4}^4,\quad \st \quad \mb A\in\msf O(n;\bb R).
\end{equation}
In this setting, we only have information of the values of $g(\mb {AD}_o)$ and $\nabla_{\mb A}g(\mb {AD}_o)$. We want to update $\mb A$ to recover $\mb D_o^*$ based on these information. We use the following lemma to enforce the orthogonal constraint.
\begin{lemma}[Projection onto the Orthogonal Group]
    \label{lemma:OrthProj}
    $\forall \mb A\in \bb R^{n\times n}$, the orthogonal matrix which is closest to $\mb A$ in Frobenius norm is the following:
    \begin{equation}
        \mathcal{P}_{\msf O(n;\bb R)}(\mb A) = \underset{\mb M \in\msf O(n;\bb R)}{\arg\min} \norm{\mb A-\mb M}{F}^2 = \mb {UV}^*,
    \end{equation}
    where $\mb {U\Sigma V}^* = \text{SVD}(\mb A)$.
\end{lemma}
\begin{proof}
See \ref{proof:OrthProj}. 
\end{proof}
The MSP algorithm that maximizing \eqref{eq:maxg} is outlined as Algorithm \ref{algo:SPOrthD}. 
\begin{algorithm}[H]
  \caption{MSP for $\ell^4$-Maximization over $\msf O(n;\bb R)$}\label{algo:SPOrthD}
  \begin{algorithmic}[1]
    \State \textbf{Initialize} $\mb A_0\in \msf O(n,\bb R)$\Comment{Initialize $\mb A_0$ for iteration}
      \For{$t=0,1,...,T-1$}
        \State $\partial \mb A_t \doteq 4 (\mb A_t\mb D_o)^{\circ3}\mb D_o^* $\Comment{$\nabla_{\mb A}\norm{\mb A\mb D_o}{4}^4=4(\mb A\mb D_o)^{\circ3}\mb D_o^* $}
        \State $ \mb {U\Sigma V}^* = \text{SVD}\big(\partial\mb A_t \big)$ 
        \State $\mb A_{t+1} = \mb {UV}^*$\Comment{Project $\mb A_{t+1}$ onto orthogonal group}
      \EndFor
      \State \textbf{Output} $\mb A_{T}$
  \end{algorithmic}
\end{algorithm}
In each iteration of Algorithm \ref{algo:SPOrthD}, we use $\norm{\mb {AD}_o}{4}^4/n$ to evaluate how ``close'' $\mb {AD}_o$ is to a signed permutation matrix, since Lemma \ref{lemma:G(A)bound} shows the global maximum of $\norm{\mb {AD}_o}{4}^4$ is $n$ and the maximal evaluation is therefore normalized to 1. In Step 3 of the MSP algorithm, the calculation of $ \partial\mb A_t = 4(\mb A_t\mb D_o)^{\circ3}\mb D_o^*$ does not require knowledge of $\mb D_o$. It is merely the gradient of the objective function: 
\begin{equation*}
    \nabla_{\mb A} g(\mb A \mb D_o) = \nabla_{\mb A}\norm{\mb A\mb D_o}{4}^4=4(\mb A\mb D_o)^{\circ3}\mb D^*_o.
\end{equation*}

As the name of the algorithm suggests, each iteration actually performs a {\em ``matching, stretching, and projection''} operation: It first matches the current estimate $\mb A_t$ to the true $\mb D_o$. Then the element-wise cubic function $(\cdot)^{\circ3}$ stretches all entries of $\mb A_t\mb D_o$ by promoting the large ones and suppressing the small ones. 
$\partial \mb A_t$ is the correlation between so ``sparsified'' pattern and the original basis $\mb D_o^*$, which is then projected back onto the closest orthogonal matrix $\mb A_{t+1}$ in Frobenius distance. 

Repeating this ``matching, stretching, and projection'' process,  $\mb A_t\mb D_o$ is increasingly sparsified while ensuring the orthogonality of each $\mb A_t$. Ideally the process will stop when $\mb A_t\mb D_o$ becomes the sparsest, that is, a signed permutation matrix. The iterative MSP algorithm utilizes the global geometry of the orthogonal group and acts more like the {\em power iteration} method or the fixed point algorithm \citep{Hyvarinen1997-NC}.
It is easy to see that for any fixed $\mb D_o$, the optimal $\mb A_\star$ is the ``fixed point'' to the following equation:
\begin{equation}
    \mb A_\star = \mathcal{P}_{\mathsf O(n;\bb R)} \big[(\mb A_\star \mb D_o)^{\circ 3} \mb D^*_o\big],
\end{equation}
where $\mathcal{P}_{\mathsf O(n;\bb R)}$ is the projection onto the orthogonal group $\mathsf O(n;\bb R)$. The proposed ``matching, stretching, and projection'' algorithm is precisely to compute the fixed point of this equation in the most natural way! Our analysis (Theorem \ref{Thm:MSPLocalConvergence}) will show that this scheme converges extremely well and actually achieves a {\em super-linear} local convergence rate.

\begin{example}[One Run of Algorithm \ref{algo:SPOrthD}]
\label{example:SPAlgo4}
To help better visualize how well the algorithm works, we consider a special case when $\mb D_o = \mb I$. The problem reduces to finding a matrix with the maximum $\ell^4$-norm over the orthogonal group:
\begin{equation}
    \max_{\mb A} g(\mb {A})\doteq\norm{\mb {A}}{4}^4,\quad \st \quad \mb A\in \msf O(n;\bb R).
\end{equation}
We randomly initialize the MSP algorithm with an orthogonal matrix $\mb A_0 \in \bb R^{3\times 3}$, and the sequences below show how the quickly the MSP algorithm quickly converges to a signed permutation matrix:
\begin{footnotesize}
\begin{equation*}
\begin{array}{llll}
    &\mb A_0 = 
     \begin{pmatrix}
       -0.8249 & 0.3820 & -0.4168 \\
       -0.5240 & -0.2398 & 0.8173 \\
       -0.2122 & -0.8925 &-0.3979 
     \end{pmatrix}        
    &\xrightarrow{\text{stretching}}
    &\mb A_0^{\circ3}=
    \begin{pmatrix}
        -0.5613 & 0.0557 & -0.0724 \\
        -0.1439 & -0.0138 & 0.5459 \\
        -0.0096 & -0.7109 & -0.0630 
    \end{pmatrix}\\
    \xrightarrow{\text{projection}}
    &\mb A_1=
    \begin{pmatrix}
       -0.9795 & 0.0621 & -0.1917 \\
       -0.1953 & -0.0594 & 0.9789 \\
       -0.0494 & -0.9963 & -0.0703 
    \end{pmatrix}
    &\xrightarrow{\text{stretching}}
    &\mb A_1^{\circ3}=
    \begin{pmatrix}
       -0.9397  & 0.0002 & -0.0070 \\
       -0.0075 &  -0.0002 & 0.9381 \\
       -0.0001 & -0.9889 & -0.0003
    \end{pmatrix}\\
    \xrightarrow{\text{projection}}
    &\mb A_2=
    \begin{pmatrix}
        -1.0000 & 0.0002 & -0.0077\\ 
        -0.0077 & -0.0003 &  1.000 \\ 
        -0.0002 &  -1.0000 & -0.0003
    \end{pmatrix}
    &\xrightarrow{\text{stretching}}
    &\mb A_2^{\circ3}=
    \begin{pmatrix}
        -0.9999 & 0.0000 & -0.0000 \\
        -0.0000 & -0.0000 & 0.9999 \\
        -0.0000 & -1.0000 & -0.0000
    \end{pmatrix}\\
    \xrightarrow{\text{projection}}
    &\mb A_3=
    \begin{pmatrix}
        -1 & 0 & 0 \\
        0  & 0 & 1 \\
        0 & -1 & 0 
    \end{pmatrix}
    &\xrightarrow{\text{output}}
    &\mb A_3^{\circ3}=
    \begin{pmatrix}
        -1 & 0 & 0 \\
        0  & 0 & 1 \\
        0 & -1 & 0 
    \end{pmatrix}.
\end{array}
\end{equation*}
\end{footnotesize}
\end{example}



\subsection{$\ell^4$-Norm Maximization for Dictionary Learning}

Having understood how to optimize the deterministic case for the expectation, we now consider the original dictionary learning problem \eqref{eq:maxHatf}:
\begin{equation*}
    \max_{\mb A} \hat{f}(\mb A,\mb Y)=\norm{\mb A \mb Y}{4}^4,\quad \st \quad \mb A\in \msf O(n;\bb R).
\end{equation*}
This naturally leads to a similar MSP algorithm, outlined as Algorithm \ref{algo:SPOrthDL}.

\begin{algorithm}[H]
  \caption{MSP for $\ell^4$-Maximization based Dictionary Learning}\label{algo:SPOrthDL}
  \begin{algorithmic}[1]
    \State \textbf{Initialize} $\mb A_0\in \msf O(n,\bb R)$ \Comment{Initialize $\mb A_0$ for iteration} 
      \For{$t=0,1,...,T-1$}
        \State $\partial \mb A_t \doteq 4(\mb {A}_t\mb {Y})^{\circ3}\mb Y^*$ \Comment{$\nabla_{\mb A}\norm{\mb {AY}}{4}^4=4(\mb {AY})^{\circ3}\mb Y^*$}
        \State $ \mb {U\Sigma V}^* = \text{SVD}\big(\partial \mb A_t\big)$ 
        \State $\mb A_{t+1} = \mb {UV}^*$\Comment{Project $\mb A_{t+1}$ onto orthogonal group}
      \EndFor
      \State \textbf{Output} $\mb A_{T},\norm{\mb {A}_{T}\mb Y}{4}^4/3np\theta$ 
  \end{algorithmic}
\end{algorithm}

Note that in the output we also normalize $\norm{\mb {AY}}{4}^4$ by dividing the maximum of its expectation: $3np\theta$ so that the maximal output value would be around 1. Similar to Algorithm \ref{algo:SPOrthD}, we also use $\norm{\mb {AD}_o}{4}^4/n$ to evaluate how ``close'' $\mb {AD}_o$ is to a signed permutation matrix in each iteration. 

The same intuition of ``matching, stretching, and projection'' for the deterministic case naturally carries over here. In Step 3, the estimate $\mb A_t$ is matched with the observation $\mb Y$. The cubic function $(\cdot)^{\circ3}$ re-scales the results and promotes entry-wise spikiness of $\mb X_t = \mb A_t \mb Y$ accordingly. Again, here $\partial\mb A_t = 4(\mb {A}_t\mb {Y})^{\circ3}\mb Y^*$ is the gradient $\nabla_{\mb A}\hat{f}(\mb A,\mb Y)$ of the objective function. However, the algorithm is not performing gradient ascent: The matrix $(\mb {A}_t\mb {Y})^{\circ3}\mb Y^*$ is actually the sample covariance of the following two random vectors: $(\mb A_t \mb y)^{\circ 3}$ and $ \mb y$. The subsequent projection of this sample covariance matrix onto the orthogonal group $\mathsf O(n;\bb R)$ normalizes the scale of the  operator $\mb A_{t+1}$, hence normalize the covariance of $\mb A_{t+1} \mb Y$ for the next iteration.

Similar to the deterministic case, for any given sparsely generated data matrix $\mb Y$, the optimal dictionary $\mb A_\star$ is the ``fixed point'' to the following equation:
\begin{equation}
    \mb A_\star = \mathcal{P}_{\mathsf O(n;\bb R)} \big[(\mb A_\star \mb Y)^{\circ 3} \mb Y^*\big],
\end{equation}
where $\mathcal{P}_{\mathsf O(n;\bb R)}$ is the projection onto the orthogonal group $\mathsf O(n;\bb R)$. The iterative ``matching, stretching, and projection'' scheme is precisely to compute the fixed point of this equation in the most natural way!

Although the data and the objective function are random here, Proposition \ref{prop:L4MSPExp} below clarifies the relationship between this expectation and the deterministic gradient $\nabla_{\mb A}g(\mb {AD}_o)$ and Proposition \ref{prop:GradHatfUnionConcentrationBound} further shows that $\nabla_{\mb A}\hat{f}(\mb A,\mb Y)$ concentrates to its expectation when $p$ increases.

\begin{proposition}[Expectation of $\nabla_{\mb A}\hat{f}(\mb A,\mb Y)$ ]
\label{prop:L4MSPExp}
Let $\mb X\in \bb R^{n\times p}$, $x_{i,j}\sim_{iid}\text{BG}(\theta)$, $\mb D_o\in \msf O(n;\bb R)$ is any orthogonal matrix, and $\mb Y = \mb D_o\mb X_o$. The expectation of $\nabla_{\mb A}\hat{f}(\mb A, \mb Y)$ satisfies this property:
\begin{equation}
    \bb E_{\mb X_o} \big[\nabla_{\mb A}\hat{f}(\mb A, \mb Y)\big] = 3p\theta(1-\theta)\nabla_{\mb A}g(\mb {AD}_o)+12p\theta^2\mb A.
\end{equation}
\end{proposition}
\begin{proof}
    See \ref{proof:L4MSPExp}.
\end{proof}
This proposition indicates that the expected stochastic gradient $\nabla_{\mb A}\hat{f}(\mb A, \mb Y)$ agrees well with the gradient of the deterministic objective $g(\mb A \mb D_o)$, expect for a bias that is linear in the current estimate $\mb A$. This suggests a possible improvement for the MSP algorithm in the stochastic case: If we have knowledge about the $\theta$ (or can estimate it online), we could subtract a bias term $\alpha\mb A$ ($ \alpha \in (0,12p\theta^2]$) from $\partial \mb A_t$ in Step 3 of Algorithm \ref{algo:SPOrthDL}. One can verify experimentally that this indeed helps further accelerate the convergence of the algorithm. 

\begin{proposition}[Concentration Bound of $\frac{1}{np}\nabla\hat{f}(\cdot,\cdot)$]
    \label{prop:GradHatfUnionConcentrationBound}
    If $\mb X_o\in \bb R^{n\times p},x_{i,j}\sim_{iid}\text{BG}(\theta)$, for any $\mb A,\mb D_o \in \msf O(n;\bb R)$, and $\mb Y=\mb D_o\mb X_o$, the following inequality holds
    \begin{equation}
        \begin{split}
            &\bb P\Bigg(\sup_{\mb A\in \msf O(n;\bb R)}\frac{1}{4np}\norm{\nabla_{\mb A} \hat{f}(\mb A,\mb Y)-\bb E\big[\nabla_{\mb A} \hat{f}(\mb A,\mb Y)]}{F}\geq \delta\Bigg)< \frac{1}{p},
        \end{split}
    \end{equation}
    when $p=\Omega(\theta n^2\ln n/\delta^2)$.
\end{proposition}
\begin{proof}
    See \ref{proof:GradHatfUnionConcentrationBound}.
\end{proof}

\section{Analysis of the MSP Algorithm}
\label{sec:Analysis}
In this section, we provide convergence analysis of the proposed MSP Algorithm \ref{algo:SPOrthD} that maximizes $g(\mb {AD}_o)$ over the orthogonal group $\msf O(n;\bb R)$. Each iteration of Algorithm \ref{algo:SPOrthD} performs the following iteration:
\begin{equation}
    \label{eq:MSPUpdate}
    \mb A_{t+1} = \mathcal{P}_{\mathsf O(n;\bb R)} \big[(\mb A_t \mb D_o)^{\circ 3} \mb D_o^*\big],
\end{equation}
notice that both $\mb A$ and $\mb D_o$ are orthogonal matrices, we can further reduce \eqref{eq:MSPUpdate} to
\begin{equation}
    \mb A_{t+1}\mb D_o = \mathcal{P}_{\mathsf O(n;\bb R)} \big[(\mb A_t \mb D_o)^{\circ 3}\big].
\end{equation}
If we view $\mb A_t \mb D_o$ as another orthogonal matrix $\mb W_t$, the convergence analysis reduces to prove that the following iteration
\begin{equation}
\label{eq:MSPUpdateNew}
    \mb W_{t+1} = \mathcal{P}_{\mathsf O(n;\bb R)} \big[(\mb W_t)^{\circ 3}\big]
\end{equation}
will converge to a signed permutation matrix $\mb W_\infty$. Hence, we conclude that the MSP algorithm \ref{algo:SPOrthD} for maximizing $g(\mb {AD}_o)$ is invariant of orthogonal rotation. So without loss of generality, we only need to provide convergence analysis for the case $\mb D_o=\mb I$ (or slightly abuse the notation a bit by changing $\mb W_t$ in \eqref{eq:MSPUpdateNew} into $\mb A_t$). 

When $\mb D_o=\mb I$, we wish to show the MSP algorithm converges to a signed permutation matrix for the optimization problem:
\begin{equation}
    \label{eq:L4MaxOrthClean}
    \max_{\mb A}g(\mb A) = \norm{\mb A}{4}^4,\quad \st \quad \mb A\in \msf O(n;\bb R),
\end{equation}
starting from any randomly initialized $\mb A_0$ on $\msf O(n;\bb R)$ with probability 1. For this purpose, we first introduce some basic properties of the space $\msf O(n;\bb R)$ and our objective function $g(\cdot)$. 

\subsection{Properties of the Orthogonal Group}

The orthogonal group is a special type of Stiefel manifold \citep{absil2009optimization} with tangent space 
\begin{equation}
    \label{OrthTangentSpace}
    T_{\mb W}\msf O(n;\bb R)\doteq\{\mb Z\mid \mb Z^*\mb W+\mb W^*\mb Z=\mb 0\},
\end{equation}
and the projection operation $\mc P_{T_{\mb W}\msf O(n;\bb R)}:\bb R^{n\times n}\to T_{\mb W}\msf O(n;\bb R)$ onto tangent space of $\msf O(n;\bb R)$ is defined as:
\begin{equation}
    \label{OrthProjection}  
    \mc P_{T_{\mb W}\msf O(n;\bb R)}(\mb Z) \doteq \mb Z - \frac{1}{2}\mb W(\mb {W}^*\mb Z+\mb Z^*\mb W) = \frac{1}{2}(\mb Z - \mb {WZ}^*\mb W).
\end{equation}
We use $\nabla_{\mb W}g(\mb W)$ to denote the gradient of $g(\mb W)$ w.r.t. $\mb W$ in $\bb R^{n\times n}$, and $\grad g(\mb W)$ to denote the Riemannian gradient of $g(\mb W)$ w.r.t. $\mb W$ on $T_{\mb W}\msf O(n;\bb R)$. Thus, we can formulate the Riemannian gradient of $\mb W$ on $T_{\mb W}\msf O(n;\bb R)$ as following:
\begin{equation}
    \label{RieGradient}
    \grad g(\mb W) = \mc P_{T_{\mb W}\msf O(n;\bb R)}(\nabla_{\mb W}g(\mb W)).
\end{equation}
The following proposition introduces the critical points of $g(\mb W)$ on $T_{\mb W}\msf O(n;\bb R)$.
\begin{proposition}
    \label{prop:L4OrthCriticalPoints}
    The critical points of $g(\mb W)$ on manifold $\msf O(n;\bb R)$ satisfies the following condition:
    \begin{equation}
        \label{eq:L4OrthCriticalPoints}
        (\mb W^{\circ3})^*\mb W =\mb W^* \mb W^{\circ3}.
    \end{equation}
\end{proposition}

\begin{proof}
    See \ref{proof:L4OrthCriticalPoints}.
\end{proof}
Therefore, $\forall \mb W \in \bb R^{n\times n}$ we can write critical points condition of $\ell^4$-norm over $\msf O(n;\bb R)$ as this following equations:
\begin{equation}
\label{prop:CriticalPoints}
    \begin{cases}
        (\mb W^{\circ3})^*\mb W = \mb W^*\mb W^{\circ3},\\
        \mb W^*\mb W = \mb I.
    \end{cases}
\end{equation}
Since the orthogonal group $\msf O(n;\bb R)$ is a continuous manifold in $\bb R^{n\times n}$ \citep{absil2009optimization,Hall2015Lie}, this indicates that critical points of $g(\mb W) = \norm{\mb W}{4}^4$ over the orthogonal group $\mb W\in \msf O(n;\bb R)$ has measure 0.

\begin{proposition}
    \label{prop:DiscreteCritPoints}
    All global maximizers of $\ell^4$-norm over the orthogonal group are isolated (nondegenerate) critical points.
\end{proposition}
\begin{proof}
    See \ref{proof:DiscreteCritPoints}. 
\end{proof}
We conjecture that the all local maximizers of the $\ell^4$-norm over $\msf O(n;\bb R)$ are global maximizers, as we will discuss in the next subsection.

\subsection{Relation between MSP and Projected Gradient Ascent (PGA)}
\label{subsec:MSPRelation}
Although the MSP algorithm over orthogonal group (Algorithm \ref{algo:SPOrthD}) and for dictionary learning (Algorithm \ref{algo:SPOrthDL}) has the same optimization procedure -- they all performs \textit{infinite} step size projected gradient ascent w.r.t. the objective function $\norm{\mb {AD}_o}{4}^4$ and $\norm{\mb {AY}}{4}^4$ respectively, we shall state their intrinsic difference clearly.   

In Proposition \ref{prop:L4MSPExp} and Proposition \ref{prop:GradHatfUnionConcentrationBound}, we can see that each iteration of the MSP algorithm for dictionary learning (Algorithm \ref{algo:SPOrthDL}) concentrates onto one step of PGA w.r.t. the objective $\norm{\mb {AD}_o}{4}^4$ with \textit{fixed step size} $\frac{1-\theta}{4\theta}$, while the MSP algorithm over the orthogonal group performs PGA with \textit{infinite} step size. We test PGA Algorithm \ref{algo:PGAL4MaxOrth} with different step size $\alpha$, as shown in Table \ref{tab:PGAOrth} for detail. We observe that:  

\begin{itemize}
    \item PGA with \textit{arbitrary} fixed step size find \textit{global maximizers} of \eqref{eq:L4MaxOrthClean};
    \item PGA converges \textit{faster} with \textit{larger} step size.
\end{itemize}
This experimental phenomenon supports the efficiency of the MSP algorithm -- it is indeed faster than any first order gradient method.

\begin{algorithm}[ht]
  \caption{PGA for $\ell^4$-Maximization over $\msf O(n;\bb R)$}\label{algo:PGAL4MaxOrth}
  \begin{algorithmic}[1]
    \State \textbf{Initialize} $\mb A_0\in \msf O(n; \bb R)$, step size $\alpha>0$ \Comment{Initialize $\mb A_0$, and $\alpha$ for iteration}
      \For{$t=0,1,...,T-1$}
        \State $\partial \mb A_t \doteq 4\mb A^{\circ3}$ \Comment{$\nabla_{\mb A}\norm{\mb A}{4}^4=4(\mb A)^{\circ3}$}
        \State $ \mb U \mb \Sigma \mb V^* = \text{SVD}\big(\mb A + \alpha\partial \mb A_t\big)$ 
        \State $\mb A_{t+1} = \mb U \mb V^*$\Comment{Project $\mb A_{t+1}$ onto $\msf O(n;\bb R)$}
      \EndFor
      \State \textbf{Output} $\mb A_{T}$ 
  \end{algorithmic}
\end{algorithm}

\begin{table}[ht]
    \centering
    \setlength{\tabcolsep}{2.5pt}
    \renewcommand{\arraystretch}{1.2}
    \begin{tabular}{|c|cccc|}
    \hline
    \multicolumn{1}{|c|}{} & \multicolumn{4}{c|}{Iterations} \\
    & $\alpha = 1$&  $\alpha = 10$ & $\alpha = 100$ &$\alpha = +\infty$  \\
    \hline
    $\msf O(5;\bb R)$ & 13 & 5 & 4 & 4\\
    \hline
    $\msf O(25;\bb R)$ & 23 & 7 & 5 & 5 \\
    \hline
    $\msf O(50;\bb R)$ & 35 & 10 & 8 & 6 \\
    \hline
    $\msf O(100;\bb R)$ & 63 & 12 & 9 & 9 \\
    \hline
    $\msf O(200;\bb R)$ & 70 & 14 & 11 & 9 \\
    \hline
    \end{tabular}
    \caption{Number of iteration for PGA (Algorithm \ref{algo:PGAL4MaxOrth}) on orthogonal group of different dimension $n$ to reach global maximizers, with the same initialization for each dimension $n$ and different step size $\alpha$. We directly apply the MSP Algorithm \ref{algo:SPOrthD} when $\alpha = +\infty$.}
    \label{tab:PGAOrth}
\end{table}

\subsection{Convergence Analysis}
We first introduce the properties of the critical points of the $\ell^4$-objective \eqref{eq:L4MaxOrthClean} in Proposition \ref{prop:MSPFixedPoint} and Proposition \ref{prop:ConvergenceToSaddle} will show that PGA with \textit{any fixed step size} $\alpha$ (even $\alpha = +\infty$) finds a critical points of \eqref{eq:L4MaxOrthClean}.
\begin{proposition}[Fixed Point of the MSP Algorithm]
    \label{prop:MSPFixedPoint}
    Given $\mb W\in\msf {SO}(n;\bb R)$, $\mb W$ is a fix point of the MSP Algorithm \ref{algo:SPOrthD} if and only if $\mb W$ is a critical point of the $\ell^4$-norm over $\msf {SO}(n;\bb R)$.
\end{proposition}
\begin{proof}
    See \ref{proof:MSPFixedPoint}.
\end{proof}

\begin{proposition}[Convergence of PGA with Arbitrary Step Size]
\label{prop:ConvergenceToSaddle}
Iterative PGA algorithm \ref{algo:PGAL4MaxOrth} with any fixed step size $\alpha>0$ ($\alpha$ can be $+\infty$ and PGA is equivalent to MSP Algorithm \ref{algo:SPOrthD} when $\alpha = +\infty$) finds a saddle point of optimization problem \eqref{eq:L4MaxOrthClean} 
\begin{equation*}
    \max_{\mb A\in \msf O(n;\bb R)}\norm{\mb A}{4}^4.
\end{equation*}
\end{proposition}
\begin{proof}   
    See \ref{proof:ConvergenceToSaddle}.
\end{proof}

In addition to Proposition \ref{prop:ConvergenceToSaddle}, the intuition for \textit{larger gradient converges faster} is due to the landscape of function $h(\mb A)$. As we can see in \eqref{eq:FunctionFDef}, when $\alpha$ decreases, the landscape of $h(\mb A)$ approaches becomes ``flat'', hence the convergence rate decreases. On the contrary, when $\alpha\to+\infty$, $h(\mb A)$ preserves the ``sharp'' curvature of $\norm{\mb A}{4}^4$, which yields the optimal convergence rate.

Although the function $g(\mb A) = \|\mb A\|_4^4$ may have many critical points, the signed permutation group $\msf{SP}(n)$ are the only global maximizers. As recent work has shown \citep{sun2015complete}, such discrete symmetry helps regulate the global landscape of the objective function and makes it amenable to global optimization. Indeed, we have observed through our extensive experiments that, under broad conditions, the proposed MSP algorithm always converges to the globally optimal solution (set), at a super-linear convergence rate.

We only give a local result on the convergence of the MSP algorithm in this paper.\footnote{We leave the study of ensuring global optimality and convergence to future work.} That is, when the initial orthogonal matrix $\mb A$ is ``close'' enough to a signed permutation matrix, the MSP algorithm converges to that signed permutation at a very fast rate. It is easy to verify the algorithm is permutation invariant. Hence w.l.o.g., we may assume the target signed permutation is the identity $\mb I$. 
\begin{theorem}[Cubic Convergence Rate around Global Maximizers]
\label{Thm:MSPLocalConvergence}
Given an orthogonal matrix $\mb A\in \msf O(n;\bb R)$, let $\mb A^\prime$ denote the output of the MSP Algorithm \ref{algo:SPOrthD} after one iteration: $\mb A^\prime = \mb {UV}^*$, where $\mb {U\Sigma V}^*=\text{SVD}(\mb A^{\circ3})$. If $\norm{\mb A-\mb I}{F}^2=\eps$, for $\eps<0.579$, then we have $\norm{\mb A^\prime - \mb I}{F}^2 < \norm{\mb A - \mb I}{F}^2$ and $\norm{\mb A^\prime - \mb I}{F}^2 < O(\eps^3)$.
\end{theorem}
\begin{proof}
    See \ref{proof:MSPLocalConvergence}.
\end{proof}

Theorem \ref{Thm:MSPLocalConvergence} shows that the MSP Algorithm \ref{algo:SPOrthD} achieves cubic convergence rate locally, which is much faster than any gradient descent methods. Our experiments in Section \ref{sec:Experiments} confirm this super-linear convergence rate for the MSP algorithms, at least in the deterministic case.

The above theorem only proves local convergence. As shown in section \ref{sec:experiment-DL}, We have observed in experiments that the algorithm actually converges {\em globally} under very broad conditions. We can see why this could be true in general from the special case when $n = 2$. Note that $\msf O(n,\bb R)$ is a disjoint manifold with two continuous parts $\msf {SO}(n,\bb R)$ and its reflection \citep{Hall2015Lie}, for convenience, we only consider $\msf {SO}(n,\bb R)$, which is a ``half'' of $\msf O(n,\bb R)$, with determinant $det(\mb A) = 1$. The next lemma shows the global convergence of our MSP algorithm in $\msf {SO}(2;\bb R)$.

\begin{proposition}[Global Convergence of the MSP Algorithm on $\msf {SO}(2;\bb R)$]
\label{prop:MSPGlobalConvn=2}
When $n=2$, if we parameterize our $\mb A_t\in \msf {SO}(2,\bb R)$ as the following:\footnote{The result of this lemma shows an update on the $\tan x$ function, which is a periodic function with period $\pi$, so we set $\theta\in[-\frac{\pi}{2},\frac{\pi}{2}]$ to avoid the periodic ambiguity, one can easily generalize this result to the case where $\theta\in[-\pi,-\frac{\pi}{2})\cup(\frac{\pi}{2},\pi]$, since the signs of $1$ won't affect our claim for pursuing a ``signed permutation'' matrix.}
\begin{equation}
\mb A_t = \begin{pmatrix}
       \cos\theta_t  &  -\sin\theta_t \\
       \sin\theta_t  &  \cos\theta_t 
    \end{pmatrix},\quad \forall \theta_t \in \Big[-\frac{\pi}{2},\frac{\pi}{2}\Big],\\
\end{equation}
then $\theta_t$ and $\theta_{t+1}$ satisfies the following relation
\begin{equation}
\theta_{t+1} = \tan^{-1}\big(\tan^3\theta_{t}\big).
\end{equation}
\end{proposition}

\begin{proof}
See \ref{proof:MSPGlobalConvn=2}.
\end{proof}

The previous proposition \ref{prop:MSPGlobalConvn=2} indicates that the MSP algorithm is essentially conducting the following iteration: 
    \begin{equation}
        \theta_{t+1} = \tan^{-1}\big(\tan^3\theta_t\big) \quad \text{or} \quad \tan \theta_{t+1} = \tan^3\theta_t.
    \end{equation}
See Figure \ref{fig:ConvergenceFunc} for a plot of this function. The iteration will converge superlinearly to the following values $\theta_\star = \pm \pi/2,0,\pm \pi$, which all correspond to signed permutation matrices in $\text{SP}(2)$. The fixed points of this iteration are $\theta = k\pi/4$, where $k\in\{-4,-3,-2,-1,0,1,2,3,4\}$. Among all these fixed points, the unstable ones are those with odd $k$, which correspond to Hadamard matrices. This phenomenon also justifies our experiments: Hadamard matrices are fixed point of the MSP iteration, but they are unstable and can be avoided with small random perturbation.

\begin{figure}[H]
    \centering
    \begin{subfigure}[t]{0.45\textwidth}
        \centering
        \includegraphics[width=\textwidth]{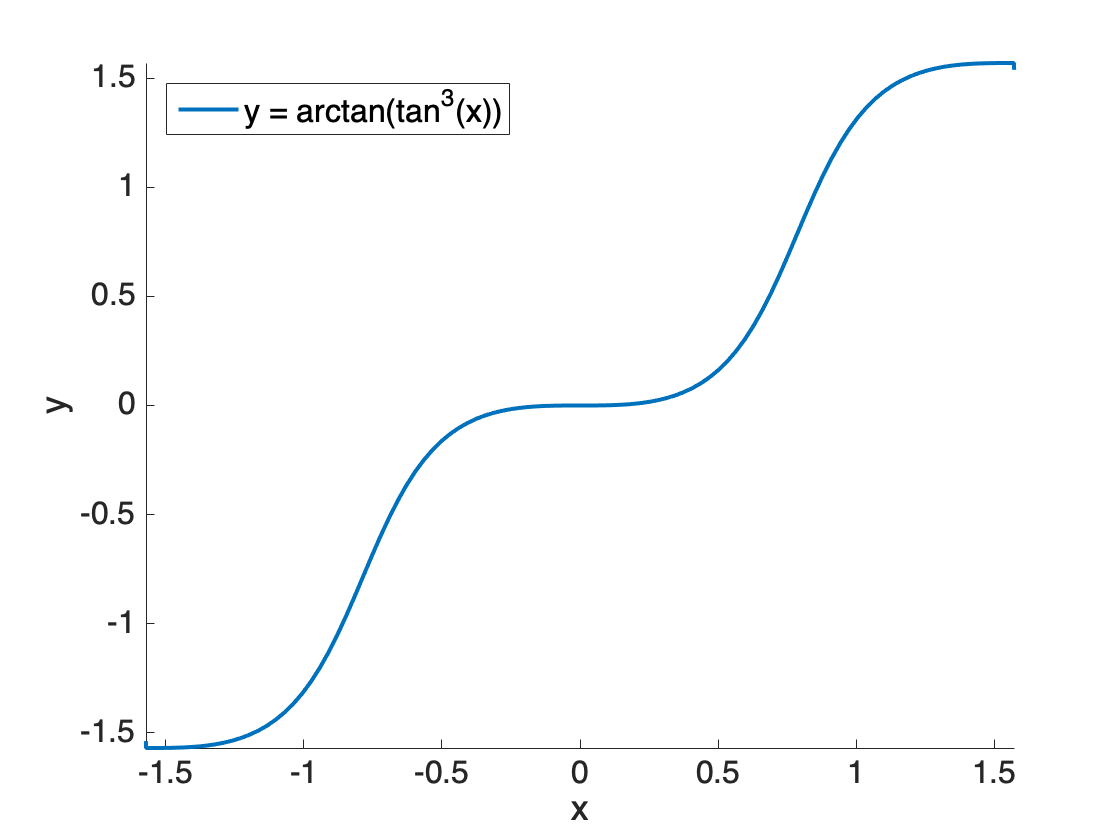}
        \caption{$x\in[-\frac{\pi}{2},\frac{\pi}{2}]$}
    \end{subfigure}
    ~
    \begin{subfigure}[t]{0.45\textwidth}
        \includegraphics[width=\textwidth]{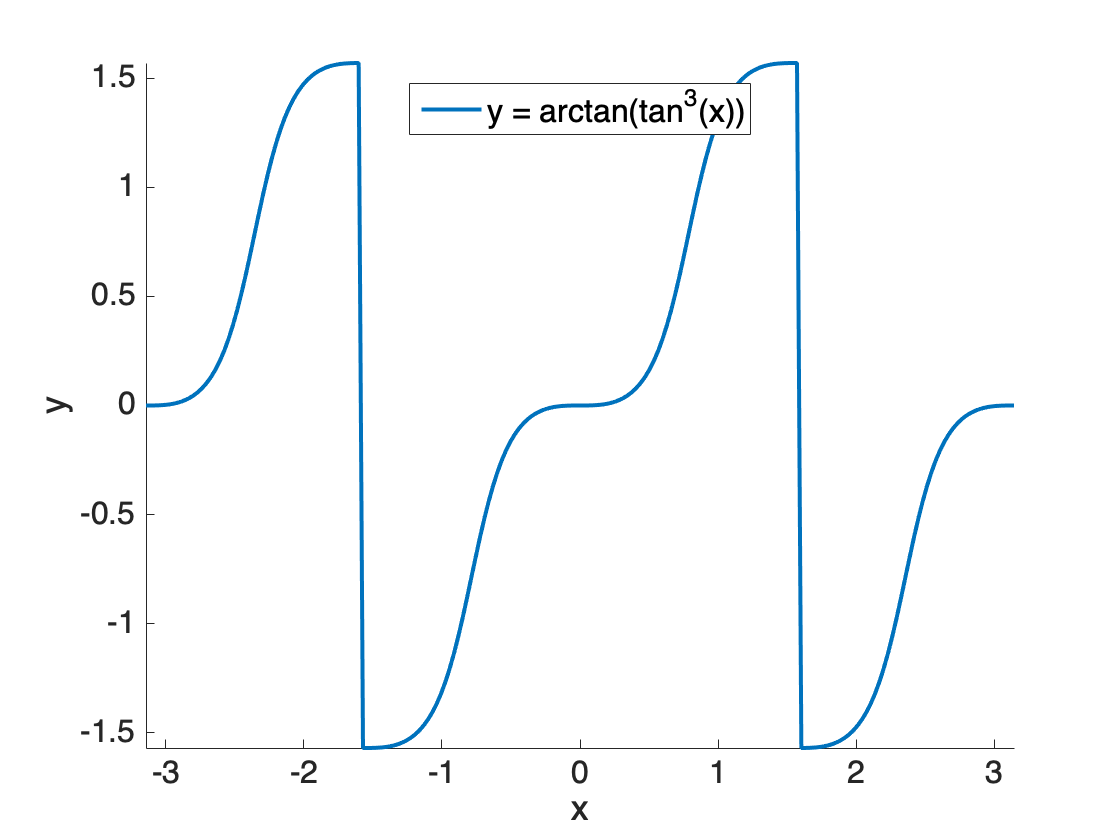}
    \caption{$x\in[-\pi,\pi]$}
    \end{subfigure}%
    \caption{Function $y = \tan^{-1}\big(\tan^3x\big)$, for different domain of $x$}
    \label{fig:ConvergenceFunc}
\end{figure}

\subsection{Generalized $\ell^{2k}$-Norm Maximization}
\label{subsec:General2kNorm}

One may notice that we can also promote sparsity over the orthogonal group by maximizing the $\ell^{2k}$-norm, $\forall k\geq2$ on the orthogonal group:\footnote{Note that when $2k=2$, $\norm{\mb A}{2}^2=n$ is a constant, so $\ell^{2k}$-norm only promotes sparsity when $2k\geq4$.}
\begin{equation}
    \max_{\mb {A}} \norm{\mb A}{2k}^{2k},\quad\st\quad \mb A\in \msf O(n;\bb R).
\end{equation}
In fact, the resulting algorithm would have a higher rate of convergence for the deterministic case, as the stretching with the power $(\cdot)^{\circ 2k-1}$ sparsifies the matrix more significantly with a larger $k$. See Section \ref{sec:2k-experiments} for experimental verification.
The following example provides one run of MSP algorithm when $2k=10$.
\begin{example}[One Run of MSP Algorithm with $2k=10$] 
\label{example:SPAlgo10}
This example shows how a random orthogonal matrix $\mb A_0 \in \msf O(n;\bb R)$ evolves into a signed permutation matrix, using $\ell^{10}$-norm:
\begin{footnotesize}
\begin{equation*}
\begin{array}{llll}
    &\mb A_0 = 
     \begin{pmatrix}
       -0.6142  &  0.3943 &  0.6836 \\
       -0.2039  &  0.7575 & -0.6201 \\
        0.7623  &  0.5203 &  0.3849
     \end{pmatrix}        
    &\xrightarrow{\text{stretching}}
    &\mb A_0^{\circ9}=
    \begin{pmatrix}
       -0.0124  &  0.0002  &  0.0326 \\
       -0.0000  &  0.0821  & -0.0136 \\
        0.0870  &  0.0028  &  0.0002
    \end{pmatrix}\\
    \xrightarrow{\text{projection}}
    &\mb A_1=
    \begin{pmatrix}
       -0.2657  &  0.0908  &  0.9598 \\
       -0.0150  &  0.9950  & -0.0983 \\
        0.9639  &  0.0406  &  0.2630
    \end{pmatrix}
    &\xrightarrow{\text{stretching}}
    &\mb A_1^{\circ9}=
    \begin{pmatrix}
       -0.0000  &  0.0000  &  0.6910 \\
       -0.0000  &  0.9562  & -0.0000 \\
        0.7185  &  0.0000  &  0.0000
    \end{pmatrix}\\
    \xrightarrow{\text{projection}}
    &\mb A_2=
    \begin{pmatrix}
        0.0000  & -0.0000  &  1.0000 \\
        0.0000  &  1.0000  & -0.0000 \\
        1.0000  & -0.0000  &  0.0000
    \end{pmatrix}
    &\xrightarrow{\text{output}}
    &\mb A_2=
    \begin{pmatrix}
         0   &  0  &  1 \\
         0   &  1  &  0 \\
         1   &  0  &  0
    \end{pmatrix}.
\end{array}
\end{equation*}
\end{footnotesize}
Comparing with Example \ref{example:SPAlgo4}, where $\ell^4$-norm MSP algorithm takes 6 iterations, the $\ell^{10}$-norm MSP finds a signed permutation matrix in only 2 iterations. 
\end{example} 
In fact, one can show that if $2k \to \infty$, the corresponding MSP algorithm converges with {\em only one} iteration for the deterministic case!

So why don't we choose $\ell^\infty$-norm instead? Notice that the above behavior is for the expectation of the objective function we actually care here. For dictionary learning, the object function $\hat{f}(\mb A,\mb Y)$ depends on the random samples $\mb Y$. Both $\hat{f}(\mb A,\mb Y)$ and $\nabla_{\mb A}\hat{f}(\mb A,\mb Y)$ concentrate on their expectations only when $p$ is large enough. For the same error threshold, if we choose larger $k$, then we will need much larger sample size $p$ for $\hat{f}(\mb A,\mb Y)$, $(\mb {AY})^{\circ2k-1}\mb Y^*$ to concentrate on their expectation $f(\mb A)$, $\bb E [(\mb {AY})^{\circ2k-1}\mb Y^*]$ respectively. As we will see in Section \ref{sec:2k-experiments} from experiments, the choice of $2k=4$ seems to be the best in terms of balancing these two contending factors of convergence rate and sample size.

\section{Experimental Verification}
\label{sec:Experiments}
\subsection{$\ell^4$-Norm Maximization over the Orthogonal Group}
First, to have some basic ideas about how fast and smoothly the proposed MSP algorithm maximizes $\ell^4$-norm. 
Figure \ref{fig:SPOrthOneRun} shows one run of the MSP Algorithm \ref{algo:SPOrthD} that maximizes $\norm{\mb A}{4}^4/n$. It reaches global maxima in less than 10 iterations for $n=50$ and $n = 100$.
\begin{figure}[t]
    \centering
    \begin{subfigure}[t]{0.4\textwidth}
        \centering
        \includegraphics[width=\textwidth]{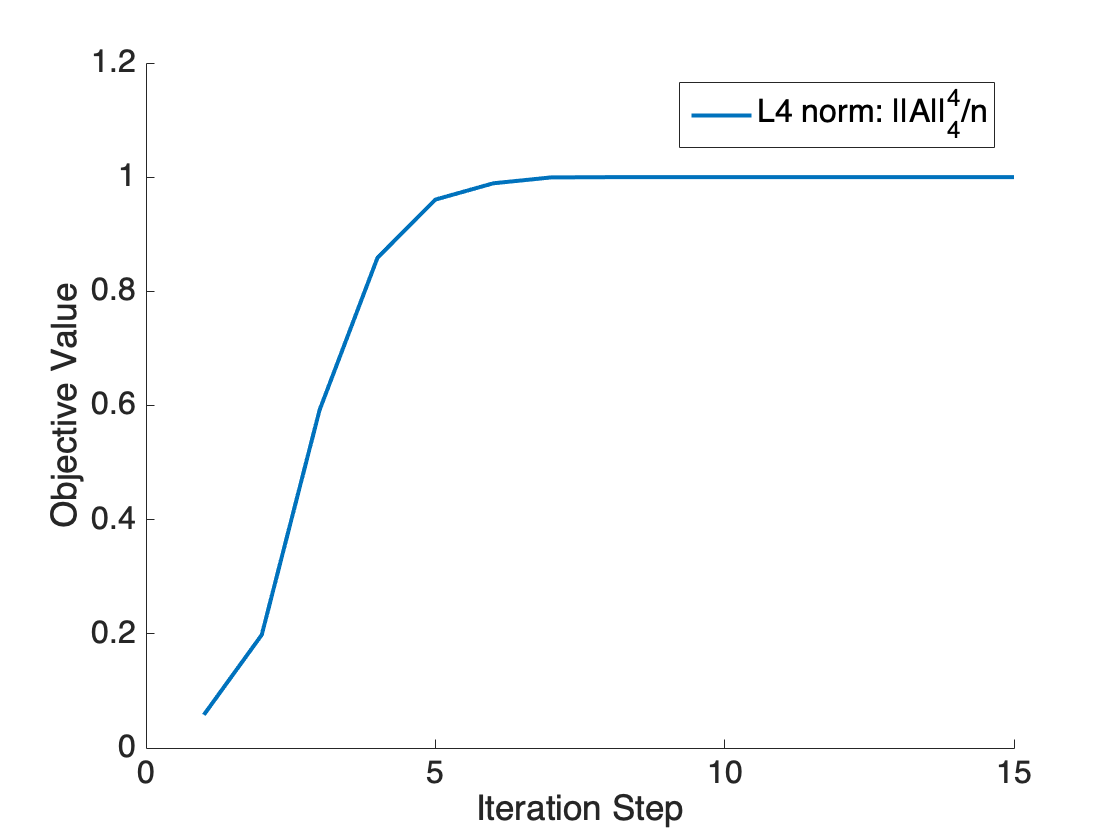}
        \caption{$n=50$}
    \end{subfigure}%
    ~ 
    \begin{subfigure}[t]{0.4\textwidth}
        \centering
        \includegraphics[width=\textwidth]{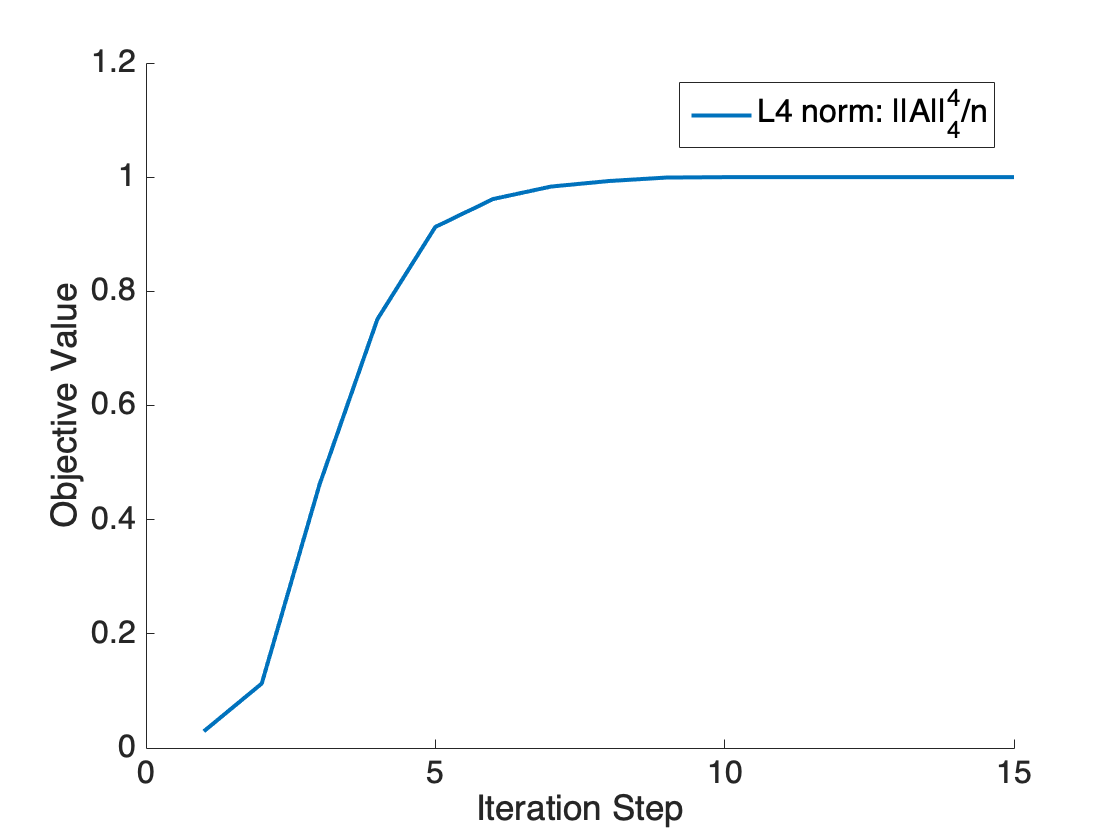}
        \caption{$n=100$}
    \end{subfigure}
    \caption{Normalized $\norm{\mb {A}}{4}^4$ in one run of MSP Algorithm \ref{algo:SPOrthD} for maximizing $\|\mb A\|_4^4$ on $\msf O(n;\bb R)$}
    \label{fig:SPOrthOneRun}
\end{figure}

\begin{figure}[t]
    \centering
    \begin{subfigure}[t]{0.4\textwidth}
        \centering
        \includegraphics[width=\textwidth]{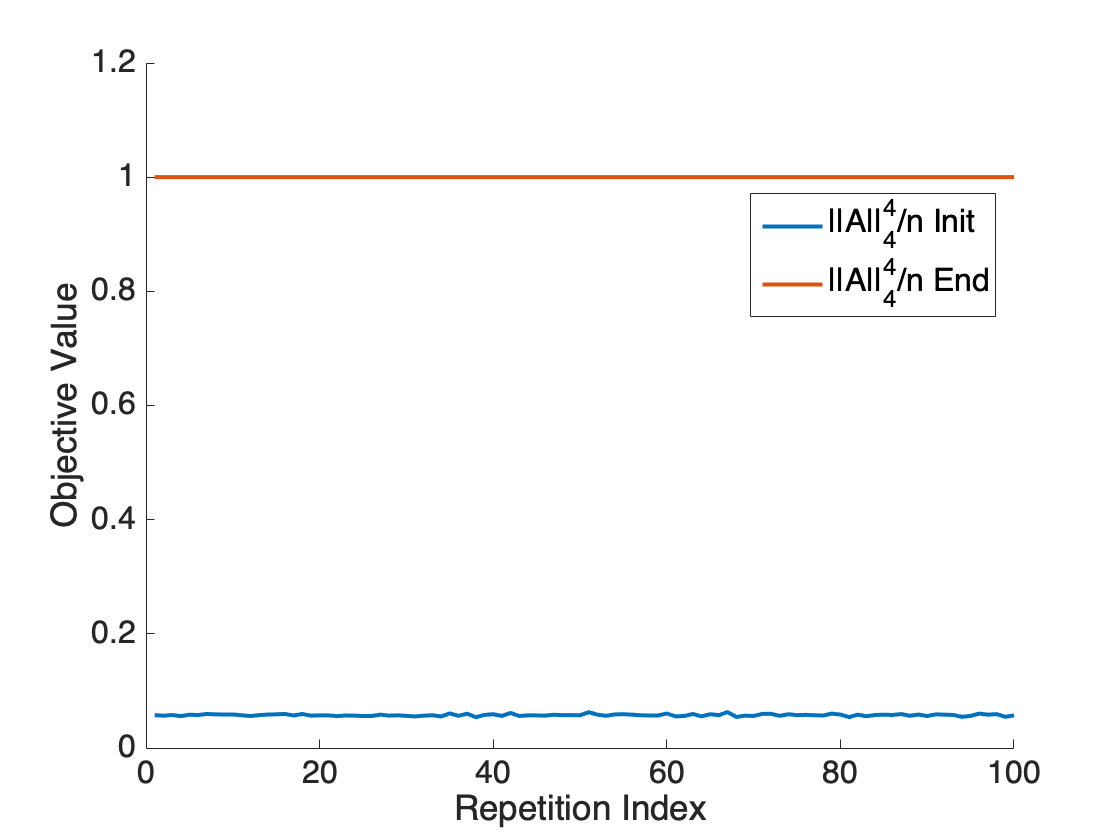}
        \caption{$n=50$}
    \end{subfigure}%
    ~ 
    \begin{subfigure}[t]{0.4\textwidth}
        \centering
        \includegraphics[width=\textwidth]{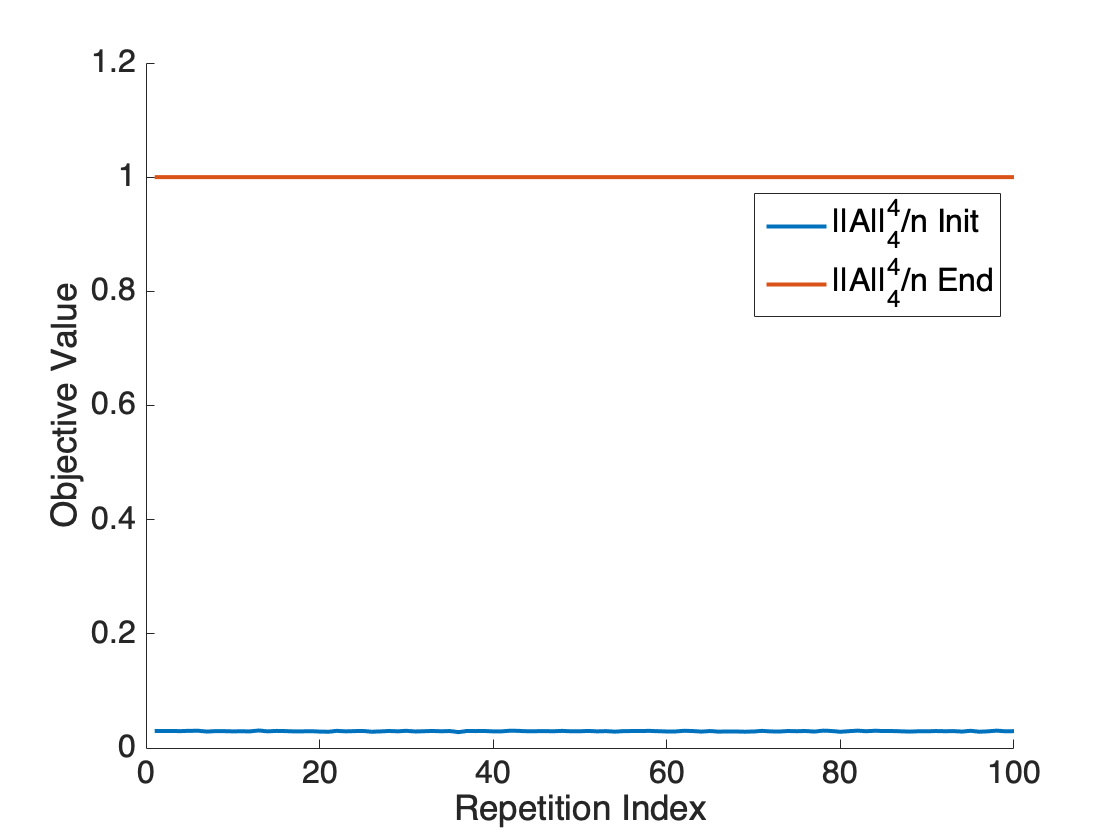}
        \caption{$n=100$}
    \end{subfigure}
    \caption{Normalized $\norm{\mb {A}}{4}^4/n$ for 100 trials of the MSP Algorithm \ref{algo:SPOrthD} on $\msf O(n;\bb R)$}
    \label{fig:SPOrth}
\end{figure}

Next we run 100 trials of the algorithm to maximize $\norm{\mb A}{4}^4/n$ on $\msf O(n;\bb R)$ with random initialization. Figure \ref{fig:SPOrth} shows that all 100 trials of the MSP algorithm reach the global maximum $1$. According to Lemma \ref{lemma:G(A)bound}, this indicates that all 100 trials (in different dimension $n=50,100$) converge to sign permutation matrices in $\msf{SP}(n)$.

\subsection{Dictionary Learning via MSP}\label{sec:experiment-DL}
Figure \ref{fig:SPDLOneRun}(a) presents one trial of the proposed MSP Algorithm \ref{algo:SPOrthDL} for dictionary learning with $\theta = 0.3$, $n=50$, and $p=20,000$. The result corroborates with statements in Lemma \ref{lemma:HatfUnionConcentrationBound} and Lemma \ref{lemma:orthproperty}: Maximizing $\hat{f}(\mb A,\mb Y)$ is largely equivalent to optimizing $g(\mb {AD}_o)$, and both values reach global maximum at the same time. Meanwhile, this result also shows our MSP algorithm is able to find the global maximum at ease, since $g(\mb {A D}_o)$ reaches its maximal value 1 (with minor errors) by maximizing $\hat{f}(\mb A,\mb Y)$. In Figure \ref{fig:SPDLOneRun}(b), we test the MSP Algorithm \ref{algo:SPOrthDL} in higher dimension $n=100,p=40,000,\theta=0.3$. In both cases, our algorithm is surprisingly efficient: It only takes around 20 iterations to recover the ground truth dictionary. 

\begin{figure}[ht]
    \centering
    \begin{subfigure}[t]{0.4\textwidth}
        \centering
        \includegraphics[width=\textwidth]{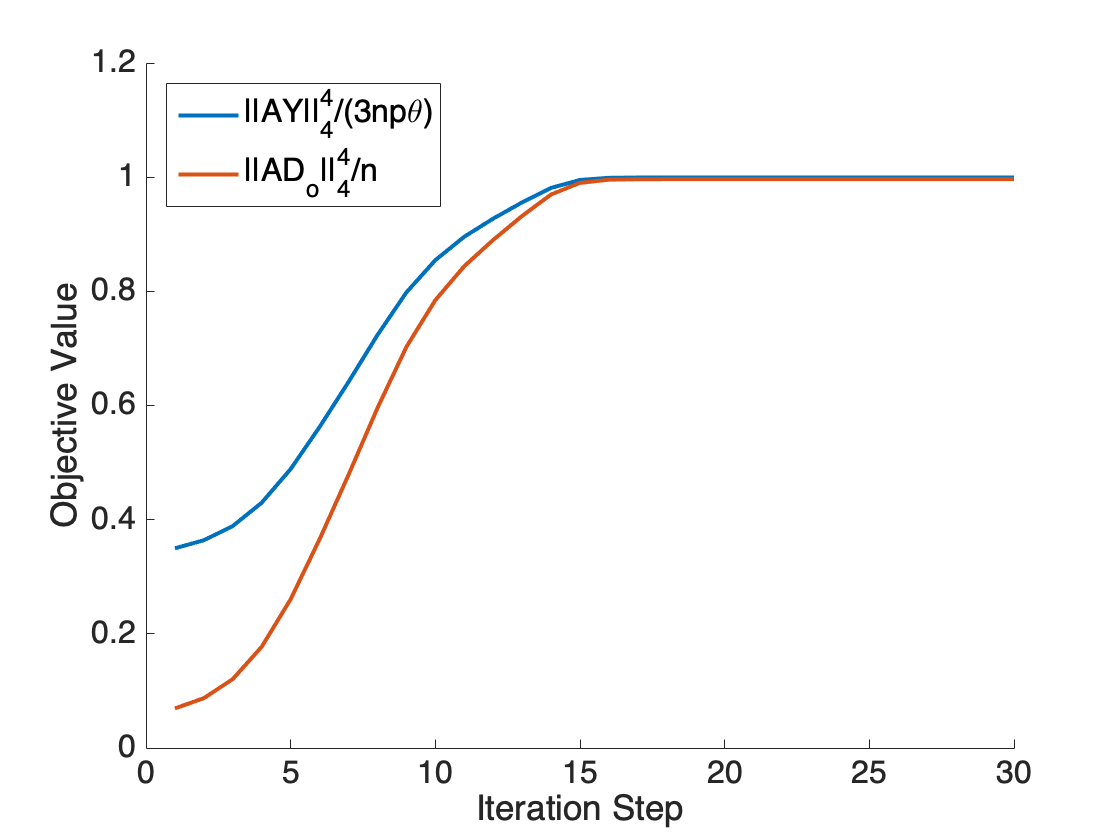}
        \caption{$n=50,p=20,000,\theta=0.3$}
    \end{subfigure}%
    ~ 
    \begin{subfigure}[t]{0.4\textwidth}
        \centering
        \includegraphics[width=\textwidth]{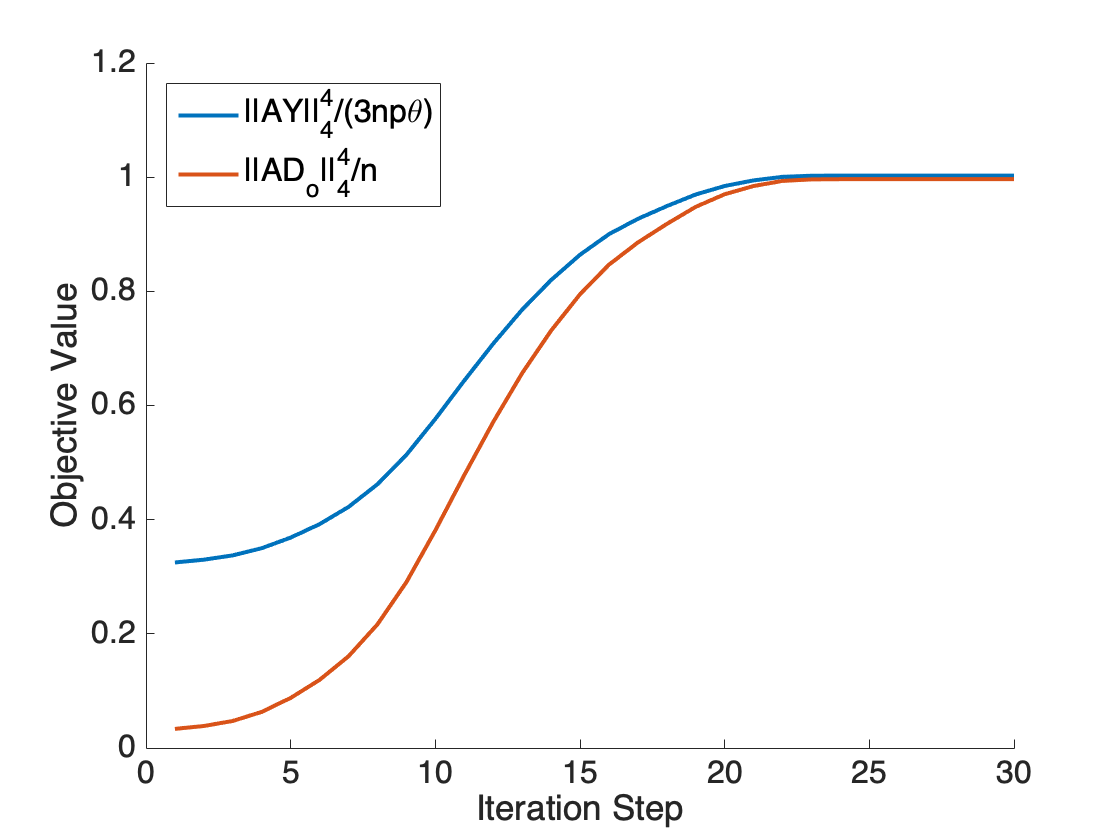}
        \caption{$n=100,p=40,000,\theta=0.3$}
    \end{subfigure}
    \caption{Normalized objective value $\norm{\mb {AD}_o}{4}^4/n$ and $\norm{\mb {AY}}{4}^4/3np\theta$ for individual trials of the MSP Algorithm \ref{algo:SPOrthDL}, with different parameters $n,p,\theta$. According to Lemma \ref{lemma:G(A)bound}, $g(\mb {AD}_o)/n$ reaches 1 indicates successful recovery for $\mb D_o$. This experiment shows the MSP algorithm finds global maxima of $\hat{f}(\mb A,\mb Y)$ thus recovers the correct dictionary $\mb D_o$.}
    \label{fig:SPDLOneRun}
\end{figure}

In Figure \ref{fig:SPDL}, we run the MSP Algorithm \ref{algo:SPOrthDL} for 100 random trials with the settings $n= 50, p = 20,000, \theta = 0.3$ and $n=100,p=40,000,\theta=0.3$. Among all 100 trials, $g(\mb {AD}_o)$ achieve the global maximal value (within statistical errors) via optimizing $\hat{f}(\mb A,\mb Y)$ in less than $30$ iterations. This experiment seems to support a conjecture: Within conditions of this experiment, the MSP algorithm recovers the globally optimal dictionary.

\begin{figure}[ht]
    \centering
    \begin{subfigure}[t]{0.4\textwidth}
        \centering
        \includegraphics[width=\textwidth]{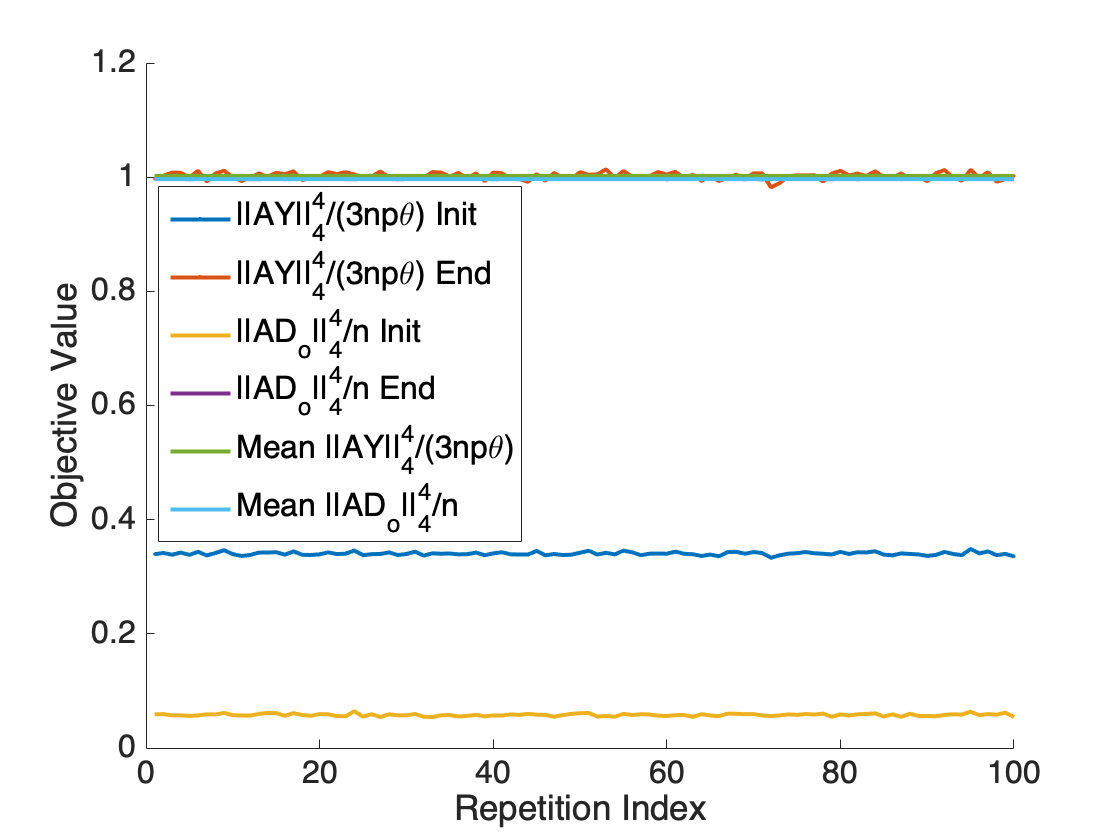}
        \caption{$n=50,p=20,000,\theta=0.3$, w/ initial values}
    \end{subfigure}%
    ~ 
    \begin{subfigure}[t]{0.4\textwidth}
        \centering
        \includegraphics[width=\textwidth]{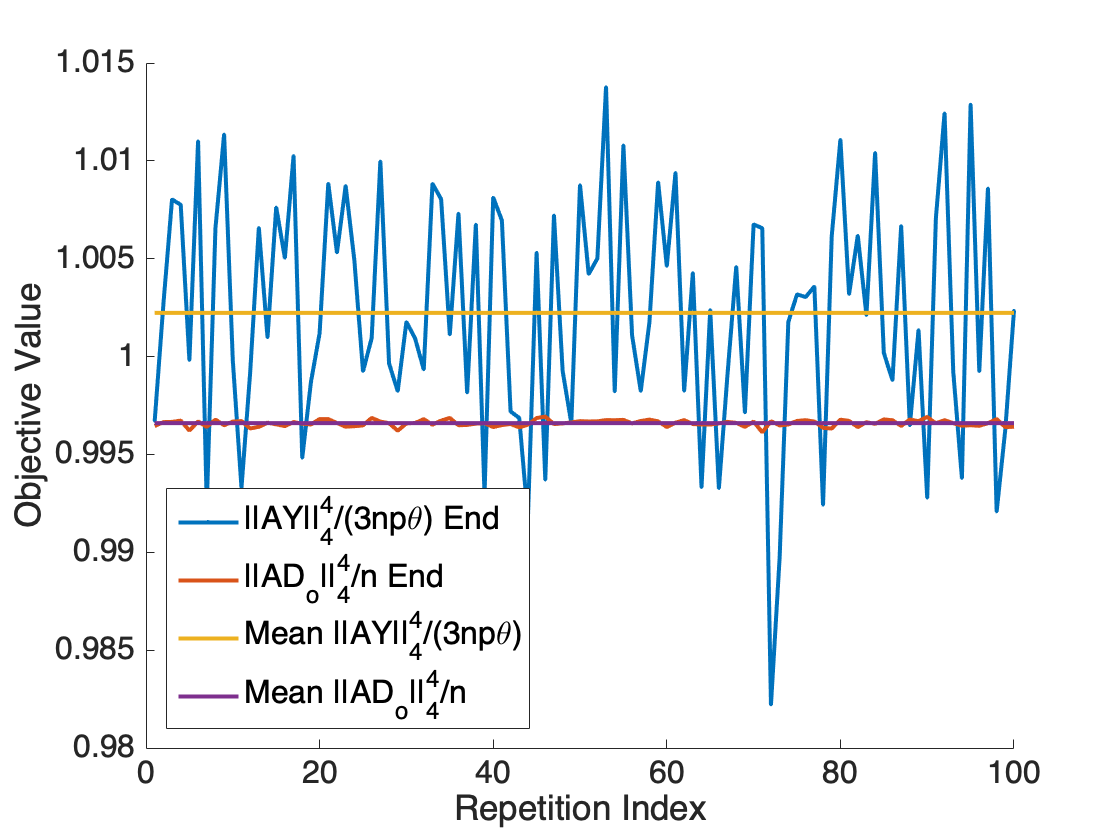}
        \caption{$n=50,p=20,000,\theta=0.3$, w/o initial values}
    \end{subfigure}
    \begin{subfigure}[t]{0.4\textwidth}
        \centering
        \includegraphics[width=\textwidth]{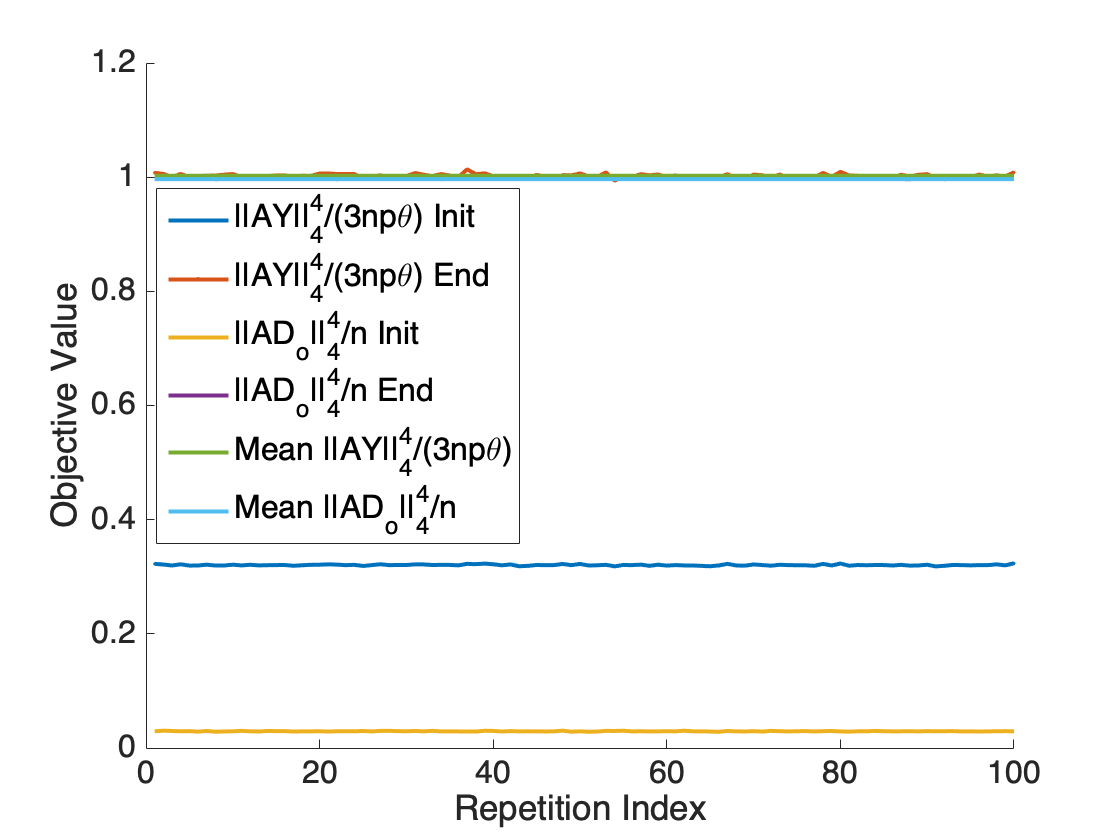}
        \caption{$n=100,p=40,000,\theta=0.3$, w/ initial values}
    \end{subfigure}%
    ~ 
    \begin{subfigure}[t]{0.4\textwidth}
        \centering
        \includegraphics[width=\textwidth]{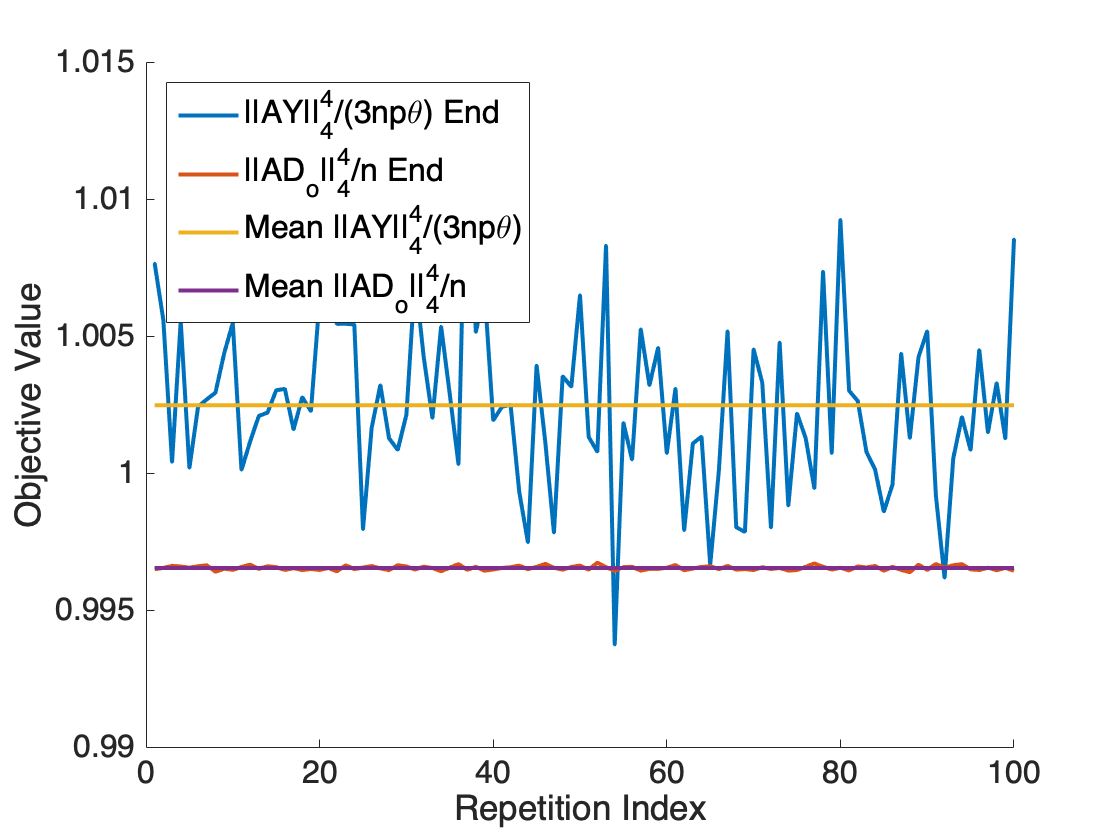}
        \caption{$n=100,p=40,000,\theta=0.3$, w/o initial values}
    \end{subfigure}
    \caption{Normalized initial and final objective values of $\norm{\mb {AD}_o}{4}^4/n$ and  $\norm{\mb {AY}}{4}^4/3np\theta$ for $100$ trials of the MSP Algorithm \ref{algo:SPOrthDL}, with $n = 50$ and $n=100$. Both objectives converge to 1 (with minor errors) for all $100$ trials. }
    \label{fig:SPDL}
\end{figure}

\subsection{$\ell^{2k}$-Norm Maximization over the Orthogonal Group}\label{sec:2k-experiments}
In Figure \ref{fig:SPOrder2k}, we conduct experiments to support the choice of $\ell^4$-norm. Figure \ref{fig:SPOrder2k}(a) shows that for the deterministic case, the MSP Algorithm \ref{algo:SPOrthD} finds signed permutation matrices faster with higher order $\ell^{2k}$-norm. But Figure \ref{fig:SPOrder2k}(b) indicates that as the order $2k$ increases, much more samples are needed by Algorithm \ref{algo:SPOrthDL} to achieve the same estimation error: $p$ grows drastically as $k$ increases. Hence, among all these sparsity-promoting norms ($\ell^{2k}$-norm), the $\ell^4$-norm strikes a good balance between sample size and convergence rate.\footnote{Later works  \citep{shen2020complete,xue2020blind} have shown that $\ell^3$-norm maximization also has the same sparsifying effect.}

\begin{figure}[ht]
    \centering
    \begin{subfigure}[t]{0.4\textwidth}
        \centering
        \includegraphics[width=\textwidth]{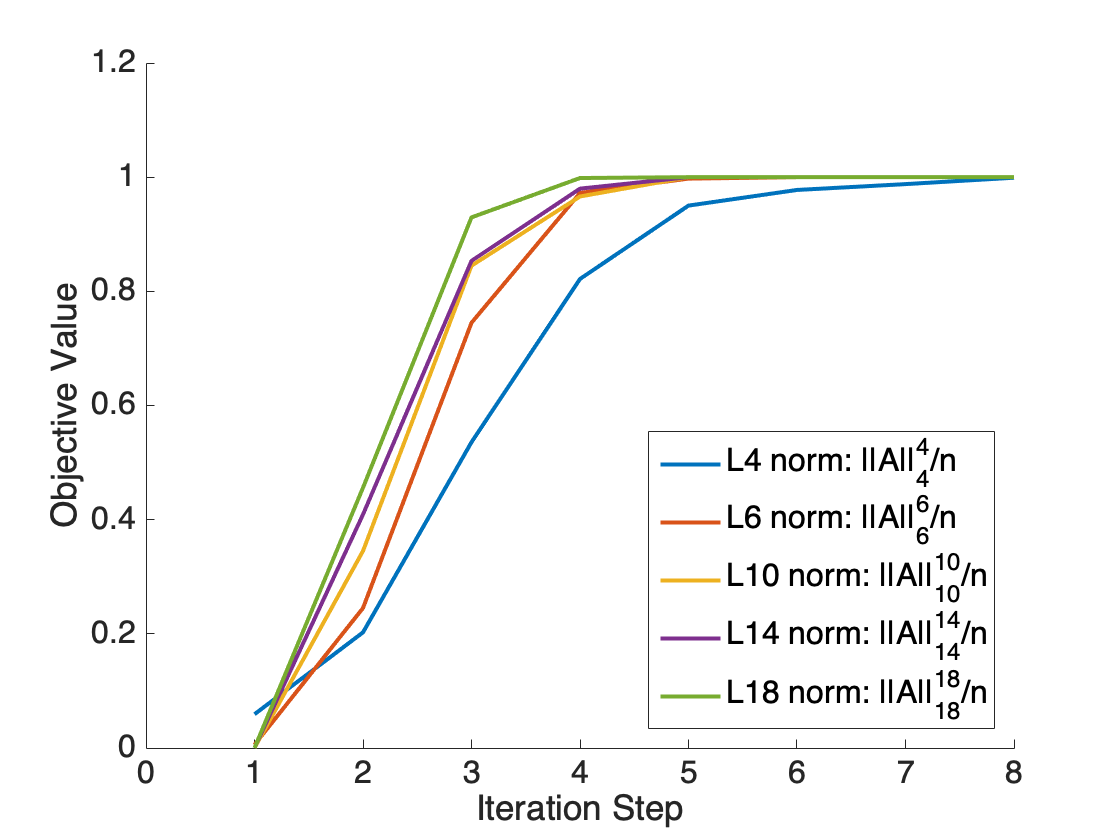}
        \caption{Convergence plots of the MSP Algorithm \ref{algo:SPOrthD} for the deterministic case, with the same initialization.}
    \end{subfigure}%
    ~ 
    \begin{subfigure}[t]{0.4\textwidth}
        \centering
        \includegraphics[width=\textwidth]{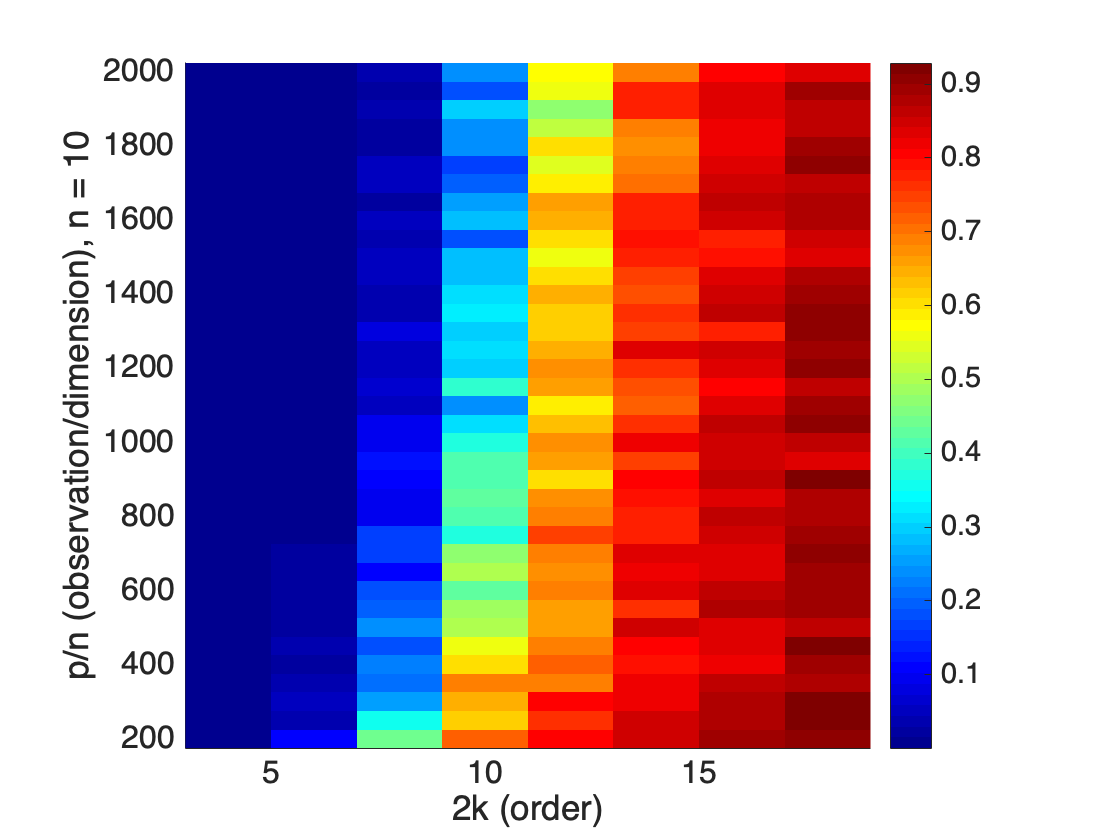}
        \caption{Average normalized error of MSP Algorithm \ref{algo:SPOrthDL} among 20 trials, varying $k$ and $p$, with $n = 10$ fixed.}
    \end{subfigure}
    \caption{Use different $\ell^{2k}$-norms for Algorithm \ref{algo:SPOrthD} and Algorithm \ref{algo:SPOrthDL}.}
    \label{fig:SPOrder2k}
\end{figure}

\subsection{Phase Transition Plots for Working Ranges of the MSP Algorithm}
\label{subsec:MSPPhaseTransition}
Encouraged by previous experiments, we conduct more extensive experiments of the MSP Algorithm \ref{algo:SPOrthDL} in broader settings to find its working range: 1) Figure \ref{fig:SPPR}(a) shows the result of varying the sparsity level $\theta$ (from 0 to 1) and sample size $p$ (from 500 to 50,000) with a fixed dimension $n=50$; 2) Figure \ref{fig:SPPR}(b) and (c) show results of changing dimension $n$ (from 10 to 1,000) and sample size $p$ (from 1,000 to 100,000) at a fixed sparsity level $\theta=0.5$. Notice that all figures demonstrate a clear phase transition for the working range. 

It is somewhat surprising to see in Figure \ref{fig:SPPR}(a) that the MSP algorithm is able to recover the dictionary correctly up to the sparsity level of $\theta \approx 0.6$ if $p$ is large enough, which almost doubles the best existing theoretical guarantee given in \cite{bai2018subgradient,sun2015complete}. We shall note that the inconsistency between the non linear phase transition curve in Figure \ref{fig:SPPR}(a) and the sample complexity $p=\Omega(\theta n^2 \ln n/\eps^2)$ is due to the constant $C>\frac{4}{3\theta(1-\theta)}$ in Theorem \ref{Thm:MainResult}.

Figure \ref{fig:SPPR}(b) and (c) show the working range for varying $n, p$ with a fixed $\theta=0.5$. Figure \ref{fig:SPPR}(b) is for a smaller range of $n$ (from 10 to 100) and Figure \ref{fig:SPPR}(c) for a larger range of $n$ (from 100 to 1,000). Figures (b) and (c) imply that the required sample size $p$ for the algorithm to succeed seems to be quadratic in the dimension $n$: $p=\Omega(n^2)$, which corresponds to our statistical results ($p=\Omega(\theta n^2\ln n)$) in Theorem \ref{Thm:MainResult} and Proposition \ref{prop:GradHatfUnionConcentrationBound}. This empirical bound is significantly better than the best theoretical bounds given in \cite{bai2018subgradient,sun2015complete} for the sample size required to ensure success. Similar observations have been reported in \cite{schramm2017fast,bai2018subgradient}.

\begin{figure}[ht]
    \centering
    \begin{subfigure}[t]{0.33\textwidth}
        \centering
        \includegraphics[trim={0.5cm 0cm 1cm 0.5cm},clip,width=\textwidth]{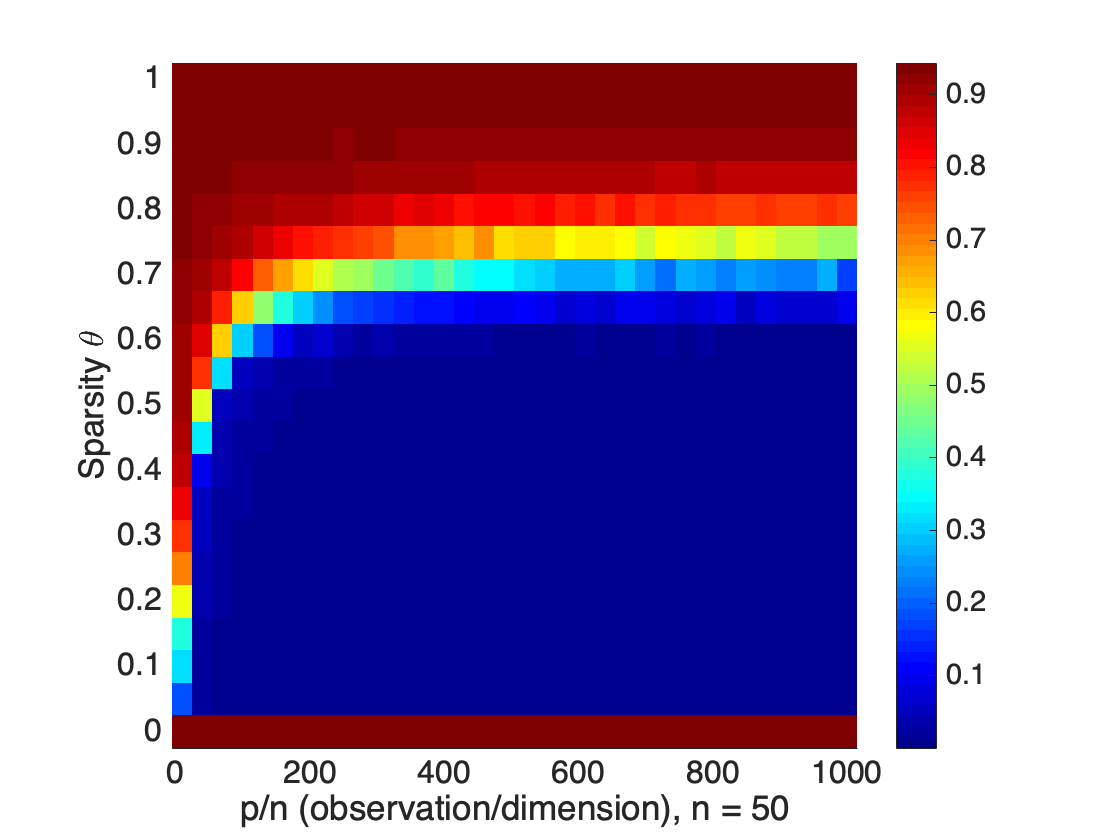}
        \caption{Changing $\theta$ from 0 to 1 and $p$ from 500 to 10,000, $n=50$.}
    \end{subfigure}%
    ~
    \begin{subfigure}[t]{0.33\textwidth}
        \centering
        \includegraphics[trim={0.5cm 0cm 1cm 0.5cm},clip,width=\textwidth]{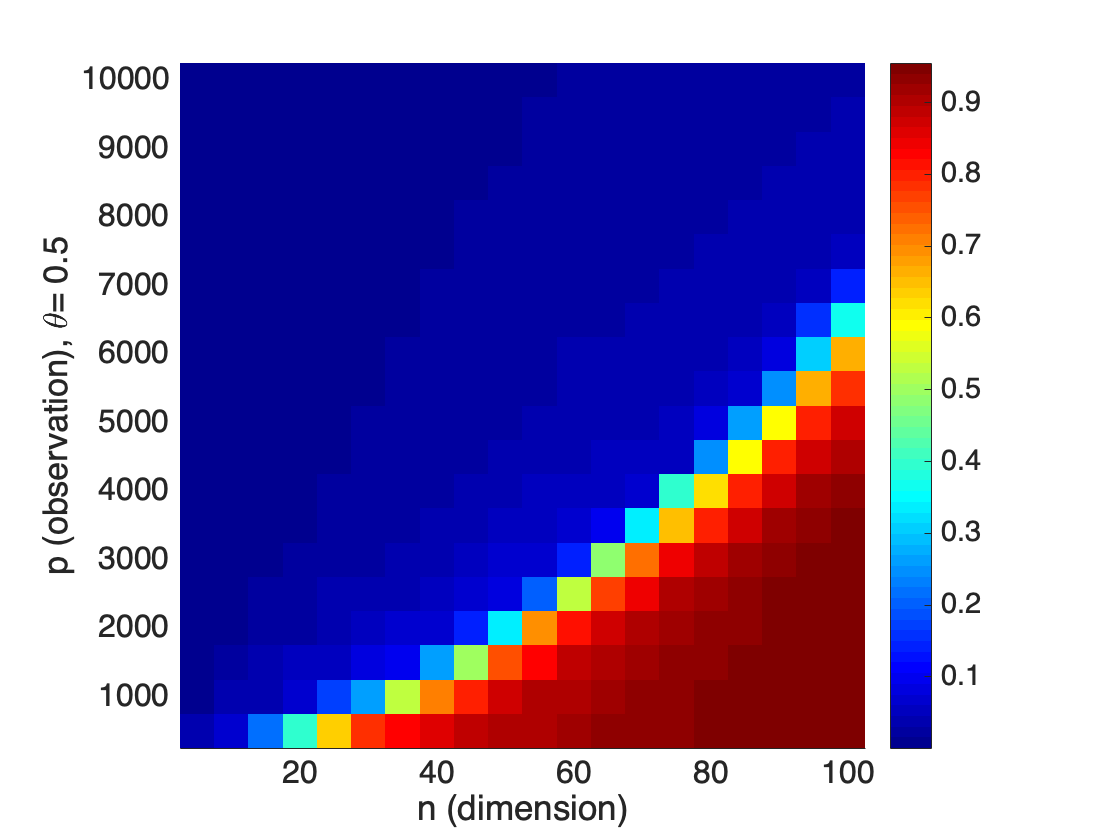}
        \caption{Changing $n$ from 10 to 100 and $p$ from $10^3$ to $10^4$, $\theta=0.5$.}
    \end{subfigure}%
    ~ 
    \begin{subfigure}[t]{0.33\textwidth}
        \centering
        \includegraphics[trim={1cm 0cm 1cm 0.5cm},clip,width=\textwidth]{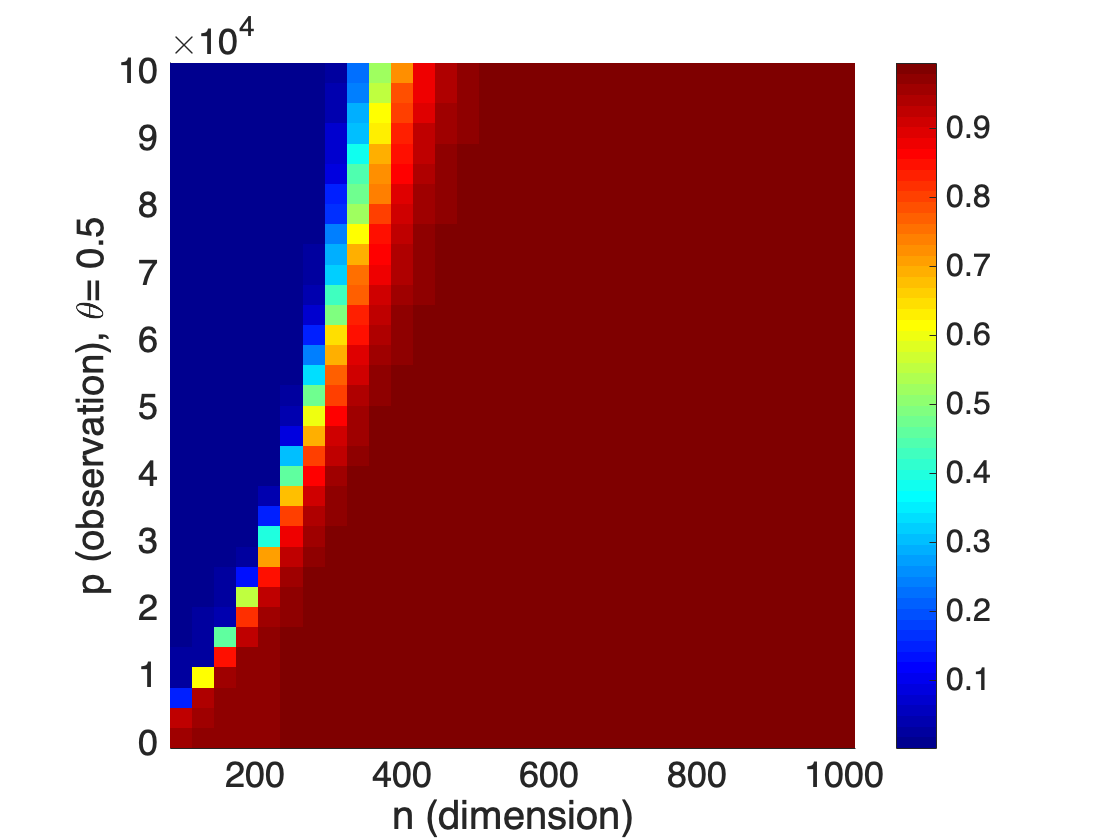}
        \caption{Changing $n$ from 100 to 1,000 and $p$ from $10^4$ to $10^5$, $\theta=0.5$.}
    \end{subfigure}
    \caption{Phase transition of average normalized error $|1-\norm{\mb {AD}_o}{4}^4/n|$ of 10 random trials for the MSP Algorithm \ref{algo:SPOrthDL} in different settings: {\bf (a):} Varying $\theta$, $p$ with fixed $n$; {\bf (b), (c): } Varying $n$, $p$ with fixed $\theta$. Red area indicates large error and blue area small error.}
    \label{fig:SPPR}
\end{figure}
\subsection{Comparison with Prior Arts}
Table \ref{tab:Comparison} compares the MSP method with the KSVD \citep{aharon2006k}, the SPAMS dictionary learning package \cite{jenatton2010proximal}, and the latest subgradient method \citep{bai2018subgradient} for different choices of $n,p$ under the same sparsity level $\theta = 0.3$. As one may see, our algorithm is significantly faster than the other algorithms in all trials. Further more, our algorithm has the potential for large scale experiments: It only takes 374.2 seconds to learn a $400\times 400$ dictionaries from $160,000$ samples. While the previous algorithms either fails to find the correct dictionary or is barely applicable. Within statistical errors, our algorithm gives slightly smaller values for $\norm{{\mb A\mb D}_o}{4}^4/n$ in some trials. But the subgradient method \citep{bai2018subgradient} uses information of the ground truth dictionary $\mb D_o$ in their stopping criteria. Our MSP algorithm removes this dependency with only mild loss in accuracy. We shall note that traditional $\ell^1$-minimization based dictionary learning is more accurate than the $\ell^4$-maximization based method, since $\ell^1$-minimization will decrease small entries to 0 while the $\ell^4$-maximization tolerate small noise. Meanwhile, the numerical implementation also affects the experimental performance -- the error of SPAMS is consistently better than the Subgradient method in terms of both accuracy and speed. 

Although the convex relaxation based on sum-of-square SDP hierarchy \citep{Barak-2015,ma2016polynomial,schramm2017fast} has theoretical guarantee for correctness under some specific statistical model, we do not include experimental comparison with the sum-of-square method only because the SoS method incurs high computational complexity in solving large scale SDP programmings or tensor decomposition problems. In fact, on the same computing devices, sum-of-square methods cannot even solve the smallest dictionary learning problem ($n=25$) in Table \ref{tab:Comparison} due to memory and computation complexity.

\begin{table}[ht]
    \centering
    \setlength{\tabcolsep}{2.5pt}
    \renewcommand{\arraystretch}{1.2}
    \begin{tabular}{|ccc|cc|cc|cc|cc|}
    \hline
    \multicolumn{3}{|c|}{}& \multicolumn{2}{c|}{KSVD}&
    \multicolumn{2}{c|}{SPAMS}&
    \multicolumn{2}{c|}{Subgradient} & \multicolumn{2}{c|}{MSP (Ours)} \\
    \hline $n$&$p$ $(\times 10^4)$  & $\theta$ & Error & Time & Error & Time & Error & Time & Error & Time  \\
    \hline
    25 & 1& 0.3&  12.16\% & 64.5s &0.02\% &5.53s & 0.25\% & 9.2s & 0.35\% & \textbf{0.2s} (15 iter)\\
    50 & 2& 0.3& 7.79\% & 165.4s & 0.02\% &23.4s & 0.27\% & 107.0s & 0.34\% & \textbf{1.3s} (20 iter)\\
    100 & 4& 0.3& 8.34\% & 569.4s &0.02\% & 365.9s & 0.27\% & 2807.9s & 0.35\% & \textbf{7.2s} (25 iter)\\
    200 & 8& 0.3&  8.46\% & 1064.0s & 0.02\%& 5420.8s & N/A & >12h & 0.35\% & \textbf{56.0s} (40 iter)\\
    400 & 16& 0.3&  11.94\% & 4646.3s & N/A & > 12h & N/A & >12h & 0.35\% & \textbf{494.4s} (60 iter) \\
    \hline
    \end{tabular}
    \caption{Comparison experiments with KSVD \citep{aharon2006k}, SPAMS \cite{jenatton2010proximal}, and Subgradient method \citep{bai2018subgradient} in different trials of dictionary learning: (a) $n=25,p=1\times 10^4,\theta = 0.3$; (b) $n=50,p=2\times 10^4,\theta = 0.3$; (c) $n=100,p=4\times 10^4,\theta = 0.3$; (d) $n=200,p=4\times 10^4,\theta = 0.3$; (e) $n=400,p=16\times 10^4,\theta = 0.3$.   Recovery error is measured as $\big|1-\norm{\mb {AD}_o}{4}^4/n\big|$, since Lemma \ref{lemma:G(A)bound} shows that a perfect recovery gives $\norm{\mb {AD}_o}{4}^4/n=1$. All experiments are averaged among 5 trials and conducted on a 2.7 GHz Intel Core i5 processor (CPU of a 13-inch Mac Pro 2015).}
    \label{tab:Comparison}
\end{table}

\subsection{Learning Sparsifying Dictionaries of Real Images}

Indeed, the theoretical results of the MSP algorithm relies heavily on the Bernoulli-Gaussian assumption of the ground truth sparse code $\mb x_i$, that is, all $\mb x_i$ has the same element-wise variance. In order to demonstrate that the MSP algorithm can be applied to broader application scenarios beyond the Bernoulli-Gaussian setting, we test our algorithm on the MNIST dataset of hand-written digits \citep{lecun1998gradient}, whose element-wise (pixel-wise) variance are barely equal. In our implementation, we vectorize each $28\times 28$ image $j$ in the MNIST dataset into a vector $\mb y_i \in \bb R^{784}$ and directly apply our MSP algorithm \ref{algo:SPOrthDL} to the whole data set $\mb Y=[\mb y_1,\mb y_2,\dots,\mb y_{50,000}]$.\footnote{The training set of MNIST contains 50,000 images of size $28\times 28$.} Figure \ref{fig:MSPMNIST}(a) shows some bases learned from these images which obviously capture shapes of the digits. Figure \ref{fig:MSPMNIST}(b) shows some less significant base vectors in the space of $\bb R^{784}$ that are orthogonal to the above. 

We also compare our results with bases learned from Principal Component Analysis (PCA). As shown in Figure \ref{fig:PCAMNIST}(a), the top bases learned from PCA are blurred with shapes from different digits mixed together, unlike those bases learned from the MSP in Figure \ref{fig:MSPMNIST}(a) that clearly capture features of multiple different digits.

\begin{figure}[ht]
\centering
    \begin{subfigure}[t]{0.4\textwidth}
    \centering
    \begin{minipage}[t]{0.15\linewidth}
    \centering
    \frame{\includegraphics[width=\linewidth]{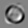}}\vspace{0.5mm}
    \frame{\includegraphics[width=\linewidth]{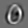}}
    \end{minipage}
    \begin{minipage}[t]{0.15\linewidth}
    \frame{\includegraphics[width=\linewidth]{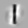}}\vspace{0.5mm}
    \frame{\includegraphics[width=\linewidth]{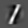}}
    \centering
    \end{minipage}
    \begin{minipage}[t]{0.15\linewidth}
    \centering
    \frame{\includegraphics[width=\linewidth]{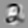}}\vspace{0.5mm}
    \frame{\includegraphics[width=\linewidth]{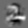}}
    \end{minipage}
    \begin{minipage}[t]{0.15\linewidth}
    \centering
    \frame{\includegraphics[width=\linewidth]{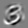}}\vspace{0.5mm}
    \frame{\includegraphics[width=\linewidth]{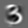}}
    \end{minipage}
    \begin{minipage}[t]{0.15\linewidth}
    \frame{\includegraphics[width=\linewidth]{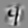}}\vspace{0.5mm}
    \frame{\includegraphics[width=\linewidth]{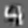}}
    \centering
    \end{minipage}
    
    \vspace{0.5mm}
    
    \begin{minipage}[t]{0.15\linewidth}
    \centering
    \frame{\includegraphics[width=\linewidth]{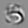}}\vspace{0.5mm}
    \frame{\includegraphics[width=\linewidth]{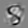}}
    \end{minipage}
    \begin{minipage}[t]{0.15\linewidth}
    \centering
    \frame{\includegraphics[width=\linewidth]{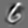}}\vspace{0.5mm}
    \frame{\includegraphics[width=\linewidth]{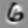}}
    \end{minipage}
    \begin{minipage}[t]{0.15\linewidth}
    \centering
    \frame{\includegraphics[width=\linewidth]{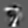}}\vspace{0.5mm}
    \frame{\includegraphics[width=\linewidth]{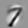}}
    \end{minipage}
    \begin{minipage}[t]{0.15\linewidth}
    \centering
    \frame{\includegraphics[width=\linewidth]{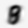}}\vspace{0.5mm}
    \frame{\includegraphics[width=\linewidth]{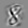}}
    \end{minipage}
    \begin{minipage}[t]{0.15\linewidth}
    \centering
    \frame{\includegraphics[width=\linewidth]{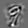}}\vspace{0.5mm}
    \frame{\includegraphics[width=\linewidth]{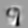}}
    \end{minipage}
    \caption{Some selected ``meaningful'' bases learned from the MNIST dataset.}
    \label{fig:MeaningfulMNIST}
    \end{subfigure}%
    ~
    \begin{subfigure}[t]{0.4\textwidth}
    \centering
    \begin{minipage}[t]{0.15\linewidth}
    \centering
    \frame{\includegraphics[width=\linewidth]{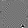}}\vspace{0.5mm}
    \frame{\includegraphics[width=\linewidth]{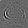}}
    \end{minipage}
    \begin{minipage}[t]{0.15\linewidth}
    \frame{\includegraphics[width=\linewidth]{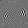}}\vspace{0.5mm}
    \frame{\includegraphics[width=\linewidth]{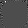}}
    \centering
    \end{minipage}
    \begin{minipage}[t]{0.15\linewidth}
    \centering
    \frame{\includegraphics[width=\linewidth]{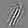}}\vspace{0.5mm}
    \frame{\includegraphics[width=\linewidth]{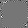}}
    \end{minipage}
    \begin{minipage}[t]{0.15\linewidth}
    \centering
    \frame{\includegraphics[width=\linewidth]{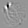}}\vspace{0.5mm}
    \frame{\includegraphics[width=\linewidth]{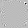}}
    \end{minipage}
    \begin{minipage}[t]{0.15\linewidth}
    \frame{\includegraphics[width=\linewidth]{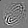}}\vspace{0.5mm}
    \frame{\includegraphics[width=\linewidth]{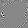}}
    \centering
    \end{minipage}
    
    \vspace{0.5mm}
    
    \begin{minipage}[t]{0.15\linewidth}
    \centering
    \frame{\includegraphics[width=\linewidth]{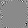}}\vspace{0.5mm}
    \frame{\includegraphics[width=\linewidth]{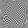}}
    \end{minipage}
    \begin{minipage}[t]{0.15\linewidth}
    \centering
    \frame{\includegraphics[width=\linewidth]{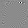}}\vspace{0.5mm}
    \frame{\includegraphics[width=\linewidth]{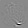}}
    \end{minipage}
    \begin{minipage}[t]{0.15\linewidth}
    \centering
    \frame{\includegraphics[width=\linewidth]{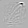}}\vspace{0.5mm}
    \frame{\includegraphics[width=\linewidth]{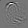}}
    \end{minipage}
    \begin{minipage}[t]{0.15\linewidth}
    \centering
    \frame{\includegraphics[width=\linewidth]{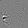}}\vspace{0.5mm}
    \frame{\includegraphics[width=\linewidth]{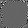}}
    \end{minipage}
    \begin{minipage}[t]{0.15\linewidth}
    \centering
    \frame{\includegraphics[width=\linewidth]{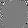}}\vspace{0.5mm}
    \frame{\includegraphics[width=\linewidth]{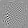}}
    \end{minipage}
    \caption{Other bases learned the MNIST dataset.}
    \label{fig:RandomMNIST}
    \end{subfigure}
    \caption{Learned dictionary through the MSP algorithm \ref{algo:SPOrthDL} on the MNIST dataset}
    \label{fig:MSPMNIST}
\end{figure}

\begin{figure}[H]
\centering
    \begin{subfigure}[t]{0.4\textwidth}
    \centering
    \begin{minipage}[t]{0.15\linewidth}
    \centering
    \frame{\includegraphics[width=\linewidth]{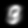}}\vspace{0.5mm}
    \frame{\includegraphics[width=\linewidth]{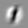}}
    \end{minipage}
    \begin{minipage}[t]{0.15\linewidth}
    \frame{\includegraphics[width=\linewidth]{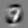}}\vspace{0.5mm}
    \frame{\includegraphics[width=\linewidth]{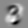}}
    \centering
    \end{minipage}
    \begin{minipage}[t]{0.15\linewidth}
    \centering
    \frame{\includegraphics[width=\linewidth]{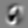}}\vspace{0.5mm}
    \frame{\includegraphics[width=\linewidth]{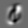}}
    \end{minipage}
    \begin{minipage}[t]{0.15\linewidth}
    \centering
    \frame{\includegraphics[width=\linewidth]{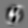}}\vspace{0.5mm}
    \frame{\includegraphics[width=\linewidth]{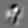}}
    \end{minipage}
    \begin{minipage}[t]{0.15\linewidth}
    \frame{\includegraphics[width=\linewidth]{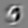}}\vspace{0.5mm}
    \frame{\includegraphics[width=\linewidth]{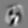}}
    \centering
    \end{minipage}
    
    \vspace{0.5mm}
    
    \begin{minipage}[t]{0.15\linewidth}
    \centering
    \frame{\includegraphics[width=\linewidth]{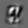}}\vspace{0.5mm}
    \frame{\includegraphics[width=\linewidth]{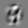}}
    \end{minipage}
    \begin{minipage}[t]{0.15\linewidth}
    \centering
    \frame{\includegraphics[width=\linewidth]{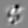}}\vspace{0.5mm}
    \frame{\includegraphics[width=\linewidth]{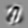}}
    \end{minipage}
    \begin{minipage}[t]{0.15\linewidth}
    \centering
    \frame{\includegraphics[width=\linewidth]{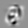}}\vspace{0.5mm}
    \frame{\includegraphics[width=\linewidth]{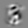}}
    \end{minipage}
    \begin{minipage}[t]{0.15\linewidth}
    \centering
    \frame{\includegraphics[width=\linewidth]{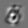}}\vspace{0.5mm}
    \frame{\includegraphics[width=\linewidth]{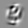}}
    \end{minipage}
    \begin{minipage}[t]{0.15\linewidth}
    \centering
    \frame{\includegraphics[width=\linewidth]{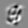}}\vspace{0.5mm}
    \frame{\includegraphics[width=\linewidth]{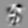}}
    \end{minipage}
    \caption{Top 20 bases from PCA}
    \label{fig:TopPCA}
    \end{subfigure}%
    ~
    \begin{subfigure}[t]{0.4\textwidth}
    \centering
    \begin{minipage}[t]{0.15\linewidth}
    \centering
    \frame{\includegraphics[width=\linewidth]{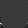}}\vspace{0.5mm}
    \frame{\includegraphics[width=\linewidth]{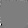}}
    \end{minipage}
    \begin{minipage}[t]{0.15\linewidth}
    \frame{\includegraphics[width=\linewidth]{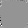}}\vspace{0.5mm}
    \frame{\includegraphics[width=\linewidth]{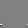}}
    \centering
    \end{minipage}
    \begin{minipage}[t]{0.15\linewidth}
    \centering
    \frame{\includegraphics[width=\linewidth]{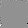}}\vspace{0.5mm}
    \frame{\includegraphics[width=\linewidth]{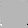}}
    \end{minipage}
    \begin{minipage}[t]{0.15\linewidth}
    \centering
    \frame{\includegraphics[width=\linewidth]{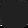}}\vspace{0.5mm}
    \frame{\includegraphics[width=\linewidth]{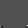}}
    \end{minipage}
    \begin{minipage}[t]{0.15\linewidth}
    \frame{\includegraphics[width=\linewidth]{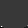}}\vspace{0.5mm}
    \frame{\includegraphics[width=\linewidth]{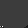}}
    \centering
    \end{minipage}
    
    \vspace{0.5mm}
    
    \begin{minipage}[t]{0.15\linewidth}
    \centering
    \frame{\includegraphics[width=\linewidth]{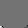}}\vspace{0.5mm}
    \frame{\includegraphics[width=\linewidth]{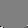}}
    \end{minipage}
    \begin{minipage}[t]{0.15\linewidth}
    \centering
    \frame{\includegraphics[width=\linewidth]{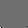}}\vspace{0.5mm}
    \frame{\includegraphics[width=\linewidth]{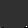}}
    \end{minipage}
    \begin{minipage}[t]{0.15\linewidth}
    \centering
    \frame{\includegraphics[width=\linewidth]{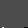}}\vspace{0.5mm}
    \frame{\includegraphics[width=\linewidth]{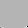}}
    \end{minipage}
    \begin{minipage}[t]{0.15\linewidth}
    \centering
    \frame{\includegraphics[width=\linewidth]{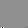}}\vspace{0.5mm}
    \frame{\includegraphics[width=\linewidth]{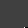}}
    \end{minipage}
    \begin{minipage}[t]{0.15\linewidth}
    \centering
    \frame{\includegraphics[width=\linewidth]{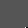}}\vspace{0.5mm}
    \frame{\includegraphics[width=\linewidth]{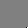}}
    \end{minipage}
    \caption{Last 20 bases from PCA}
    \label{fig:LastPCA}
    \end{subfigure}
    \caption{Learned Bases by PCA on MNIST dataset}
    \label{fig:PCAMNIST}
\end{figure}

We test the reconstruction results using the learned dictionary from the MSP algorithm and compare the reconstruction results using bases learned from PCA. In Figure \ref{fig:ReconMNIST}, we compare the original MNIST data \citep{lecun1998gradient}, the reconstruction results from the MSP algorithm and PCA, using the top 1, 2, 3, 4, 5, 10, and 25 bases, respectively. As we can see in Figure \ref{fig:ReconMNIST}, the reconstruction results for both the MSP and PCA improve, when the number of used bases increases. As shown in the case when only 2 bases are used, the MSP algorithm is able to represent clear shape for most digits already (see the three top left images), while PCA only provides mostly blurred results. When top 5 bases are used, the MSP algorithm almost capture all information of the input images, while PCA still gives rather blurred reconstruction. Both methods are able to capture salient information when the number of used bases are increased above 10. This is rather reasonable as there are about 10 different digits in this dataset and our method is precisely expected to have advantages when the number of bases is below 10. 

This experiment has shown that the proposed MSP algorithm can be applied to more general setting beyond the Bernoulli-Gaussian setting. Moreover, due to its efficiency, the MSP algorithm can be extend to other large scale visual dataset such as CIFAR10 \citep{krizhevsky2014cifar,zhai2019understanding}.

\begin{figure}[H]
\centering
    \begin{subfigure}[t]{0.48\textwidth}
    \centering
    \begin{minipage}[t]{0.08\linewidth}
    \centering
    \frame{\includegraphics[width=\linewidth]{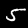}}\vspace{0.5mm}
    \frame{\includegraphics[width=\linewidth]{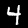}}
    \end{minipage}
    \begin{minipage}[t]{0.08\linewidth}
    \frame{\includegraphics[width=\linewidth]{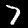}}\vspace{0.5mm}
    \frame{\includegraphics[width=\linewidth]{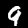}}
    \centering
    \end{minipage}
    \begin{minipage}[t]{0.08\linewidth}
    \centering
    \frame{\includegraphics[width=\linewidth]{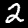}}\vspace{0.5mm}
    \frame{\includegraphics[width=\linewidth]{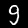}}
    \end{minipage}
    \begin{minipage}[t]{0.08\linewidth}
    \centering
    \frame{\includegraphics[width=\linewidth]{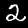}}\vspace{0.5mm}
    \frame{\includegraphics[width=\linewidth]{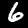}}
    \end{minipage}
    \begin{minipage}[t]{0.08\linewidth}
    \frame{\includegraphics[width=\linewidth]{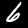}}\vspace{0.5mm}
    \frame{\includegraphics[width=\linewidth]{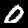}}
    \centering
    \end{minipage}
    \begin{minipage}[t]{0.08\linewidth}
    \centering
    \frame{\includegraphics[width=\linewidth]{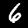}}\vspace{0.5mm}
    \frame{\includegraphics[width=\linewidth]{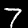}}
    \end{minipage}
    \begin{minipage}[t]{0.08\linewidth}
    \centering
    \frame{\includegraphics[width=\linewidth]{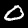}}\vspace{0.5mm}
    \frame{\includegraphics[width=\linewidth]{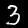}}
    \end{minipage}
    \begin{minipage}[t]{0.08\linewidth}
    \centering
    \frame{\includegraphics[width=\linewidth]{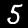}}\vspace{0.5mm}
    \frame{\includegraphics[width=\linewidth]{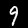}}
    \end{minipage}
    \begin{minipage}[t]{0.08\linewidth}
    \centering
    \frame{\includegraphics[width=\linewidth]{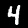}}\vspace{0.5mm}
    \frame{\includegraphics[width=\linewidth]{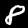}}
    \end{minipage}
    \begin{minipage}[t]{0.08\linewidth}
    \centering
    \frame{\includegraphics[width=\linewidth]{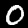}}\vspace{0.5mm}
    \frame{\includegraphics[width=\linewidth]{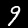}}
    \end{minipage}
    \caption{Original Images from the MNIST dataset}
    \label{fig:OriginalMNIST_PCA}
    \end{subfigure}%
    ~
    \begin{subfigure}[t]{0.48\textwidth}
    \centering
    \begin{minipage}[t]{0.08\linewidth}
    \centering
    \frame{\includegraphics[width=\linewidth]{Figure/Reconstruction/Ori/MNIST_Ori_81.jpg}}\vspace{0.5mm}
    \frame{\includegraphics[width=\linewidth]{Figure/Reconstruction/Ori/MNIST_Ori_82.jpg}}
    \end{minipage}
    \begin{minipage}[t]{0.08\linewidth}
    \frame{\includegraphics[width=\linewidth]{Figure/Reconstruction/Ori/MNIST_Ori_83.jpg}}\vspace{0.5mm}
    \frame{\includegraphics[width=\linewidth]{Figure/Reconstruction/Ori/MNIST_Ori_84.jpg}}
    \centering
    \end{minipage}
    \begin{minipage}[t]{0.08\linewidth}
    \centering
    \frame{\includegraphics[width=\linewidth]{Figure/Reconstruction/Ori/MNIST_Ori_85.jpg}}\vspace{0.5mm}
    \frame{\includegraphics[width=\linewidth]{Figure/Reconstruction/Ori/MNIST_Ori_86.jpg}}
    \end{minipage}
    \begin{minipage}[t]{0.08\linewidth}
    \centering
    \frame{\includegraphics[width=\linewidth]{Figure/Reconstruction/Ori/MNIST_Ori_87.jpg}}\vspace{0.5mm}
    \frame{\includegraphics[width=\linewidth]{Figure/Reconstruction/Ori/MNIST_Ori_88.jpg}}
    \end{minipage}
    \begin{minipage}[t]{0.08\linewidth}
    \frame{\includegraphics[width=\linewidth]{Figure/Reconstruction/Ori/MNIST_Ori_89.jpg}}\vspace{0.5mm}
    \frame{\includegraphics[width=\linewidth]{Figure/Reconstruction/Ori/MNIST_Ori_90.jpg}}
    \centering
    \end{minipage}
    \begin{minipage}[t]{0.08\linewidth}
    \centering
    \frame{\includegraphics[width=\linewidth]{Figure/Reconstruction/Ori/MNIST_Ori_91.jpg}}\vspace{0.5mm}
    \frame{\includegraphics[width=\linewidth]{Figure/Reconstruction/Ori/MNIST_Ori_92.jpg}}
    \end{minipage}
    \begin{minipage}[t]{0.08\linewidth}
    \centering
    \frame{\includegraphics[width=\linewidth]{Figure/Reconstruction/Ori/MNIST_Ori_93.jpg}}\vspace{0.5mm}
    \frame{\includegraphics[width=\linewidth]{Figure/Reconstruction/Ori/MNIST_Ori_94.jpg}}
    \end{minipage}
    \begin{minipage}[t]{0.08\linewidth}
    \centering
    \frame{\includegraphics[width=\linewidth]{Figure/Reconstruction/Ori/MNIST_Ori_95.jpg}}\vspace{0.5mm}
    \frame{\includegraphics[width=\linewidth]{Figure/Reconstruction/Ori/MNIST_Ori_96.jpg}}
    \end{minipage}
    \begin{minipage}[t]{0.08\linewidth}
    \centering
    \frame{\includegraphics[width=\linewidth]{Figure/Reconstruction/Ori/MNIST_Ori_97.jpg}}\vspace{0.5mm}
    \frame{\includegraphics[width=\linewidth]{Figure/Reconstruction/Ori/MNIST_Ori_98.jpg}}
    \end{minipage}
    \begin{minipage}[t]{0.08\linewidth}
    \centering
    \frame{\includegraphics[width=\linewidth]{Figure/Reconstruction/Ori/MNIST_Ori_99.jpg}}\vspace{0.5mm}
    \frame{\includegraphics[width=\linewidth]{Figure/Reconstruction/Ori/MNIST_Ori_100.jpg}}
    \end{minipage}
    \caption{Original Images from the MNIST dataset}
    \label{fig:OriginalMNIST}
    \end{subfigure}%
    \\
    \begin{subfigure}[t]{0.48\textwidth}
    \centering
    \begin{minipage}[t]{0.08\linewidth}
    \centering
    \frame{\includegraphics[width=\linewidth]{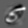}}\vspace{0.5mm}
    \frame{\includegraphics[width=\linewidth]{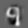}}
    \end{minipage}
    \begin{minipage}[t]{0.08\linewidth}
    \frame{\includegraphics[width=\linewidth]{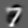}}\vspace{0.5mm}
    \frame{\includegraphics[width=\linewidth]{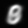}}
    \centering
    \end{minipage}
    \begin{minipage}[t]{0.08\linewidth}
    \centering
    \frame{\includegraphics[width=\linewidth]{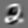}}\vspace{0.5mm}
    \frame{\includegraphics[width=\linewidth]{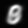}}
    \end{minipage}
    \begin{minipage}[t]{0.08\linewidth}
    \centering
    \frame{\includegraphics[width=\linewidth]{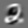}}\vspace{0.5mm}
    \frame{\includegraphics[width=\linewidth]{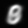}}
    \end{minipage}
    \begin{minipage}[t]{0.08\linewidth}
    \frame{\includegraphics[width=\linewidth]{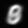}}\vspace{0.5mm}
    \frame{\includegraphics[width=\linewidth]{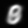}}
    \centering
    \end{minipage}
    \begin{minipage}[t]{0.08\linewidth}
    \centering
    \frame{\includegraphics[width=\linewidth]{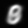}}\vspace{0.5mm}
    \frame{\includegraphics[width=\linewidth]{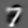}}
    \end{minipage}
    \begin{minipage}[t]{0.08\linewidth}
    \centering
    \frame{\includegraphics[width=\linewidth]{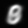}}\vspace{0.5mm}
    \frame{\includegraphics[width=\linewidth]{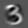}}
    \end{minipage}
    \begin{minipage}[t]{0.08\linewidth}
    \centering
    \frame{\includegraphics[width=\linewidth]{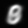}}\vspace{0.5mm}
    \frame{\includegraphics[width=\linewidth]{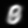}}
    \end{minipage}
    \begin{minipage}[t]{0.08\linewidth}
    \centering
    \frame{\includegraphics[width=\linewidth]{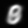}}\vspace{0.5mm}
    \frame{\includegraphics[width=\linewidth]{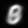}}
    \end{minipage}
    \begin{minipage}[t]{0.08\linewidth}
    \centering
    \frame{\includegraphics[width=\linewidth]{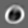}}\vspace{0.5mm}
    \frame{\includegraphics[width=\linewidth]{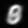}}
    \end{minipage}
    \caption{Reconstruction with top 1 Basis by MSP}
    \label{fig:MSPReconMNIST1}
    \end{subfigure}%
    ~
    \begin{subfigure}[t]{0.48\textwidth}
    \centering
    \begin{minipage}[t]{0.08\linewidth}
    \centering
    \frame{\includegraphics[width=\linewidth]{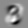}}\vspace{0.5mm}
    \frame{\includegraphics[width=\linewidth]{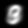}}
    \end{minipage}
    \begin{minipage}[t]{0.08\linewidth}
    \frame{\includegraphics[width=\linewidth]{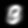}}\vspace{0.5mm}
    \frame{\includegraphics[width=\linewidth]{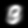}}
    \centering
    \end{minipage}
    \begin{minipage}[t]{0.08\linewidth}
    \centering
    \frame{\includegraphics[width=\linewidth]{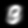}}\vspace{0.5mm}
    \frame{\includegraphics[width=\linewidth]{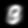}}
    \end{minipage}
    \begin{minipage}[t]{0.08\linewidth}
    \centering
    \frame{\includegraphics[width=\linewidth]{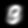}}\vspace{0.5mm}
    \frame{\includegraphics[width=\linewidth]{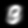}}
    \end{minipage}
    \begin{minipage}[t]{0.08\linewidth}
    \frame{\includegraphics[width=\linewidth]{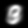}}\vspace{0.5mm}
    \frame{\includegraphics[width=\linewidth]{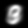}}
    \centering
    \end{minipage}
    \begin{minipage}[t]{0.08\linewidth}
    \centering
    \frame{\includegraphics[width=\linewidth]{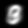}}\vspace{0.5mm}
    \frame{\includegraphics[width=\linewidth]{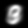}}
    \end{minipage}
    \begin{minipage}[t]{0.08\linewidth}
    \centering
    \frame{\includegraphics[width=\linewidth]{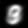}}\vspace{0.5mm}
    \frame{\includegraphics[width=\linewidth]{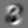}}
    \end{minipage}
    \begin{minipage}[t]{0.08\linewidth}
    \centering
    \frame{\includegraphics[width=\linewidth]{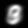}}\vspace{0.5mm}
    \frame{\includegraphics[width=\linewidth]{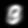}}
    \end{minipage}
    \begin{minipage}[t]{0.08\linewidth}
    \centering
    \frame{\includegraphics[width=\linewidth]{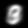}}\vspace{0.5mm}
    \frame{\includegraphics[width=\linewidth]{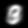}}
    \end{minipage}
    \begin{minipage}[t]{0.08\linewidth}
    \centering
    \frame{\includegraphics[width=\linewidth]{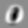}}\vspace{0.5mm}
    \frame{\includegraphics[width=\linewidth]{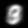}}
    \end{minipage}
    \caption{Reconstruction with top 1 Basis by PCA}
    \label{fig:PCAReconMNIST1}
    \end{subfigure}%
    \\
    \begin{subfigure}[b]{0.48\textwidth}
    \centering
    \begin{minipage}[t]{0.08\linewidth}
    \centering
    \frame{\includegraphics[width=\linewidth]{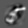}}\vspace{0.5mm}
    \frame{\includegraphics[width=\linewidth]{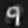}}
    \end{minipage}
    \begin{minipage}[t]{0.08\linewidth}
    \frame{\includegraphics[width=\linewidth]{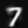}}\vspace{0.5mm}
    \frame{\includegraphics[width=\linewidth]{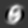}}
    \centering
    \end{minipage}
    \begin{minipage}[t]{0.08\linewidth}
    \centering
    \frame{\includegraphics[width=\linewidth]{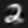}}\vspace{0.5mm}
    \frame{\includegraphics[width=\linewidth]{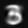}}
    \end{minipage}
    \begin{minipage}[t]{0.08\linewidth}
    \centering
    \frame{\includegraphics[width=\linewidth]{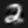}}\vspace{0.5mm}
    \frame{\includegraphics[width=\linewidth]{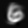}}
    \end{minipage}
    \begin{minipage}[t]{0.08\linewidth}
    \frame{\includegraphics[width=\linewidth]{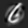}}\vspace{0.5mm}
    \frame{\includegraphics[width=\linewidth]{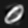}}
    \centering
    \end{minipage}
    \begin{minipage}[t]{0.08\linewidth}
    \centering
    \frame{\includegraphics[width=\linewidth]{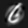}}\vspace{0.5mm}
    \frame{\includegraphics[width=\linewidth]{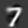}}
    \end{minipage}
    \begin{minipage}[t]{0.08\linewidth}
    \centering
    \frame{\includegraphics[width=\linewidth]{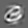}}\vspace{0.5mm}
    \frame{\includegraphics[width=\linewidth]{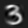}}
    \end{minipage}
    \begin{minipage}[t]{0.08\linewidth}
    \centering
    \frame{\includegraphics[width=\linewidth]{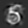}}\vspace{0.5mm}
    \frame{\includegraphics[width=\linewidth]{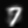}}
    \end{minipage}
    \begin{minipage}[t]{0.08\linewidth}
    \centering
    \frame{\includegraphics[width=\linewidth]{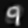}}\vspace{0.5mm}
    \frame{\includegraphics[width=\linewidth]{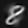}}
    \end{minipage}
    \begin{minipage}[t]{0.08\linewidth}
    \centering
    \frame{\includegraphics[width=\linewidth]{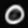}}\vspace{0.5mm}
    \frame{\includegraphics[width=\linewidth]{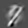}}
    \end{minipage}
    \caption{Reconstruction with top 2 Bases by MSP}
    \label{fig:MSPReconMNIST2}
    \end{subfigure}%
    ~
    \begin{subfigure}[b]{0.48\textwidth}
    \centering
    \begin{minipage}[t]{0.08\linewidth}
    \centering
    \frame{\includegraphics[width=\linewidth]{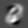}}\vspace{0.5mm}
    \frame{\includegraphics[width=\linewidth]{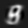}}
    \end{minipage}
    \begin{minipage}[t]{0.08\linewidth}
    \frame{\includegraphics[width=\linewidth]{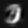}}\vspace{0.5mm}
    \frame{\includegraphics[width=\linewidth]{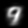}}
    \centering
    \end{minipage}
    \begin{minipage}[t]{0.08\linewidth}
    \centering
    \frame{\includegraphics[width=\linewidth]{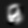}}\vspace{0.5mm}
    \frame{\includegraphics[width=\linewidth]{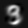}}
    \end{minipage}
    \begin{minipage}[t]{0.08\linewidth}
    \centering
    \frame{\includegraphics[width=\linewidth]{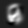}}\vspace{0.5mm}
    \frame{\includegraphics[width=\linewidth]{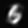}}
    \end{minipage}
    \begin{minipage}[t]{0.08\linewidth}
    \frame{\includegraphics[width=\linewidth]{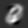}}\vspace{0.5mm}
    \frame{\includegraphics[width=\linewidth]{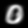}}
    \centering
    \end{minipage}
    \begin{minipage}[t]{0.08\linewidth}
    \centering
    \frame{\includegraphics[width=\linewidth]{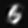}}\vspace{0.5mm}
    \frame{\includegraphics[width=\linewidth]{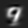}}
    \end{minipage}
    \begin{minipage}[t]{0.08\linewidth}
    \centering
    \frame{\includegraphics[width=\linewidth]{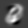}}\vspace{0.5mm}
    \frame{\includegraphics[width=\linewidth]{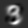}}
    \end{minipage}
    \begin{minipage}[t]{0.08\linewidth}
    \centering
    \frame{\includegraphics[width=\linewidth]{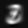}}\vspace{0.5mm}
    \frame{\includegraphics[width=\linewidth]{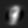}}
    \end{minipage}
    \begin{minipage}[t]{0.08\linewidth}
    \centering
    \frame{\includegraphics[width=\linewidth]{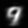}}\vspace{0.5mm}
    \frame{\includegraphics[width=\linewidth]{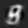}}
    \end{minipage}
    \begin{minipage}[t]{0.08\linewidth}
    \centering
    \frame{\includegraphics[width=\linewidth]{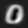}}\vspace{0.5mm}
    \frame{\includegraphics[width=\linewidth]{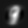}}
    \end{minipage}
    \caption{Reconstruction with top 2 Bases by PCA}
    \label{fig:PCAReconMNIST2}
    \end{subfigure}%
    \\
    \begin{subfigure}[t]{0.48\textwidth}
    \centering
    \begin{minipage}[t]{0.08\linewidth}
    \centering
    \frame{\includegraphics[width=\linewidth]{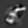}}\vspace{0.5mm}
    \frame{\includegraphics[width=\linewidth]{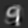}}
    \end{minipage}
    \begin{minipage}[t]{0.08\linewidth}
    \frame{\includegraphics[width=\linewidth]{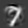}}\vspace{0.5mm}
    \frame{\includegraphics[width=\linewidth]{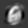}}
    \centering
    \end{minipage}
    \begin{minipage}[t]{0.08\linewidth}
    \centering
    \frame{\includegraphics[width=\linewidth]{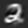}}\vspace{0.5mm}
    \frame{\includegraphics[width=\linewidth]{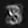}}
    \end{minipage}
    \begin{minipage}[t]{0.08\linewidth}
    \centering
    \frame{\includegraphics[width=\linewidth]{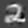}}\vspace{0.5mm}
    \frame{\includegraphics[width=\linewidth]{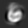}}
    \end{minipage}
    \begin{minipage}[t]{0.08\linewidth}
    \frame{\includegraphics[width=\linewidth]{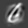}}\vspace{0.5mm}
    \frame{\includegraphics[width=\linewidth]{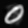}}
    \centering
    \end{minipage}
    \begin{minipage}[t]{0.08\linewidth}
    \centering
    \frame{\includegraphics[width=\linewidth]{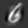}}\vspace{0.5mm}
    \frame{\includegraphics[width=\linewidth]{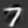}}
    \end{minipage}
    \begin{minipage}[t]{0.08\linewidth}
    \centering
    \frame{\includegraphics[width=\linewidth]{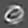}}\vspace{0.5mm}
    \frame{\includegraphics[width=\linewidth]{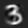}}
    \end{minipage}
    \begin{minipage}[t]{0.08\linewidth}
    \centering
    \frame{\includegraphics[width=\linewidth]{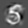}}\vspace{0.5mm}
    \frame{\includegraphics[width=\linewidth]{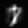}}
    \end{minipage}
    \begin{minipage}[t]{0.08\linewidth}
    \centering
    \frame{\includegraphics[width=\linewidth]{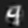}}\vspace{0.5mm}
    \frame{\includegraphics[width=\linewidth]{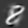}}
    \end{minipage}
    \begin{minipage}[t]{0.08\linewidth}
    \centering
    \frame{\includegraphics[width=\linewidth]{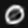}}\vspace{0.5mm}
    \frame{\includegraphics[width=\linewidth]{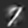}}
    \end{minipage}
    \caption{Reconstruction of with top 3 Bases by MSP}
    \label{fig:MSPReconMNIST3}
    \end{subfigure}%
    ~
    \begin{subfigure}[t]{0.48\textwidth}
    \centering
    \begin{minipage}[t]{0.08\linewidth}
    \centering
    \frame{\includegraphics[width=\linewidth]{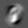}}\vspace{0.5mm}
    \frame{\includegraphics[width=\linewidth]{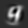}}
    \end{minipage}
    \begin{minipage}[t]{0.08\linewidth}
    \frame{\includegraphics[width=\linewidth]{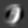}}\vspace{0.5mm}
    \frame{\includegraphics[width=\linewidth]{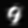}}
    \centering
    \end{minipage}
    \begin{minipage}[t]{0.08\linewidth}
    \centering
    \frame{\includegraphics[width=\linewidth]{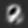}}\vspace{0.5mm}
    \frame{\includegraphics[width=\linewidth]{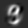}}
    \end{minipage}
    \begin{minipage}[t]{0.08\linewidth}
    \centering
    \frame{\includegraphics[width=\linewidth]{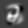}}\vspace{0.5mm}
    \frame{\includegraphics[width=\linewidth]{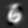}}
    \end{minipage}
    \begin{minipage}[t]{0.08\linewidth}
    \frame{\includegraphics[width=\linewidth]{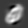}}\vspace{0.5mm}
    \frame{\includegraphics[width=\linewidth]{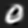}}
    \centering
    \end{minipage}
    \begin{minipage}[t]{0.08\linewidth}
    \centering
    \frame{\includegraphics[width=\linewidth]{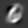}}\vspace{0.5mm}
    \frame{\includegraphics[width=\linewidth]{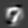}}
    \end{minipage}
    \begin{minipage}[t]{0.08\linewidth}
    \centering
    \frame{\includegraphics[width=\linewidth]{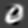}}\vspace{0.5mm}
    \frame{\includegraphics[width=\linewidth]{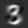}}
    \end{minipage}
    \begin{minipage}[t]{0.08\linewidth}
    \centering
    \frame{\includegraphics[width=\linewidth]{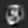}}\vspace{0.5mm}
    \frame{\includegraphics[width=\linewidth]{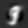}}
    \end{minipage}
    \begin{minipage}[t]{0.08\linewidth}
    \centering
    \frame{\includegraphics[width=\linewidth]{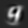}}\vspace{0.5mm}
    \frame{\includegraphics[width=\linewidth]{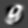}}
    \end{minipage}
    \begin{minipage}[t]{0.08\linewidth}
    \centering
    \frame{\includegraphics[width=\linewidth]{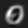}}\vspace{0.5mm}
    \frame{\includegraphics[width=\linewidth]{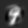}}
    \end{minipage}
    \caption{Reconstruction with top 3 Bases by PCA}
    \label{fig:PCAReconMNIST3}
    \end{subfigure}%
    \\
    \begin{subfigure}[b]{0.48\textwidth}
    \centering
    \begin{minipage}[t]{0.08\linewidth}
    \centering
    \frame{\includegraphics[width=\linewidth]{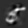}}\vspace{0.5mm}
    \frame{\includegraphics[width=\linewidth]{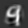}}
    \end{minipage}
    \begin{minipage}[t]{0.08\linewidth}
    \frame{\includegraphics[width=\linewidth]{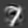}}\vspace{0.5mm}
    \frame{\includegraphics[width=\linewidth]{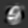}}
    \centering
    \end{minipage}
    \begin{minipage}[t]{0.08\linewidth}
    \centering
    \frame{\includegraphics[width=\linewidth]{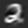}}\vspace{0.5mm}
    \frame{\includegraphics[width=\linewidth]{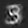}}
    \end{minipage}
    \begin{minipage}[t]{0.08\linewidth}
    \centering
    \frame{\includegraphics[width=\linewidth]{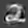}}\vspace{0.5mm}
    \frame{\includegraphics[width=\linewidth]{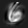}}
    \end{minipage}
    \begin{minipage}[t]{0.08\linewidth}
    \frame{\includegraphics[width=\linewidth]{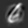}}\vspace{0.5mm}
    \frame{\includegraphics[width=\linewidth]{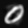}}
    \centering
    \end{minipage}
    \begin{minipage}[t]{0.08\linewidth}
    \centering
    \frame{\includegraphics[width=\linewidth]{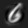}}\vspace{0.5mm}
    \frame{\includegraphics[width=\linewidth]{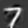}}
    \end{minipage}
    \begin{minipage}[t]{0.08\linewidth}
    \centering
    \frame{\includegraphics[width=\linewidth]{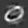}}\vspace{0.5mm}
    \frame{\includegraphics[width=\linewidth]{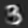}}
    \end{minipage}
    \begin{minipage}[t]{0.08\linewidth}
    \centering
    \frame{\includegraphics[width=\linewidth]{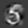}}\vspace{0.5mm}
    \frame{\includegraphics[width=\linewidth]{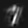}}
    \end{minipage}
    \begin{minipage}[t]{0.08\linewidth}
    \centering
    \frame{\includegraphics[width=\linewidth]{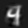}}\vspace{0.5mm}
    \frame{\includegraphics[width=\linewidth]{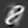}}
    \end{minipage}
    \begin{minipage}[t]{0.08\linewidth}
    \centering
    \frame{\includegraphics[width=\linewidth]{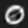}}\vspace{0.5mm}
    \frame{\includegraphics[width=\linewidth]{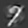}}
    \end{minipage}
    \caption{Reconstruction with top 4 Bases by MSP}
    \label{fig:MSPReconMNIST4}
    \end{subfigure}%
    ~
    \begin{subfigure}[b]{0.48\textwidth}
    \centering
    \begin{minipage}[t]{0.08\linewidth}
    \centering
    \frame{\includegraphics[width=\linewidth]{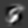}}\vspace{0.5mm}
    \frame{\includegraphics[width=\linewidth]{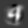}}
    \end{minipage}
    \begin{minipage}[t]{0.08\linewidth}
    \frame{\includegraphics[width=\linewidth]{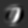}}\vspace{0.5mm}
    \frame{\includegraphics[width=\linewidth]{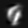}}
    \centering
    \end{minipage}
    \begin{minipage}[t]{0.08\linewidth}
    \centering
    \frame{\includegraphics[width=\linewidth]{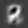}}\vspace{0.5mm}
    \frame{\includegraphics[width=\linewidth]{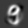}}
    \end{minipage}
    \begin{minipage}[t]{0.08\linewidth}
    \centering
    \frame{\includegraphics[width=\linewidth]{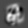}}\vspace{0.5mm}
    \frame{\includegraphics[width=\linewidth]{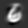}}
    \end{minipage}
    \begin{minipage}[t]{0.08\linewidth}
    \frame{\includegraphics[width=\linewidth]{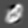}}\vspace{0.5mm}
    \frame{\includegraphics[width=\linewidth]{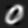}}
    \centering
    \end{minipage}
    \begin{minipage}[t]{0.08\linewidth}
    \centering
    \frame{\includegraphics[width=\linewidth]{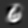}}\vspace{0.5mm}
    \frame{\includegraphics[width=\linewidth]{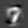}}
    \end{minipage}
    \begin{minipage}[t]{0.08\linewidth}
    \centering
    \frame{\includegraphics[width=\linewidth]{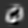}}\vspace{0.5mm}
    \frame{\includegraphics[width=\linewidth]{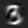}}
    \end{minipage}
    \begin{minipage}[t]{0.08\linewidth}
    \centering
    \frame{\includegraphics[width=\linewidth]{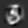}}\vspace{0.5mm}
    \frame{\includegraphics[width=\linewidth]{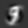}}
    \end{minipage}
    \begin{minipage}[t]{0.08\linewidth}
    \centering
    \frame{\includegraphics[width=\linewidth]{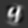}}\vspace{0.5mm}
    \frame{\includegraphics[width=\linewidth]{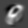}}
    \end{minipage}
    \begin{minipage}[t]{0.08\linewidth}
    \centering
    \frame{\includegraphics[width=\linewidth]{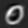}}\vspace{0.5mm}
    \frame{\includegraphics[width=\linewidth]{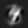}}
    \end{minipage}
    \caption{Reconstruction with top 4 Bases by PCA}
    \label{fig:PCAReconMNIST4}
    \end{subfigure}%
    \\
    \begin{subfigure}[b]{0.48\textwidth}
    \centering
    \begin{minipage}[t]{0.08\linewidth}
    \centering
    \frame{\includegraphics[width=\linewidth]{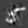}}\vspace{0.5mm}
    \frame{\includegraphics[width=\linewidth]{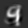}}
    \end{minipage}
    \begin{minipage}[t]{0.08\linewidth}
    \frame{\includegraphics[width=\linewidth]{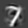}}\vspace{0.5mm}
    \frame{\includegraphics[width=\linewidth]{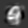}}
    \centering
    \end{minipage}
    \begin{minipage}[t]{0.08\linewidth}
    \centering
    \frame{\includegraphics[width=\linewidth]{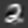}}\vspace{0.5mm}
    \frame{\includegraphics[width=\linewidth]{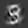}}
    \end{minipage}
    \begin{minipage}[t]{0.08\linewidth}
    \centering
    \frame{\includegraphics[width=\linewidth]{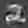}}\vspace{0.5mm}
    \frame{\includegraphics[width=\linewidth]{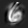}}
    \end{minipage}
    \begin{minipage}[t]{0.08\linewidth}
    \frame{\includegraphics[width=\linewidth]{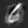}}\vspace{0.5mm}
    \frame{\includegraphics[width=\linewidth]{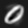}}
    \centering
    \end{minipage}
    \begin{minipage}[t]{0.08\linewidth}
    \centering
    \frame{\includegraphics[width=\linewidth]{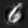}}\vspace{0.5mm}
    \frame{\includegraphics[width=\linewidth]{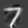}}
    \end{minipage}
    \begin{minipage}[t]{0.08\linewidth}
    \centering
    \frame{\includegraphics[width=\linewidth]{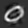}}\vspace{0.5mm}
    \frame{\includegraphics[width=\linewidth]{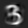}}
    \end{minipage}
    \begin{minipage}[t]{0.08\linewidth}
    \centering
    \frame{\includegraphics[width=\linewidth]{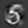}}\vspace{0.5mm}
    \frame{\includegraphics[width=\linewidth]{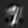}}
    \end{minipage}
    \begin{minipage}[t]{0.08\linewidth}
    \centering
    \frame{\includegraphics[width=\linewidth]{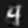}}\vspace{0.5mm}
    \frame{\includegraphics[width=\linewidth]{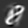}}
    \end{minipage}
    \begin{minipage}[t]{0.08\linewidth}
    \centering
    \frame{\includegraphics[width=\linewidth]{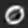}}\vspace{0.5mm}
    \frame{\includegraphics[width=\linewidth]{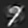}}
    \end{minipage}
    \caption{Reconstruction with top 5 Bases by MSP}
    \label{fig:MSPReconMNIST5}
    \end{subfigure}%
    ~
    \begin{subfigure}[b]{0.48\textwidth}
    \centering
    \begin{minipage}[t]{0.08\linewidth}
    \centering
    \frame{\includegraphics[width=\linewidth]{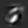}}\vspace{0.5mm}
    \frame{\includegraphics[width=\linewidth]{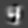}}
    \end{minipage}
    \begin{minipage}[t]{0.08\linewidth}
    \frame{\includegraphics[width=\linewidth]{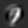}}\vspace{0.5mm}
    \frame{\includegraphics[width=\linewidth]{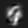}}
    \centering
    \end{minipage}
    \begin{minipage}[t]{0.08\linewidth}
    \centering
    \frame{\includegraphics[width=\linewidth]{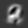}}\vspace{0.5mm}
    \frame{\includegraphics[width=\linewidth]{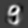}}
    \end{minipage}
    \begin{minipage}[t]{0.08\linewidth}
    \centering
    \frame{\includegraphics[width=\linewidth]{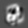}}\vspace{0.5mm}
    \frame{\includegraphics[width=\linewidth]{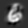}}
    \end{minipage}
    \begin{minipage}[t]{0.08\linewidth}
    \frame{\includegraphics[width=\linewidth]{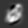}}\vspace{0.5mm}
    \frame{\includegraphics[width=\linewidth]{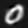}}
    \centering
    \end{minipage}
    \begin{minipage}[t]{0.08\linewidth}
    \centering
    \frame{\includegraphics[width=\linewidth]{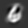}}\vspace{0.5mm}
    \frame{\includegraphics[width=\linewidth]{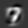}}
    \end{minipage}
    \begin{minipage}[t]{0.08\linewidth}
    \centering
    \frame{\includegraphics[width=\linewidth]{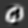}}\vspace{0.5mm}
    \frame{\includegraphics[width=\linewidth]{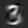}}
    \end{minipage}
    \begin{minipage}[t]{0.08\linewidth}
    \centering
    \frame{\includegraphics[width=\linewidth]{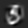}}\vspace{0.5mm}
    \frame{\includegraphics[width=\linewidth]{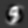}}
    \end{minipage}
    \begin{minipage}[t]{0.08\linewidth}
    \centering
    \frame{\includegraphics[width=\linewidth]{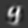}}\vspace{0.5mm}
    \frame{\includegraphics[width=\linewidth]{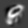}}
    \end{minipage}
    \begin{minipage}[t]{0.08\linewidth}
    \centering
    \frame{\includegraphics[width=\linewidth]{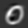}}\vspace{0.5mm}
    \frame{\includegraphics[width=\linewidth]{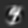}}
    \end{minipage}
    \caption{Reconstruction with top 5 Bases by PCA}
    \label{fig:PCAReconMNIST5}
    \end{subfigure}%
    \\
    \begin{subfigure}[b]{0.48\textwidth}
    \centering
    \begin{minipage}[t]{0.08\linewidth}
    \centering
    \frame{\includegraphics[width=\linewidth]{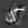}}\vspace{0.5mm}
    \frame{\includegraphics[width=\linewidth]{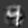}}
    \end{minipage}
    \begin{minipage}[t]{0.08\linewidth}
    \frame{\includegraphics[width=\linewidth]{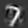}}\vspace{0.5mm}
    \frame{\includegraphics[width=\linewidth]{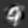}}
    \centering
    \end{minipage}
    \begin{minipage}[t]{0.08\linewidth}
    \centering
    \frame{\includegraphics[width=\linewidth]{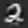}}\vspace{0.5mm}
    \frame{\includegraphics[width=\linewidth]{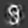}}
    \end{minipage}
    \begin{minipage}[t]{0.08\linewidth}
    \centering
    \frame{\includegraphics[width=\linewidth]{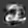}}\vspace{0.5mm}
    \frame{\includegraphics[width=\linewidth]{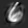}}
    \end{minipage}
    \begin{minipage}[t]{0.08\linewidth}
    \frame{\includegraphics[width=\linewidth]{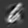}}\vspace{0.5mm}
    \frame{\includegraphics[width=\linewidth]{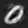}}
    \centering
    \end{minipage}
    \begin{minipage}[t]{0.08\linewidth}
    \centering
    \frame{\includegraphics[width=\linewidth]{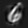}}\vspace{0.5mm}
    \frame{\includegraphics[width=\linewidth]{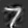}}
    \end{minipage}
    \begin{minipage}[t]{0.08\linewidth}
    \centering
    \frame{\includegraphics[width=\linewidth]{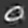}}\vspace{0.5mm}
    \frame{\includegraphics[width=\linewidth]{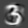}}
    \end{minipage}
    \begin{minipage}[t]{0.08\linewidth}
    \centering
    \frame{\includegraphics[width=\linewidth]{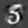}}\vspace{0.5mm}
    \frame{\includegraphics[width=\linewidth]{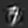}}
    \end{minipage}
    \begin{minipage}[t]{0.08\linewidth}
    \centering
    \frame{\includegraphics[width=\linewidth]{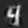}}\vspace{0.5mm}
    \frame{\includegraphics[width=\linewidth]{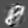}}
    \end{minipage}
    \begin{minipage}[t]{0.08\linewidth}
    \centering
    \frame{\includegraphics[width=\linewidth]{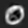}}\vspace{0.5mm}
    \frame{\includegraphics[width=\linewidth]{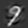}}
    \end{minipage}
    \caption{Reconstruction with top 10 Bases by MSP}
    \label{fig:MSPReconMNIST10}
    \end{subfigure}%
    ~
    \begin{subfigure}[b]{0.48\textwidth}
    \centering
    \begin{minipage}[t]{0.08\linewidth}
    \centering
    \frame{\includegraphics[width=\linewidth]{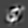}}\vspace{0.5mm}
    \frame{\includegraphics[width=\linewidth]{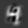}}
    \end{minipage}
    \begin{minipage}[t]{0.08\linewidth}
    \frame{\includegraphics[width=\linewidth]{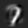}}\vspace{0.5mm}
    \frame{\includegraphics[width=\linewidth]{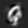}}
    \centering
    \end{minipage}
    \begin{minipage}[t]{0.08\linewidth}
    \centering
    \frame{\includegraphics[width=\linewidth]{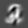}}\vspace{0.5mm}
    \frame{\includegraphics[width=\linewidth]{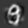}}
    \end{minipage}
    \begin{minipage}[t]{0.08\linewidth}
    \centering
    \frame{\includegraphics[width=\linewidth]{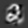}}\vspace{0.5mm}
    \frame{\includegraphics[width=\linewidth]{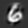}}
    \end{minipage}
    \begin{minipage}[t]{0.08\linewidth}
    \frame{\includegraphics[width=\linewidth]{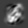}}\vspace{0.5mm}
    \frame{\includegraphics[width=\linewidth]{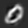}}
    \centering
    \end{minipage}
    \begin{minipage}[t]{0.08\linewidth}
    \centering
    \frame{\includegraphics[width=\linewidth]{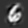}}\vspace{0.5mm}
    \frame{\includegraphics[width=\linewidth]{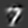}}
    \end{minipage}
    \begin{minipage}[t]{0.08\linewidth}
    \centering
    \frame{\includegraphics[width=\linewidth]{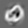}}\vspace{0.5mm}
    \frame{\includegraphics[width=\linewidth]{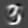}}
    \end{minipage}
    \begin{minipage}[t]{0.08\linewidth}
    \centering
    \frame{\includegraphics[width=\linewidth]{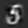}}\vspace{0.5mm}
    \frame{\includegraphics[width=\linewidth]{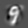}}
    \end{minipage}
    \begin{minipage}[t]{0.08\linewidth}
    \centering
    \frame{\includegraphics[width=\linewidth]{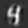}}\vspace{0.5mm}
    \frame{\includegraphics[width=\linewidth]{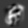}}
    \end{minipage}
    \begin{minipage}[t]{0.08\linewidth}
    \centering
    \frame{\includegraphics[width=\linewidth]{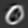}}\vspace{0.5mm}
    \frame{\includegraphics[width=\linewidth]{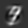}}
    \end{minipage}
    \caption{Reconstruction with top 10 Bases by PCA}
    \label{fig:PCAReconMNIST10}
    \end{subfigure}%
    \\
    \begin{subfigure}[b]{0.48\textwidth}
    \centering
    \begin{minipage}[t]{0.08\linewidth}
    \centering
    \frame{\includegraphics[width=\linewidth]{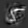}}\vspace{0.5mm}
    \frame{\includegraphics[width=\linewidth]{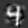}}
    \end{minipage}
    \begin{minipage}[t]{0.08\linewidth}
    \frame{\includegraphics[width=\linewidth]{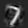}}\vspace{0.5mm}
    \frame{\includegraphics[width=\linewidth]{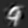}}
    \centering
    \end{minipage}
    \begin{minipage}[t]{0.08\linewidth}
    \centering
    \frame{\includegraphics[width=\linewidth]{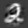}}\vspace{0.5mm}
    \frame{\includegraphics[width=\linewidth]{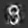}}
    \end{minipage}
    \begin{minipage}[t]{0.08\linewidth}
    \centering
    \frame{\includegraphics[width=\linewidth]{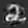}}\vspace{0.5mm}
    \frame{\includegraphics[width=\linewidth]{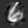}}
    \end{minipage}
    \begin{minipage}[t]{0.08\linewidth}
    \frame{\includegraphics[width=\linewidth]{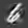}}\vspace{0.5mm}
    \frame{\includegraphics[width=\linewidth]{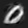}}
    \centering
    \end{minipage}
    \begin{minipage}[t]{0.08\linewidth}
    \centering
    \frame{\includegraphics[width=\linewidth]{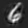}}\vspace{0.5mm}
    \frame{\includegraphics[width=\linewidth]{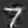}}
    \end{minipage}
    \begin{minipage}[t]{0.08\linewidth}
    \centering
    \frame{\includegraphics[width=\linewidth]{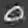}}\vspace{0.5mm}
    \frame{\includegraphics[width=\linewidth]{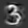}}
    \end{minipage}
    \begin{minipage}[t]{0.08\linewidth}
    \centering
    \frame{\includegraphics[width=\linewidth]{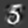}}\vspace{0.5mm}
    \frame{\includegraphics[width=\linewidth]{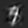}}
    \end{minipage}
    \begin{minipage}[t]{0.08\linewidth}
    \centering
    \frame{\includegraphics[width=\linewidth]{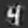}}\vspace{0.5mm}
    \frame{\includegraphics[width=\linewidth]{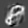}}
    \end{minipage}
    \begin{minipage}[t]{0.08\linewidth}
    \centering
    \frame{\includegraphics[width=\linewidth]{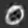}}\vspace{0.5mm}
    \frame{\includegraphics[width=\linewidth]{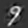}}
    \end{minipage}
    \caption{Reconstruction with top 15 Bases by MSP}
    \label{fig:MSPReconMNIST15}
    \end{subfigure}%
    ~
    \begin{subfigure}[b]{0.48\textwidth}
    \centering
    \begin{minipage}[t]{0.08\linewidth}
    \centering
    \frame{\includegraphics[width=\linewidth]{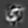}}\vspace{0.5mm}
    \frame{\includegraphics[width=\linewidth]{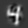}}
    \end{minipage}
    \begin{minipage}[t]{0.08\linewidth}
    \frame{\includegraphics[width=\linewidth]{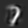}}\vspace{0.5mm}
    \frame{\includegraphics[width=\linewidth]{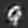}}
    \centering
    \end{minipage}
    \begin{minipage}[t]{0.08\linewidth}
    \centering
    \frame{\includegraphics[width=\linewidth]{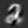}}\vspace{0.5mm}
    \frame{\includegraphics[width=\linewidth]{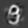}}
    \end{minipage}
    \begin{minipage}[t]{0.08\linewidth}
    \centering
    \frame{\includegraphics[width=\linewidth]{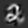}}\vspace{0.5mm}
    \frame{\includegraphics[width=\linewidth]{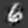}}
    \end{minipage}
    \begin{minipage}[t]{0.08\linewidth}
    \frame{\includegraphics[width=\linewidth]{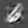}}\vspace{0.5mm}
    \frame{\includegraphics[width=\linewidth]{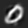}}
    \centering
    \end{minipage}
    \begin{minipage}[t]{0.08\linewidth}
    \centering
    \frame{\includegraphics[width=\linewidth]{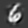}}\vspace{0.5mm}
    \frame{\includegraphics[width=\linewidth]{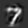}}
    \end{minipage}
    \begin{minipage}[t]{0.08\linewidth}
    \centering
    \frame{\includegraphics[width=\linewidth]{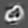}}\vspace{0.5mm}
    \frame{\includegraphics[width=\linewidth]{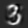}}
    \end{minipage}
    \begin{minipage}[t]{0.08\linewidth}
    \centering
    \frame{\includegraphics[width=\linewidth]{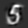}}\vspace{0.5mm}
    \frame{\includegraphics[width=\linewidth]{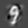}}
    \end{minipage}
    \begin{minipage}[t]{0.08\linewidth}
    \centering
    \frame{\includegraphics[width=\linewidth]{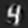}}\vspace{0.5mm}
    \frame{\includegraphics[width=\linewidth]{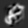}}
    \end{minipage}
    \begin{minipage}[t]{0.08\linewidth}
    \centering
    \frame{\includegraphics[width=\linewidth]{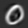}}\vspace{0.5mm}
    \frame{\includegraphics[width=\linewidth]{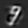}}
    \end{minipage}
    \caption{Reconstruction with top 15 Bases by PCA}
    \label{fig:PCAReconMNIST15}
    \end{subfigure}%
    \\
    \begin{subfigure}[b]{0.48\textwidth}
    \centering
    \begin{minipage}[t]{0.08\linewidth}
    \centering
    \frame{\includegraphics[width=\linewidth]{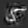}}\vspace{0.5mm}
    \frame{\includegraphics[width=\linewidth]{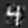}}
    \end{minipage}
    \begin{minipage}[t]{0.08\linewidth}
    \frame{\includegraphics[width=\linewidth]{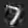}}\vspace{0.5mm}
    \frame{\includegraphics[width=\linewidth]{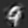}}
    \centering
    \end{minipage}
    \begin{minipage}[t]{0.08\linewidth}
    \centering
    \frame{\includegraphics[width=\linewidth]{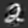}}\vspace{0.5mm}
    \frame{\includegraphics[width=\linewidth]{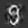}}
    \end{minipage}
    \begin{minipage}[t]{0.08\linewidth}
    \centering
    \frame{\includegraphics[width=\linewidth]{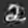}}\vspace{0.5mm}
    \frame{\includegraphics[width=\linewidth]{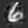}}
    \end{minipage}
    \begin{minipage}[t]{0.08\linewidth}
    \frame{\includegraphics[width=\linewidth]{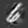}}\vspace{0.5mm}
    \frame{\includegraphics[width=\linewidth]{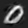}}
    \centering
    \end{minipage}
    \begin{minipage}[t]{0.08\linewidth}
    \centering
    \frame{\includegraphics[width=\linewidth]{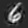}}\vspace{0.5mm}
    \frame{\includegraphics[width=\linewidth]{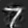}}
    \end{minipage}
    \begin{minipage}[t]{0.08\linewidth}
    \centering
    \frame{\includegraphics[width=\linewidth]{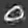}}\vspace{0.5mm}
    \frame{\includegraphics[width=\linewidth]{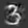}}
    \end{minipage}
    \begin{minipage}[t]{0.08\linewidth}
    \centering
    \frame{\includegraphics[width=\linewidth]{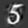}}\vspace{0.5mm}
    \frame{\includegraphics[width=\linewidth]{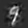}}
    \end{minipage}
    \begin{minipage}[t]{0.08\linewidth}
    \centering
    \frame{\includegraphics[width=\linewidth]{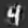}}\vspace{0.5mm}
    \frame{\includegraphics[width=\linewidth]{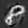}}
    \end{minipage}
    \begin{minipage}[t]{0.08\linewidth}
    \centering
    \frame{\includegraphics[width=\linewidth]{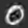}}\vspace{0.5mm}
    \frame{\includegraphics[width=\linewidth]{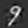}}
    \end{minipage}
    \caption{Reconstruction with top 20 Bases by MSP}
    \label{fig:MSPReconMNIST20}
    \end{subfigure}%
    ~
    \begin{subfigure}[b]{0.48\textwidth}
    \centering
    \begin{minipage}[t]{0.08\linewidth}
    \centering
    \frame{\includegraphics[width=\linewidth]{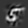}}\vspace{0.5mm}
    \frame{\includegraphics[width=\linewidth]{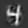}}
    \end{minipage}
    \begin{minipage}[t]{0.08\linewidth}
    \frame{\includegraphics[width=\linewidth]{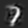}}\vspace{0.5mm}
    \frame{\includegraphics[width=\linewidth]{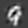}}
    \centering
    \end{minipage}
    \begin{minipage}[t]{0.08\linewidth}
    \centering
    \frame{\includegraphics[width=\linewidth]{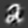}}\vspace{0.5mm}
    \frame{\includegraphics[width=\linewidth]{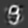}}
    \end{minipage}
    \begin{minipage}[t]{0.08\linewidth}
    \centering
    \frame{\includegraphics[width=\linewidth]{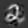}}\vspace{0.5mm}
    \frame{\includegraphics[width=\linewidth]{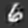}}
    \end{minipage}
    \begin{minipage}[t]{0.08\linewidth}
    \frame{\includegraphics[width=\linewidth]{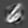}}\vspace{0.5mm}
    \frame{\includegraphics[width=\linewidth]{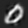}}
    \centering
    \end{minipage}
    \begin{minipage}[t]{0.08\linewidth}
    \centering
    \frame{\includegraphics[width=\linewidth]{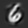}}\vspace{0.5mm}
    \frame{\includegraphics[width=\linewidth]{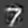}}
    \end{minipage}
    \begin{minipage}[t]{0.08\linewidth}
    \centering
    \frame{\includegraphics[width=\linewidth]{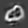}}\vspace{0.5mm}
    \frame{\includegraphics[width=\linewidth]{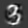}}
    \end{minipage}
    \begin{minipage}[t]{0.08\linewidth}
    \centering
    \frame{\includegraphics[width=\linewidth]{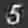}}\vspace{0.5mm}
    \frame{\includegraphics[width=\linewidth]{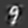}}
    \end{minipage}
    \begin{minipage}[t]{0.08\linewidth}
    \centering
    \frame{\includegraphics[width=\linewidth]{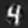}}\vspace{0.5mm}
    \frame{\includegraphics[width=\linewidth]{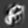}}
    \end{minipage}
    \begin{minipage}[t]{0.08\linewidth}
    \centering
    \frame{\includegraphics[width=\linewidth]{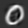}}\vspace{0.5mm}
    \frame{\includegraphics[width=\linewidth]{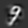}}
    \end{minipage}
    \caption{Reconstruction with top 20 Bases by PCA}
    \label{fig:PCAReconMNIST20}
    \end{subfigure}%
    \\
    \begin{subfigure}[b]{0.48\textwidth}
    \centering
    \begin{minipage}[t]{0.08\linewidth}
    \centering
    \frame{\includegraphics[width=\linewidth]{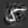}}\vspace{0.5mm}
    \frame{\includegraphics[width=\linewidth]{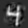}}
    \end{minipage}
    \begin{minipage}[t]{0.08\linewidth}
    \frame{\includegraphics[width=\linewidth]{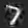}}\vspace{0.5mm}
    \frame{\includegraphics[width=\linewidth]{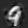}}
    \centering
    \end{minipage}
    \begin{minipage}[t]{0.08\linewidth}
    \centering
    \frame{\includegraphics[width=\linewidth]{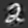}}\vspace{0.5mm}
    \frame{\includegraphics[width=\linewidth]{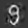}}
    \end{minipage}
    \begin{minipage}[t]{0.08\linewidth}
    \centering
    \frame{\includegraphics[width=\linewidth]{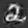}}\vspace{0.5mm}
    \frame{\includegraphics[width=\linewidth]{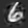}}
    \end{minipage}
    \begin{minipage}[t]{0.08\linewidth}
    \frame{\includegraphics[width=\linewidth]{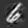}}\vspace{0.5mm}
    \frame{\includegraphics[width=\linewidth]{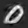}}
    \centering
    \end{minipage}
    \begin{minipage}[t]{0.08\linewidth}
    \centering
    \frame{\includegraphics[width=\linewidth]{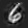}}\vspace{0.5mm}
    \frame{\includegraphics[width=\linewidth]{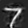}}
    \end{minipage}
    \begin{minipage}[t]{0.08\linewidth}
    \centering
    \frame{\includegraphics[width=\linewidth]{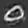}}\vspace{0.5mm}
    \frame{\includegraphics[width=\linewidth]{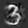}}
    \end{minipage}
    \begin{minipage}[t]{0.08\linewidth}
    \centering
    \frame{\includegraphics[width=\linewidth]{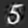}}\vspace{0.5mm}
    \frame{\includegraphics[width=\linewidth]{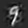}}
    \end{minipage}
    \begin{minipage}[t]{0.08\linewidth}
    \centering
    \frame{\includegraphics[width=\linewidth]{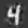}}\vspace{0.5mm}
    \frame{\includegraphics[width=\linewidth]{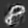}}
    \end{minipage}
    \begin{minipage}[t]{0.08\linewidth}
    \centering
    \frame{\includegraphics[width=\linewidth]{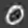}}\vspace{0.5mm}
    \frame{\includegraphics[width=\linewidth]{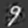}}
    \end{minipage}
    \caption{Reconstruction with top 25 Bases by MSP}
    \label{fig:MSPReconMNIST25}
    \end{subfigure}%
    ~
    \begin{subfigure}[b]{0.48\textwidth}
    \centering
    \begin{minipage}[t]{0.08\linewidth}
    \centering
    \frame{\includegraphics[width=\linewidth]{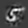}}\vspace{0.5mm}
    \frame{\includegraphics[width=\linewidth]{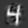}}
    \end{minipage}
    \begin{minipage}[t]{0.08\linewidth}
    \frame{\includegraphics[width=\linewidth]{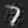}}\vspace{0.5mm}
    \frame{\includegraphics[width=\linewidth]{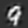}}
    \centering
    \end{minipage}
    \begin{minipage}[t]{0.08\linewidth}
    \centering
    \frame{\includegraphics[width=\linewidth]{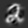}}\vspace{0.5mm}
    \frame{\includegraphics[width=\linewidth]{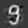}}
    \end{minipage}
    \begin{minipage}[t]{0.08\linewidth}
    \centering
    \frame{\includegraphics[width=\linewidth]{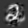}}\vspace{0.5mm}
    \frame{\includegraphics[width=\linewidth]{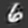}}
    \end{minipage}
    \begin{minipage}[t]{0.08\linewidth}
    \frame{\includegraphics[width=\linewidth]{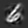}}\vspace{0.5mm}
    \frame{\includegraphics[width=\linewidth]{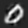}}
    \centering
    \end{minipage}
    \begin{minipage}[t]{0.08\linewidth}
    \centering
    \frame{\includegraphics[width=\linewidth]{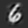}}\vspace{0.5mm}
    \frame{\includegraphics[width=\linewidth]{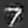}}
    \end{minipage}
    \begin{minipage}[t]{0.08\linewidth}
    \centering
    \frame{\includegraphics[width=\linewidth]{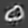}}\vspace{0.5mm}
    \frame{\includegraphics[width=\linewidth]{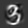}}
    \end{minipage}
    \begin{minipage}[t]{0.08\linewidth}
    \centering
    \frame{\includegraphics[width=\linewidth]{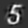}}\vspace{0.5mm}
    \frame{\includegraphics[width=\linewidth]{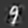}}
    \end{minipage}
    \begin{minipage}[t]{0.08\linewidth}
    \centering
    \frame{\includegraphics[width=\linewidth]{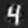}}\vspace{0.5mm}
    \frame{\includegraphics[width=\linewidth]{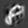}}
    \end{minipage}
    \begin{minipage}[t]{0.08\linewidth}
    \centering
    \frame{\includegraphics[width=\linewidth]{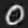}}\vspace{0.5mm}
    \frame{\includegraphics[width=\linewidth]{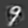}}
    \end{minipage}
    \caption{Reconstruction with top 25 Bases by PCA}
    \label{fig:PCAReconMNIST25}
    \end{subfigure}%
    
    \caption{Comparison of compact representation with learned dictionary from the MSP algorithm \ref{algo:SPOrthDL} with the PCA bases for the MNIST image dataset.}
    \label{fig:ReconMNIST}
\end{figure}

\section{Conclusions and Discussions}
\label{sec:Discussion}
In this paper, we see that a complete dictionary can be effectively and efficiently learned with nearly minimum sample complexity by the proposed simple MSP algorithm. The computational complexity of the algorithm is essentially a few dozens of SVDs, allowing us to learn dictionary for very high-dimensional data.  

Regarding sample complexity, the two main measure concentration results of Theorem \ref{Thm:MainResult} and Proposition \ref{prop:GradHatfUnionConcentrationBound} both require a sample complexity of $p=\Omega(\theta n^2\log n/\eps^2)$ for the global maximizers of the $\ell^4$-objective to be (close to) the correct $n\times n$ dictionary. This bound is consistent with our experiments in section \ref{subsec:MSPPhaseTransition}. \cite{Barak-2015,ma2016polynomial,schramm2017fast,bai2018subgradient} have reported similar empirical evidences that at least $p=\Omega(n^2)$ samples are needed to recover an $n\times n$ dictionary. 

Regarding the order of $\ell^{2k}$-norm, we adopt $\ell^4$-norm because $2k=4$ is the minimum even order that promotes sparsity. However, later works \citep{shen2020complete,xue2020blind} show that $\ell^3$-norm maximization also promotes sparsity and the MSP algorithm (PGA with infinite step size) for $\ell^4$-norm maximization naturally generalizes to $\ell^3$-norm maximization. 

Regarding optimality, Proposition \ref{prop:ConvergenceToSaddle} shows than random initialization of PGA algorithm with any step size finds a critical points of the $\ell^4$-norm over $\msf O(n;\bb R)$, Theorem \ref{Thm:MSPLocalConvergence} has shown the local convergence of the proposed MSP Algorithm \ref{algo:SPOrthD} with a cubic rate for general $n$ around global maximizers, and Proposition \ref{prop:MSPGlobalConvn=2} has proven its global convergence for $n=2$. It is natural to conjecture that the signed-permutations would be the only stable maximizers of the $\ell^4$-norm over the entire orthogonal group, hence the proposed algorithm converges globally, just like the experiments have indicated. Nevertheless, a rigorous proof is still elusive at this point. 

Although some initial experiments have already suggested that the proposed algorithm works stably with mild noise and real data -- showing clear advantages of the so learned dictionary over the classic PCA bases, the proposed algorithm actually generalizes well to data with dense noise, outliers, and sparse corruptions \citep{zhai2019understanding}. This paper has only addressed the case with a complete (square) dictionary. But there are ample reasons to believe that similar formulations and techniques presented in this paper can be extended to the over-complete case $\mb D \in \bb R^{n\times m}$ with $n < m$, at least when $m = O(n)$.

\section{Acknowledgement}
We would like to thank Yichao Zhou and Dr. Chong You of Berkeley EECS Department for stimulating discussions during preparation of this manuscript. We would like to thank professor Ju Sun of University of Minnesota CSE and Dr. Julien Mairal of Inria for references to existing work and experimental comparison. We would also like to thank Haozhi Qi of Berkeley EECS for help with the experiments. YZ would like to thank professor Yuejie Chi of CMU ECE for meaningful comments on the local convergence result. YZ would also like to thank Ye Xue from HKUST and Jinxing Wang from CUHK for insightful suggestions and proofreading. Yi likes to thank professor Bernd Sturmfels of Berkeley Math Department for help with analyzing algebraic properties of  critical points of the $\ell^4$-norm over the orthogonal group and thank professor Alan Weinstein of Berkeley Math Department for discussions and references on the role of $\ell^4$-norm in spherical harmonic analysis. Yi would like to acknowledge that this work is partially supported by research grant N00014-20-1-2002 from Office of Naval Research (ONR) and research grant from Tsinghua-Berkeley Shenzhen Institute (TBSI). JW gratefully acknowledges support from NSF grants  1733857, 1838061, 1740833, and 1740391, and thanks Sam Buchanan for discussions related to the geometry of $\ell^4$.

\newpage
\appendices
\section{Proofs of Section \ref{sec:StatCharacter}}

\subsection{Proof of Lemma \ref{lemma:orthproperty}}
\begin{claim}
    $\forall \theta \in (0,1)$, let $\mb X_o\in \bb R^{n\times p}$, $x_{i,j}\sim_{iid}\text{BG}(\theta)$, $\mb D_o\in \msf O(n;\bb R)$ is any orthogonal matrix, and $\mb Y = \mb D_o\mb X_o$. Then, $\forall \mb A\in \msf O(n;\bb R)$, we have
    \begin{equation}
        \frac{1}{3p\theta}f(\mb A) = (1-\theta)g(\mb {AD}_o) + \theta n.
    \end{equation}
\end{claim}

\begin{proof}
    \label{proof:OrthProperty}
    For simplicity, since $\mb D_o\in \msf O(n;\bb R)$, let $\mb W = \mb {AD}_o$, we know that $\mb W\in \msf O(n;\bb R)$. By the fact that $\mb W\in \msf O(n;\bb R)$, we have:
    \begin{equation}
        \label{eq:GFRelation}
        \begin{split}
            &n = \sum_{i=1}^n\Big(\sum_{k=1}^nw_{i,k}^2\Big)^2= \underbrace{\sum_{i=1}^n\sum_{k=1}^nw_{i,k}^4}_{g(\mb W)} + 2\sum_{i=1}^n\sum_{1\leq k_1<k_2\leq n}w_{i,k_1}^2w_{i,k_2}^2\\
            \implies& \sum_{i=1}^n\sum_{1\leq k_1<k_2\leq n}w_{i,k_1}^2w_{i,k_2}^2 = \frac{n-g(\mb W)}{2}.
        \end{split}
    \end{equation}
    Next, we calculate $\frac{1}{3p\theta}f(\mb A)$:
    \begin{equation}
    \label{eq:L4ExpDerivation}
    \begin{split}
        \frac{f(\mb A)}{3p\theta} & = \frac{1}{3p\theta}\bb E_{\mb X_o}\hat{f}(\mb A; \mb Y) = \frac{1}{3p\theta}\bb E_{\mb X_o}\norm{\mb {AD}_o\mb X_o}{4}^4\\
        & = \frac{1}{3p\theta}\bb E_{\mb X_o}\norm{\mb W\mb X_o}{4}^4 = \frac{1}{3p\theta}\sum_{i=1}^n\sum_{j=1}^p\bb E_{\mb X_0}\Big(\sum_{k=1}^nw_{i,k}x_{k,j}\Big)^4\\
        & =\frac{1}{3p\theta}\sum_{i=1}^n\sum_{j=1}^p\Big[3\theta\sum_{k=1}^n w_{i,k}^4 + 6\theta^2\sum_{1\leq k_1<k_2\leq n}w_{i,k_1}^2w_{i,k_2}^2 \Big]\\
        & = \underbrace{\sum_{i=1}^n\sum_{k=1}^nw_{i,k}^4}_{g(\mb W)} + 2\theta \underbrace{\sum_{i=1}^n\sum_{1\leq k_1<k_2\leq n}w_{i,k_1}^2w_{i,k_2}^2}_{\frac{n-g(\mb W)}{2}}\\
        & = (1-\theta)g(\mb W)+\theta n = (1-\theta)g(\mb {AD}_o)+\theta n\quad \text{(By \eqref{eq:GFRelation})},
\end{split}
\end{equation}
which completes the proof.
\end{proof}

\subsection{Proof of Lemma \ref{lemma:HatfUnionConcentrationBound}}
\begin{claim}[Concentration Bound of $\frac{1}{np}\hat{f}(\cdot,\cdot)$]
    $\forall \theta\in(0,1)$, if $\mb X\in \bb R^{n\times p}$, $x_{i,j}\sim_{iid}\text{BG}(\theta)$, for any $\delta>0$, the following inequality holds  
    \begin{equation}
        \label{eq:HatfUnionConcentrationBound}
         \begin{split}
            &\bb P\Bigg(\sup_{\mb W\in \msf O(n;\bb R)}\frac{1}{np}\abs{\norm{\mb W \mb X}{4}^4-\bb E \norm{\mb W \mb X}{4}^4}\geq \delta\Bigg)\\
            <& \exp\Bigg( -\frac{3p\delta^2}{c_1\theta+8n(\ln p)^4\delta}+n^2\ln\Big(\frac{60np(\ln p)^4}{\delta}\Big) \Bigg)\\
            &+\exp\Bigg(-\frac{p\delta^2}{c_2\theta}+n^2\ln\Big(\frac{60np(\ln p)^4}{\delta}\Big)\Bigg)+2np\theta \exp\Bigg(-\frac{(\ln p)^2}{2}\Bigg),
        \end{split}
    \end{equation}
    for some constants $c_1>10^4,c_2>3360$. Moreover 
    \begin{equation}
        \label{eq:DecentHatfUnionConcentrationBound}
        \begin{split}
             &\exp\Bigg( -\frac{3p\delta^2}{c_1\theta+8n(\ln p)^4\delta}+n^2\ln\Big(\frac{60np(\ln p)^4}{\delta}\Big) \Bigg)\\
            &+\exp\Bigg(-\frac{p\delta^2}{c_2\theta}+n^2\ln\Big(\frac{60np(\ln p)^4}{\delta}\Big)\Bigg)+2np\theta \exp\Bigg(-\frac{(\ln p)^2}{2}\Bigg)\leq \frac{1}{p},
        \end{split}
    \end{equation}
    when $p=\Omega(\theta n^2\ln n/\delta^2)$.
\end{claim}

\begin{proof}
    \label{proof:HatfUnionConcentrationBound}
    Let $\bar{\mb X}\in \bb R^{n\times p}$ denote the truncated $\mb X$ by bound $B$
    \begin{equation}
        \bar{x}_{i,j}=
        \begin{cases}
            x_{i,j}&\textrm{if}\quad \abs{x_{i,j}}\leq B\\
            0&\textrm{else}
        \end{cases}.
    \end{equation}
    By Lemma \ref{lemma:BGMatrixTruncation}, we know that $\norm{\mb X}{\infty}\leq B$ happens with probability at least $1-2np\theta\exp(-B^2/2)$ and $\bar{\mb X} = \mb X$ holds whenever $\norm{\mb X}{\infty}\leq B$. So we know that $\bar{\mb X} \neq \mb X$ holds with probability at most $2np\theta\exp(-B^2/2)$, and thus:
    \begin{equation}
        \label{eq:UnionBoundWithAndWithoutTruncation}
        \begin{split}
            &\bb P\Bigg(\sup_{\mb W\in\msf O(n;\bb R)}\frac{1}{np}\abs{\norm{\mb W \mb X}{4}^4-\bb E \norm{\mb W \mb X}{4}^4}\geq \delta\Bigg)\\
            \leq& \bb P \Bigg(\sup_{\mb W\in\msf O(n;\bb R)}\frac{1}{np}\abs{\norm{\mb W \mb X}{4}^4-\bb E \norm{\mb W \mb X}{4}^4}\geq \delta,\mb X = \bar{\mb X}\Bigg)+\bb P\Big(\mb X \neq \bar{\mb X}\Big)\\
            \leq&\bb P\Bigg(\sup_{\mb W\in\msf O(n;\bb R)}\frac{1}{np}\abs{\norm{\mb W \bar{\mb X}}{4}^4-\bb E \norm{\mb W \mb X}{4}^4}\geq \delta\Bigg)+ 2np\theta e^{-\frac{B^2}{2}}.
        \end{split}
    \end{equation}
    \textbf{$\eps-$net Covering.} For any positive $\eps$ satisfying:
    \begin{equation}
        \label{eq:EpsCondition}
        \eps\leq \frac{\delta}{10 npB^4},
    \end{equation}
     by Lemma \ref{lemma:epsCoveringStiefelManifold}, there exists an $\eps-$nets:
    \begin{equation}
        \mc S_\eps = \{\mb W_1,\mb W_2,\dots,\mb W_{|\mc S_\eps|}\},
    \end{equation} 
    which covers $\msf O(n;\bb R)$
    \begin{equation}
        \label{eq:EpsBallCoveringOrthogonalGroup}
        \msf O(n;\bb R)\subset \bigcup_{l=1}^{|\mc S_\eps|} \bb B(\mb W_l,\eps),
    \end{equation}
    in operator norm $\norm{\cdot}{2}$. Moreover, we have:
    \begin{equation}
        \label{eq:EpsBallMaxNumber}
        |\mc S_\eps|\leq \Big(\frac{6}{\eps}\Big)^{n^2}.    
    \end{equation}
    So $\forall \mb W\in \msf O(n;\bb R)$, there exists $l\in [|\mc S_\eps|]$, such that $\norm{\mb W-\mb W_l}{2}\leq \eps$. Thus, we have:
    \begin{equation}
        \label{eq:ProbabilityTruncationEpsNet}
        \begin{split}
            &\bb P\Bigg(\frac{1}{np}\abs{\norm{\mb {W}\bar{\mb X}}{4}^4-\bb E\norm{\mb {WX}}{4}^4}\geq \delta \Bigg)\\
            \leq & \bb P\Bigg( \frac{1}{np}\Big(\abs{\norm{\mb W\bar{\mb X}}{4}^4-\norm{\mb W_l\bar{\mb X}}{4}^4}+\abs{\norm{\mb W_l\bar{\mb X}}{4}^4-\bb E\norm{\mb {W}_l\mb X}{4}^4}+\abs{\bb E\norm{\mb W_l\mb X}{4}^4-\bb E\norm{\mb {WX}}{4}^4}\Big)\geq\delta\Bigg)\\
            \leq & \bb P\Bigg(\frac{1}{np}\abs{\norm{\mb W_l\bar{\mb X}}{4}^4-\bb E\norm{\mb {W}_l\mb X}{4}^4}+[4npB^4+12n\theta(1-\theta)]\norm{\mb W-\mb W_l}{2}\geq \delta \Bigg) \quad \text{(Lemma \ref{lemma:HatfLipschitzBound}, \ref{lemma:fLipschitzBound})}\\
            < &\bb P \Bigg( \frac{1}{np}\abs{\norm{\mb W_l\bar{\mb X}}{4}^4-\bb E\norm{\mb {W}_l\mb X}{4}^4}+5npB^4\eps\geq \delta \Bigg)\quad \text{($p,B$ are large numbers, $\norm{\mb W-\mb W_l}{2}\leq \eps$)}\\
            \leq & \bb P\Bigg(\frac{1}{np}\abs{\norm{\mb W_l\bar{\mb X}}{4}^4-\bb E\norm{\mb {W}_l\mb X}{4}^4} + \frac{\delta}{2}\geq \delta \Bigg)\quad \text{(We assume $\eps\leq \frac{\delta}{10 npB^4}$)}\\
             = & \bb P\Bigg(\frac{1}{np}\abs{\norm{\mb W_l\bar{\mb X}}{4}^4-\bb E\norm{\mb {W}_l\mb X}{4}^4} \geq \frac{\delta}{2} \Bigg).
        \end{split}
    \end{equation}
    \textbf{Analysis.}
    For random variable $\bar{\mb X}$, we have:
    \begin{equation}
        \label{eq:BoundMatrixTruncationEXP}
        \abs{\bb E\norm{\mb W_l\mb X}{4}^4 - \bb E \norm{\mb W_l\bar {\mb X}}{4}^4} = \abs{\bb E\big[ \norm{\mb W_l\mb X}{4}^4 \cdot\indicator{\{\norm{\mb X }{\infty}>B\}}\big]} \leq\sqrt{\bb E \norm{\mb W_l\mb X}{4}^8} \sqrt{\bb E\indicator{\{\norm{\mb X}{\infty}>B\}}},   
    \end{equation}
    moreover, if we view $\norm{\mb W_l\mb X}{4}^4=\sum_{j=1}^p\norm{\mb W_l\mb x_j}{4}^4$ as sum of $p$ independent random variables, we have:
    \begin{equation}
        \label{eq:BoundMatrixPowerEight}
        \begin{split}
            & \bb E\norm{\mb {W}_l\mb {X}}{4}^8\\
            = & \bb E\Bigg(\Big[\sum_{j=1}^p \norm{\mb {W}_l\mb {x}_j}{4}^4\Big]^2\Bigg) = \bb E\Bigg(\sum_{j=1}^p\norm{\mb W_l\mb x_j}{4}^8\Bigg)+2\bb E\Bigg(\sum_{1\leq j_1<j_2\leq p}\norm{\mb W_l\mb x_{j_1}}{4}^4\norm{\mb W_l\mb x_{j_2}}{4}^4\Bigg)\\
            = & \bb E\Bigg(\sum_{j=1}^p\norm{\mb W_l\mb x_j}{4}^8\Bigg)+2\sum_{1\leq j_1<j_2\leq p}\bb E\norm{\mb W_l\mb x_{j_1}}{4}^4 \bb E\norm{\mb W_l\mb x_{j_2}}{4}^4\quad \text{($\mb x_j$ are independent)}\\
            = & \sum_{j=1}^p\bb E\norm{\mb W_l\mb x_j}{4}^8 + p(p-1)\Big[3\theta(1-\theta)\norm{\mb W_l}{4}^4+3\theta^2 n\Big]^2 \quad \text{(By Lemma \ref{lemma:orthproperty})}\\
            \leq& Cpn^2\theta + 9p(p-1)n^2\theta^2\quad \text{(By \eqref{eq:ExpWXPowerEight} of Lemma \ref{lemma:HatfEightMoment})}\\
            \leq & C_1 p^2n^2\theta^2 
        \end{split}
    \end{equation}
    for some constants $C_1>9$, for sufficiently large $n,p$. Also, by Lemma \ref{lemma:BGMatrixTruncation}, we have:
    \begin{equation}
        \label{eq:BoundIndicator}
        \bb E\indicator{\{\norm{\mb X}{\infty}>B\}}\leq 2np\theta e^{-\frac{B^2}{2}}. 
    \end{equation}
    Substitute the results of previous two inequalities \eqref{eq:BoundMatrixPowerEight}, \eqref{eq:BoundIndicator} into \eqref{eq:BoundMatrixTruncationEXP}, yields
    \begin{equation}
        \begin{split}
            \abs{\bb E\norm{\mb W_l\mb X}{4}^4 - \bb E \norm{\mb W_l\bar {\mb X}}{4}^4}\leq& \sqrt{C_1} pn\theta\cdot \sqrt{2np\theta \exp\big(-B^2/2\big)}= C_2 n^{\frac{3}{2}}p^{\frac{3}{2}}\theta^{\frac{3}{2}}e^{-\frac{B^2}{4}},
        \end{split}
    \end{equation}
    for some constants $C_2>3\sqrt{2}$. Hence, when
    \begin{equation}
        \label{eq:BoundB}
        B > 2\sqrt{\ln \Big(\frac{C_3 n^{\frac{1}{2}}p^{\frac{1}{2}}\theta^{\frac{3}{2}}}{\delta}\Big)},
    \end{equation}
    for some constants $C_3>12\sqrt 2$, we have:
    \begin{equation}
        \label{eq:XAndBarXBound}
        \frac{1}{np}\abs{\bb E\norm{\mb W_l\mb X}{4}^4 - \bb E \norm{\mb W_l\bar {\mb X}}{4}^4}\leq C_2 n^{\frac{1}{2}}p^{\frac{1}{2}}\theta^{\frac{3}{2}}e^{-\frac{B^2}{4}} < \frac{\delta}{4}. 
    \end{equation}
    Therefore, combine \eqref{eq:ProbabilityTruncationEpsNet}, we have:
    \begin{equation}
        \label{eq:ProbabilityNormalToTruncationEXP1}
        \begin{split}
            & \bb P\Bigg(\frac{1}{np}\abs{\norm{\mb W \bar{\mb X}}{4}^4-\bb E \norm{\mb W \mb X}{4}^4}\geq \delta\Bigg) < \bb P\Bigg(\frac{1}{np}\abs{\norm{\mb W_l \bar{\mb X}}{4}^4-\bb E \norm{\mb W_l \mb X}{4}^4}\geq \frac{\delta}{2}\Bigg)\\
            = & \bb P\Bigg(\frac{1}{np}\abs{\norm{\mb W_l \bar{\mb X}}{4}^4-\bb E \norm{\mb W_l \bar{\mb X}}{4}^4+\bb E \norm{\mb W_l \bar{\mb X}}{4}^4-\bb E \norm{\mb W_l \mb X}{4}^4}\geq \frac{\delta}{2}\Bigg)\\
            \leq & \bb P\Bigg(\frac{1}{np}\abs{\norm{\mb W_l \bar{\mb X}}{4}^4-\bb E \norm{\mb W_l \bar{\mb X}}{4}^4}+\frac{1}{np}\abs{\bb E \norm{\mb W_l \bar{\mb X}}{4}^4-\bb E \norm{\mb W_l \mb X}{4}^4}\geq \frac{\delta}{2}\Bigg)\\
            \leq & \bb P\Bigg(\frac{1}{np}\abs{\norm{\mb W_l \bar{\mb X}}{4}^4-\bb E \norm{\mb W_l \bar{\mb X}}{4}^4}\geq \frac{\delta}{4}\Bigg)\quad\text{By \eqref{eq:XAndBarXBound}}\\
            =& \underbrace{\bb P\Bigg(\frac{1}{np}\Big(\norm{\mb W_l \bar{\mb X}}{4}^4-\bb E \norm{\mb W_l \bar{\mb X}}{4}^4 \Big)\geq \frac{\delta}{4}\Bigg)}_{\Gamma_1}+\underbrace{\bb P\Bigg(\frac{1}{np}\Big(\norm{\mb W_l \bar{\mb X}}{4}^4-\bb E \norm{\mb W_l \bar{\mb X}}{4}^4 \Big)\leq -\frac{\delta}{4}\Bigg)}_{\Gamma_2}.
        \end{split}
    \end{equation}
    \textbf{Point-wise Bernstein Inequality}. Next, we will apply Bernstein's inequality on $\norm{\mb W_l\bar{\mb X}}{4}^4$ to bound its upper ($\Gamma_1$) and lower tail ($\Gamma_2$). Note that we can view $\norm{\mb W_l\bar{\mb X}}{4}^4$ as sum of $p$ independent variables $\norm{\mb W_l\bar{\mb X}}{4}^4 = \sum_{j=1}^p\norm{\mb {W}_l\bar{\mb {x}}_j}{4}^4$, and each of them is bounded by
    \begin{equation}
        \norm{\mb W_l\bar{\mb x}_j}{4}^4=\norm{\mb W_l\bar{\mb x}_j}{2}^4\cdot\norm{\frac{\mb W_l\bar{\mb x}_j}{\norm{\mb W_l\bar{\mb x}_j}{2}}}{4}^4\leq \norm{\mb W_l\bar{\mb x}_j}{2}^4 =\norm{\bar{\mb x}_j}{2}^4 \leq n^2B^4.
    \end{equation}
    Also, in order to bound $\bb E\norm{\mb {W}_l\bar {\mb x}_j}{4}^8$, we consider
    \begin{equation}
        \bb E\norm{\mb {W}_l\mb {x}_j}{4}^8-\bb E\norm{\mb {W}_l\bar {\mb x}_j}{4}^8 = \bb E\big( \norm{\mb {W}_l\mb {x}_j}{4}^8-\norm{\mb {W}_l\bar {\mb x}_j}{4}^8\big) = \bb E\big( \norm{\mb {W}_l\mb {x}_j}{4}^8\cdot \indicator{\{\norm{\mb X}{\infty}>B\}} \big) \geq 0,
    \end{equation}
    along with \eqref{eq:ExpWXPowerEight} in Lemma \ref{lemma:HatfEightMoment}, this implies 
    \begin{equation}
        \bb E\norm{\mb W_l\bar{\mb x}_j}{4}^8\leq \bb E\norm{\mb {W}_l\mb {x}_j}{4}^8\leq Cn^2\theta,
    \end{equation}
    where $C>105$ is the same constant as \eqref{eq:BoundMatrixPowerEight}. Now we apply Bernstein's inequality on $\Gamma_1$:
    \begin{equation}
        \label{eq:Gamma1Bound}
        \begin{split}
            \Gamma_1=&\bb P\Bigg(\frac{1}{np}\Big(\norm{\mb W_l \bar{\mb X}}{4}^4-\bb E \norm{\mb W_l \bar{\mb X}}{4}^4\Big)\geq \frac{\delta}{4}\Bigg)
            \\
            =&\bb P\Bigg(\sum_{j=1}^p\Big(\norm{\mb W_l \bar{\mb x}_j}{4}^4-\bb E \norm{\mb W_l \bar{\mb x}_j}{4}^4\Big)\geq p\cdot \frac{n\delta}{4}\Bigg)\\
            \leq& \exp\Bigg(-\frac{pn^2\delta^2/16}{2\big(\frac{1}{p} \sum_{j=1}^p\bb E\norm{\mb {W}_l\bar{\mb x}_j}{4}^8 + n^2B^4\cdot n\delta/12 \big)} \Bigg)\\
            =&\exp\Bigg(-\frac{3pn^2\delta^2}{\frac{96}{p}\sum_{j=1}^p\bb E\norm{\mb W_l\bar{\mb x}_j}{4}^8+8n^3B^4\delta}\Bigg)\\
            \leq& \exp \Bigg(-\frac{3pn^2\delta^2}{\frac{96}{p}\sum_{j=1}^p\bb E\norm{\mb W_l\mb x_j}{4}^8+8n^3B^4\delta}\Bigg)\\
            \leq &\exp \Bigg(- \frac{3pn^2\delta^2}{c_1n^2\theta+8n^3B^4\delta} \Bigg)= \exp\Bigg(-\frac{3p\delta^2}{c_1 \theta+8nB^4\delta}\Bigg),
        \end{split}
    \end{equation}
    for a constant $c_1>10^{4}$. Next, we apply Bernstein's inequality on $\Gamma_2$, along with result in \eqref{eq:BoundMatrixPowerEight}. Note that $\forall j\in [p]$, $\norm{\mb W_l \bar{\mb x}_j}{4}^4$ is lower bounded by $0$, hence we have:
    \begin{equation}
        \label{eq:Gamma2Bound}
        \begin{split}
            \Gamma_2 =&\bb P\Bigg(\frac{1}{np}\Big(\norm{\mb W_l\bar{\mb X}}{4}^4-\bb E \norm{\mb W_l\bar{\mb X}}{4}^4 \Big)\leq -\frac{\delta}{4}\Bigg)\\
            =&\bb P\Bigg(\sum_{j=1}^p\Big(\norm{\mb W_l\bar{\mb x}_j}{4}^4-\bb E \norm{\mb W_l\bar{\mb x}_j}{4}^4\Big)\leq -p\cdot\frac{n\delta}{4}\Bigg)\\
            \leq & \exp\Bigg(-\frac{pn^2\delta^2/16}{\frac{2}{p}\sum_{j=1}^p\bb E\norm{\mb W_l\bar{\mb x}_j}{4}^8}\Bigg)\leq \exp\Bigg(-\frac{pn^2\delta^2/16}{\frac{2}{p}\sum_{j=1}^p\bb E\norm{\mb W_l\mb x_j}{4}^8}\Bigg)\\
            \leq&\exp\Bigg(-\frac{pn^2\delta^2}{32Cn^2\theta}\Bigg)\leq\exp\Bigg(-\frac{p\delta^2}{c_2\theta}\Bigg),
        \end{split}
    \end{equation}
    where $C>105$ is the same constant as \eqref{eq:BoundMatrixPowerEight} and $c_2>3360$ is another constant. Replacing \eqref{eq:Gamma1Bound} and \eqref{eq:Gamma2Bound} into \eqref{eq:ProbabilityNormalToTruncationEXP1}, yields
    \begin{equation}
        \label{eq:PointWiseBoundInEpsBall}
        \begin{split}
            \bb P\Bigg(\frac{1}{np}\abs{\norm{\mb W \bar{\mb X}}{4}^4-\bb E \norm{\mb W \mb X}{4}^4}\geq \delta\Bigg) \leq \Gamma_1 + \Gamma_2\leq   \exp\Bigg(-\frac{3p\delta^2}{c_1 \theta+8nB^4\delta}\Bigg) + \exp\Bigg(-\frac{p\delta^2}{c_2\theta}\Bigg)
        \end{split}
    \end{equation}
    for some constants $c_1>10^4,c_2>3360$.\\
    \textbf{Union Bound.} Now, we will give a union bound for 
    \begin{equation}
        \bb P\Bigg(\sup_{\mb W\in\msf O(n;\bb R)}\frac{1}{np}\abs{\norm{\mb W \bar{\mb X}}{4}^4-\bb E \norm{\mb W \mb X}{4}^4}\geq \delta\Bigg).
    \end{equation}
    Notice that
    \begin{equation}
        \label{eq:UnionConcentrationBoundLastStep}
        \begin{split}
            & \bb P\Bigg(\sup_{\mb W\in\msf O(n;\bb R)}\frac{1}{np}\abs{\norm{\mb W \bar{\mb X}}{4}^4-\bb E \norm{\mb W \mb X}{4}^4}\geq \delta\Bigg)\\
            \leq & \sum_{l=1}^{|\mc S_\eps|}\bb P\Bigg(\sup_{\mb W\in\bb B(\mb W_l,\eps)}\frac{1}{np}\abs{\norm{\mb W \bar{\mb X}}{4}^4-\bb E \norm{\mb W \mb X}{4}^4}\geq \delta\Bigg)\quad \text{(By $\eps-$covering in \eqref{eq:EpsBallCoveringOrthogonalGroup})}\\
            \leq & \sum_{l=1}^{|\mc S_\eps|} \Bigg[\exp\Bigg(-\frac{3p\delta^2}{c_1 \theta+8nB^4\delta}\Bigg) + \exp\Bigg(-\frac{p\delta^2}{c_2\theta}\Bigg)\Bigg]\quad \text{(By \eqref{eq:PointWiseBoundInEpsBall} when $\eps\leq \frac{\delta}{10npB^4}$)}\\
            \leq & \Big(\frac{6}{\eps}\Big)^{n^2}\Bigg[\exp\Bigg(-\frac{3p\delta^2}{c_1 \theta+8nB^4\delta}\Bigg) + \exp\Bigg(-\frac{p\delta^2}{c_2\theta}\Bigg)\Bigg] \quad \text{(By \eqref{eq:EpsBallMaxNumber})}\\
            =& \exp\Bigg(n^2\ln\Big(\frac{60npB^4}{\delta}\Big)\Bigg)\Bigg[\exp\Bigg(-\frac{3p\delta^2}{c_1 \theta+8nB^4\delta}\Bigg) + \exp\Bigg(-\frac{p\delta^2}{c_2\theta}\Bigg)\Bigg]\quad \text{(Let $\eps = \frac{\delta}{10npB^4}$)}\\
            =& \exp\Bigg( -\frac{3p\delta^2}{c_1 \theta+8nB^4\delta}+n^2\ln\Big(\frac{60npB^4}{\delta}\Big) \Bigg)+\exp\Bigg(-\frac{p\delta^2}{c_2\theta}+n^2\ln\Big(\frac{60npB^4}{\delta}\Big)\Bigg).
        \end{split}
    \end{equation}
    Note that \eqref{eq:BoundB} requires a lower bound on $B$, here we can choose $B = \ln p$, which satisfies \eqref{eq:BoundB} when $p$ is large enough (say $p = \Omega(n)$). Combine \eqref{eq:UnionBoundWithAndWithoutTruncation} and substitute $B=\ln p$ into \eqref{eq:UnionConcentrationBoundLastStep}, we have:
    \begin{equation}
        \label{eq:UnionConcentrationBoundFinal}
        \begin{split}
            &\bb P\Bigg(\sup_{\mb W\in\msf O(n;\bb R)}\frac{1}{np}\abs{\norm{\mb W \mb X}{4}^4-\bb E \norm{\mb W \mb X}{4}^4}\geq \delta\Bigg)\\
            \leq&\bb P\Bigg(\sup_{\mb W\in\msf O(n;\bb R)}\frac{1}{np}\abs{\norm{\mb W \bar{\mb X}}{4}^4-\bb E \norm{\mb W \mb X}{4}^4}\geq \delta\Bigg)+ 2np\theta e^{-\frac{B^2}{2}}\\
            \leq& \exp\Bigg( -\frac{3p\delta^2}{c_1 \theta+8n(\ln p)^4\delta}+n^2\ln\Big(\frac{60np(\ln p)^4}{\delta}\Big) \Bigg)\\
            &+\exp\Bigg(-\frac{p\delta^2}{c_2\theta}+n^2\ln\Big(\frac{60np(\ln p)^4}{\delta}\Big)\Bigg)+2np\theta \exp\Bigg(-\frac{(\ln p)^2}{2}\Bigg),
        \end{split}
    \end{equation}
    for some constants $c_1>10^4,c_2>3360$, which completes the proof for \eqref{eq:HatfUnionConcentrationBound}. When $p=\Omega(\theta n^2 \ln n/\delta^2)$, we have: 
    \begin{equation}
        \begin{split}
            &\exp\Bigg( -\frac{3p\delta^2}{c_1\theta+8n(\ln p)^4\delta}+n^2\ln\Big(\frac{60np(\ln p)^4}{\delta}\Big) \Bigg)\\
            &+\exp\Bigg(-\frac{p\delta^2}{c_2\theta}+n^2\ln\Big(\frac{60np(\ln p)^4}{\delta}\Big)\Bigg)+2np\theta \exp\Bigg(-\frac{(\ln p)^2}{2}\Bigg)\leq \frac{1}{3p}+\frac{1}{3p}+\frac{1}{3p}=\frac{1}{p}, 
        \end{split}
    \end{equation}
    which completes the proof for \eqref{eq:DecentHatfUnionConcentrationBound}.
\end{proof}

\subsection{Proof of Lemma \ref{lemma:G(A)bound}}
\begin{claim}[Extrema of $\ell^4$-Norm over Orthogonal Group]
     For any orthogonal matrix $\mb A\in \msf O(n;\bb R)$, $g(\mb A) = \norm{\mb A}{4}^4\in[1, n]$, $g(\mb A)$ reaches maximum if and only if $\mb A\in \text{SP}(n)$.
\end{claim}
\begin{proof}
    \label{proof:G(A)bound}
    For the maximum of $g(\mb A)$
    \begin{equation}
        \label{G(A)Max}
        g(\mb A) = \sum_{i=1}^n\sum_{j=1}^n a_{i,j}^4\leq\sum_{i=1}^n\sum_{j=1}^n a_{i,j}^4+ 2\sum_{i=1}^n\sum_{1\leq j_1<j_2\leq n}a_{i,j_1}^2a_{i,j_2}^2 = \sum_{i=1}^n \Big( \sum_{j=1}^na_{i,j}^2\Big)^2 = n.
    \end{equation}
    And when equality holds, we have
    \begin{equation}
        \label{OrthCondition}
        a_{i,j_1}a_{i,j_2} = 0,\quad \forall i,j_1\neq, j_2\in [n], 
    \end{equation}
    which implies $\mb A\in\text{SP}(n)$.
\end{proof}

\subsection{Proof of Lemma \ref{lemma:L4ExtremaBoundOrthGroup}}
\begin{claim}[Approximate Maxima of $\ell^4$-Norm over the Orthogonal Group]
    Suppose $\mb W$ is an orthogonal matrix: $\mb W\in \msf O(n;\bb R)$. $\forall \eps\in[0,1]$, if $\frac{1}{n}\norm{\mb W}{4}^4\geq1-\eps$, then $\exists \mb P\in \text{SP}(n)$, such that 
    \begin{equation}
        \label{eq:WMinusPermBound}
        \frac{1}{n}\norm{\mb W - \mb P}{F}^2 \leq 2\eps.
    \end{equation}
\end{claim}

\begin{proof}
    \label{proof:L4ExtremaBoundOrthGroup}
    Note that the condition $\frac{1}{n}\norm{\mb W}{4}^4>1-\eps$ can be viewed as 
    \begin{equation}
        \frac{1}{n}\sum_{i=1}^n\norm{\mb w_i}{4}^4\geq 1-\eps,
    \end{equation}
    where each column vector $\mb w_i$ of $\mb W$ satisfies $\mb w_i\in \bb S^{n-1}$. By Lemma \ref{lemma:L4MultipleVectorExtremaBoundSphere}, we we know there exists $j_1,j_2,\dots j_n\in[n]$, such that 
    \begin{equation}
        \label{eq:SumwPMBound}
        \frac{1}{n}\sum_{i=1}^n\norm{\mb w_i - s_i\mb e_{j_i}}{2}^2 \leq 2\eps,
    \end{equation}
    where $\mb e_{j_i}$ are one vector in canonical basis and $s_i\in \{1,-1\}$ indicates the sign of $\mb e_{j_i}$, $\forall i\in [n]$. Since $\mb W\in \msf O(n;\bb R)$, one can easily show that $\forall i_1\neq i_2, i_1,i_2\in[n]$, the canonical vector they are corresponding to $\mb e_{j_{i_1}},\mb e_{j_{i_2}}$ are different, that is, $j_{i_1}\neq j_{i_2}$ (otherwise suppose $\exists i_1\neq i_2$, such that $\mb w_{i_1}$ and $\mb w_{i_2}$ correspond to the same canonical vector $\mb e_j$ in \eqref{eq:SumwPMBound}, one can easily show that $\mb w_{i_1}^*\mb w_{i_2}\neq 0$, which contradicts with the orthogonality of $\mb W$). Hence, there exists a sign permutation matrix: 
    \begin{equation}
        \mb P = [sign(w_{j_1,1})\mb e_{j_1}, sign(w_{j_2,2})\mb e_{j_2},\dots ,sign(w_{j_n,n})\mb e_{j_n}],
    \end{equation}
    such that
    \begin{equation}
        \norm{\mb W - \mb P}{F}^2=\sum_{i=1}^n\norm{\mb w_i-sign(w_{j_i,1})\mb e_{j_i}}{2}^2\leq 2n\eps\implies
        \frac{1}{n}\norm{\mb W - \mb P}{F}^2\leq 2\eps,
    \end{equation}
    which completes the proof.
\end{proof}

\subsection{Proof of Theorem \ref{Thm:MainResult}}
\begin{claim}[Correctness of Global Maxima]
    $\forall \theta\in(0,1)$, assume $\mb X_o=\{x_{i,j}\}\in\bb R^{n\times p}$ is a Bernoulli-Gaussian matrix, $\mb D_o\in \msf O(n;\bb R)$ is any orthogonal matrix, and $\mb Y = \mb D_o\mb X_o$. Suppose $\hat{\mb A}_\star$ is a global maximizer of the optimization problem
    \begin{equation*}
        \max_{\mb A} \hat{f}(\mb A,\mb Y)=\norm{\mb A\mb Y}{4}^4,\quad \st \quad \mb A \in \msf O(n;\bb R), 
    \end{equation*}
    then for any $\eps\in[0,1]$, there exists a signed permutation matrix $\mb P \in \text{SP}(n)$, such that
    \begin{equation}
        \frac{1}{n}\norm{\hat{\mb A}_\star^*-\mb D_o\mb P}{F}^2\leq C\eps,
     \end{equation}
     with probability at least
     \begin{equation}
        \label{eq:UnionConcentrationMainResult}
        \begin{split}
            1&-\exp\Bigg( -\frac{3p\eps^2}{c_1 \theta+8n(\ln p)^4\eps}+n^2\ln\Big(\frac{60np(\ln p)^4}{\eps}\Big) \Bigg)\\
            &-\exp\Bigg(-\frac{p\eps^2}{c_2\theta}+n^2\ln\Big(\frac{60np(\ln p)^4}{\eps}\Big)\Bigg)-2np\theta \exp\Bigg(-\frac{(\ln p)^2}{2}\Bigg),
        \end{split}
    \end{equation}
    for some constants $c_1>10^4,c_2>3360,C>\frac{4}{3\theta(1-\theta)}$. Moreover 
    \begin{equation}
        \label{eq:DecentUnionConcentrationMainResult}
        \begin{split}
             1&-\exp\Bigg( -\frac{3p\eps^2}{c_1\theta+8n(\ln p)^4\eps}+n^2\ln\Big(\frac{60np(\ln p)^4}{\eps}\Big) \Bigg)\\
            &-\exp\Bigg(-\frac{p\eps^2}{c_2\theta}+n^2\ln\Big(\frac{60np(\ln p)^4}{\eps}\Big)\Bigg)-2np\theta \exp\Bigg(-\frac{(\ln p)^2}{2}\Bigg)\geq 1-\frac{1}{p},
        \end{split}
    \end{equation}
    when $p=\Omega(\theta n^2\ln n/\eps^2)$.
    
\end{claim}

\begin{proof}
    \label{proof:MainResult}
    Suppose $\mb A_\star$ is the global maximizer of optimization program \eqref{eq:maxf}:
    \begin{equation*}
        \max_{\mb A} f(\mb A) = \bb E \norm{\mb {AY}}{4}^4 \quad\st\mb A \in \msf O(n;\bb R),
    \end{equation*}
    then by \eqref{eq:HatfUnionConcentrationBound} and \eqref{eq:DecentHatfUnionConcentrationBound} in Lemma \ref{lemma:HatfUnionConcentrationBound}, when $p=\Omega(\theta n^2\ln n/\eps^2)$, we know that with probability at least 
    \begin{equation}
        \begin{split}
            1&-\exp\Bigg( -\frac{3p\eps^2}{c_1 \theta+8n(\ln p)^4\eps}+n^2\ln\Big(\frac{60np(\ln p)^4}{\eps}\Big) \Bigg)\\
            &-\exp\Bigg(-\frac{p\eps^2}{c_2\theta}+n^2\ln\Big(\frac{60np(\ln p)^4}{\eps}\Big)\Bigg)-2np\theta \exp\Bigg(-\frac{(\ln p)^2}{2}\Bigg)\geq 1-\frac{1}{p},
        \end{split}
    \end{equation}
     we have 
    \begin{equation}
        \label{eq:UseConcentrationBound}
        \frac{1}{np}\abs{\hat{f}(\hat{\mb A}_\star,\mb Y)-f(\hat{\mb A}_\star)}\leq\eps \quad \text{ and } \quad \frac{1}{np}\abs{\hat{f}(\mb A_\star,\mb Y)-f(\mb A_\star)}\leq\eps.
    \end{equation}
    which implies 
    \begin{equation}
        \label{eq:hatfhatAIneqDerivation}
        \frac{1}{np}f(\hat{\mb A}_\star)\leq \frac{1}{np}f(\mb A_\star)<\frac{1}{np}\hat{f}(\mb A_\star,\mb Y)+\eps \leq \frac{1}{np}\hat{f}(\hat{\mb A}_\star,\mb Y)+\eps < \frac{1}{np}f(\hat{\mb A}_\star)+2\eps.
    \end{equation}
    In the above inequality, the first and the third $\leq$ is due to the global optimality of $f(\mb A_\star)$ and $\hat{f}(\hat{\mb A}_\star,\mb Y)$, the second and the last $<$ is due to \eqref{eq:UseConcentrationBound}. Simplify \eqref{eq:hatfhatAIneqDerivation}, yields
    \begin{equation}
        \label{eq:hatfhatARelationSimplified}
        \frac{1}{np}f(\hat{\mb A}_\star)\in \Bigg(\frac{1}{np}f(\mb A_\star)-2\eps, \frac{1}{np}f(\mb A_\star)\Bigg).
    \end{equation}
    From Lemma \ref{lemma:orthproperty}, we have:
    \begin{equation*}
        \frac{1}{3p\theta}f(\mb A) = (1-\theta)g(\mb {AD}_o)+\theta n,\quad \forall \mb A\in \msf O(n;\bb R),
    \end{equation*}
    which implies 
    \begin{equation}
        \label{eq:hatghatARelation}
        \begin{split}
            &\frac{3\theta(1-\theta)}{n}g(\hat{\mb A}_\star\mb D_o)\in \Bigg(\frac{3\theta(1-\theta)}{n}g(\mb A_\star\mb D_o)-2\eps, \frac{3\theta(1-\theta)}{n}g(\mb A_\star\mb D_o)\Bigg]\\
            \implies &\frac{1}{n}\norm{\hat{\mb A}_\star\mb D_o}{4}^4\in \Bigg(\frac{1}{n}\norm{\mb A_\star\mb D_o}{4}^4-\frac{2\eps}{3\theta(1-\theta)},\frac{1}{n}\norm{\mb A_\star\mb D_o}{4}^4 \Bigg].
        \end{split}
    \end{equation}
    Lemma \ref{lemma:orthproperty} tells us that $\mb {A}_\star\mb {D}_o\in \text{SP}(n)$, combining Lemma \ref{lemma:G(A)bound}, we know that $\norm{\mb A_\star\mb D_o}{4}^4=n$. Thus we can further simplify \eqref{eq:hatghatARelation} as
    \begin{equation*}
        \frac{1}{n}\norm{\hat{\mb A}_\star\mb D_o}{4}^4\in \Bigg( 1-\frac{2\eps}{3\theta(1-\theta)},1\Bigg].
    \end{equation*}
    Applying Lemma \ref{lemma:L4ExtremaBoundOrthGroup} (change $\eps$ in Lemma \ref{lemma:L4ExtremaBoundOrthGroup} into $2\eps/3\theta(1-\theta)$), we know that there exists $\mb P\in \text{SP}(n)$, such that
    \begin{equation}
        \frac{1}{n}\norm{\hat{\mb A}_\star\mb D_o - \mb P}{F}^2\leq \frac{4\eps}{3\theta(1-\theta)}.
    \end{equation}
    By the rotational invariant of Frobenius norm, we have
    \begin{equation}
    \label{eq:hatAProbabilisticBound}
        \frac{1}{n}\norm{\hat{\mb A}_\star^* - \mb D_o\mb P^*}{F}^2\leq \frac{4\eps}{3\theta(1-\theta)},
    \end{equation}
    which completes the proof.
\end{proof}

\section{Proofs of Section \ref{sec:Algorithm}}
\subsection{Proof of Lemma \ref{lemma:OrthProj}}
\begin{claim}[Projection onto Orthogonal Group]
    $\forall \mb A\in \bb R^{n\times n}$, the orthogonal matrix which has minimum Frobenius norm with $\mb A$ is the following
    \begin{equation}
       \mathcal{P}_{\msf O(n;\bb R)}(\mb A) =  \underset{\mb M \in\msf O(n;\bb R)}{\arg\min} \norm{\mb M-\mb A}{F}^2 = \mb {UV}^*,
    \end{equation}
    where $\mb {U\Sigma V}^* = \text{SVD}(\mb A)$.
\end{claim}
\begin{proof}
\label{proof:OrthProj}
Notice that
\begin{equation}
    \begin{split}
        \norm{\mb M - \mb A}{F}^2 =& tr\big((\mb M - \mb A)(\mb M - \mb A)^*\big)\\
        =&tr\big(\mb I - \mb {AM}^* - \mb {MA}^* +\mb {AA}^* \big)=n-2tr\big(\mb {MA}^*\big)+tr\big(\mb {AA}^*\big).
    \end{split}
\end{equation}
    Since $tr\big(\mb {AA}^*\big)$ is a constant, we know that
\begin{equation}
    \underset{\mb M \in\msf O(n;\bb R)}{\arg\min} \norm{\mb M-\mb A}{F}^2 = \underset{\mb M \in\msf O(n;\bb R)}{\arg\max} tr\big(\mb {AM}^* \big).
\end{equation}
Let $\mb {U\Sigma V}^*$ be the \text{SVD} of $\mb A$, then
\begin{equation}
    tr\big(\mb {AM}^* \big) = tr\big( \mb {U\Sigma V}^*\mb M\big) = tr\big( \mb \Sigma \mb V^*\mb M^*\mb U\big)\leq \sum_{i=1}^n\sigma_{i}(\mb A)\sigma_i(\mb V^*\mb M^*\mb U) = \sum_{i=1}^n\sigma_{i}(\mb A),
\end{equation}
where inequality is obtained through Von Neumann's trace inequality, and the equality holds if and only if $\mb {V}^*\mb M^*\mb U$ is diagonal matrix (in fact, identity matrix), which implies $\mb V^*\mb M^*\mb U = \mb I\implies \mb M = \mb {UV}^*$.
\end{proof}

\subsection{Proof of Proposition \ref{prop:L4MSPExp}}
\begin{claim}[Expectation of $\nabla_{\mb A}\hat{f}(\mb A,\mb Y)$ ]

Let $\mb X\in \bb R^{n\times p}$, $x_{i,j}\sim_{iid}\text{BG}(\theta)$, $\mb D_o\in \msf O(n;\bb R)$ is any orthogonal matrix, and $\mb Y = \mb D_o\mb X_o$. The expectation of $\nabla_{\mb A}\hat{f}(\mb A, \mb Y)$ satisfies this property:
\begin{equation}
    \bb E_{\mb X_o} \nabla_{\mb A}\hat{f}(\mb A, \mb Y)= 3p\theta(1-\theta)\nabla_{\mb A}g(\mb {AD}_o)+12p\theta^2\mb A.
\end{equation}
\end{claim}

\begin{proof}
    \label{proof:L4MSPExp}
    For simplicity, since $\mb D_o\in \msf O(n;\bb R)$, let $\mb W = \mb {AD}_o$, we know that $\mb W\in \msf O(n;\bb R)$. By the fact that $\mb W\in \msf O(n;\bb R)$, we have:
    \begin{equation}
        \label{eq:L4MSPExpDerivation1}
        \frac{1}{4}\nabla_{\mb A}\hat{f}(\mb A, \mb Y) = (\mb {AY})^{\circ3}\mb Y^* = (\mb {WX}_o)^{\circ3}\mb X_o^*\mb D_o^*,
    \end{equation}
    moreover,
    \begin{equation}
        \label{eq:L4MSPExpDerivation2}
        \begin{split}
            \{(\mb {AY})^{\circ3}\}_{i,j} =& \{(\mb {WX}_o)^{\circ3}\}_{i,j}= \Big(\sum_{k=1}^nw_{i,k}x_{k,j}\Big)^3\\
            =&\sum_{k=1}^n w_{i,k}^3x_{k,j}^3 + 3\Big(\sum_{1\leq k_1<k_2\leq n}w_{i,k_1}^2x_{k_1,j}^2w_{i,k_2}x_{k_2,j}+w_{i,k_1}x_{k_1,j}w_{i,k_2}^2x_{k_2,j}^2\Big)\\
            &+6\Big(\sum_{1\leq k_1<k_2<k_3\leq n}w_{i,k_1}x_{i,k_1}w_{i,k_2}x_{i,k_2}w_{i,k_3}x_{i,k_3}\Big).
        \end{split}
    \end{equation}
    And hence, we know:
    \begin{equation}
        \begin{split}
            &\{(\mb {WX}_o)^{\circ3}\mb {X}_o^*\}_{i,j^\prime}\\
            =& \sum_{j=1}^p\Bigg[x_{j^\prime,j}\sum_{k=1}^n w_{i,k}^3x_{k,j}^3 + 3x_{j^\prime,j}\Big(\sum_{1\leq k_1<k_2\leq n}w_{i,k_1}^2x_{k_1,j}^2w_{i,k_2}x_{k_2,j}+w_{i,k_1}x_{k_1,j}w_{i,k_2}^2x_{k_2,j}^2\Big)\\
            &+6x_{j^\prime,j}\Big(\sum_{1\leq k_1<k_2<k_3\leq n}w_{i,k_1}x_{k_1,j}w_{i,k_2}x_{k_2,j}w_{i,k_3}x_{k_3,j}\Big)\Bigg].
        \end{split}
    \end{equation}
    Thus, 
    \begin{equation}
        \begin{split}
            \bb E_{\mb X_o} \{(\mb {WX}_o)^{\circ3}\mb {X}_o^*\}_{i,j^\prime} &= 3p\theta w_{i,j^\prime}^3+3p\theta^2\sum_{\substack{1\leq j\leq n\\j\neq j^\prime}}w_{i,j}^2w_{i,j^\prime}\\
            &= 3p\theta w_{i,j^\prime}^3+3p\theta^2(1-w_{i,j^\prime}^2)w_{i,j^\prime}=3p\theta(1-\theta)w_{i,j^\prime}^3+3p\theta^2w_{i,j^\prime},
        \end{split}
    \end{equation}
    which implies 
    \begin{equation}
        \label{eq:GradHatfExpCalculation}
        \begin{split}
             \frac{1}{4p}\bb E_{\mb X_o}\nabla_{\mb A}\hat{f}(\mb A, \mb Y)&= \frac{1}{p}\bb E_{\mb X_o} (\mb {WX}_o)^{\circ3}\mb {X}_o^*\mb D_o^* = 3\theta(1-\theta)\mb W^{\circ3}\mb D_o^*+3\theta^2\mb W\mb D_o^*\\
             &=3\theta(1-\theta)(\mb {AD}_o)^{\circ3}\mb D_o^*+3\theta^2\mb A =\frac{3}{4}\theta(1-\theta)\nabla_{\mb A}g(\mb {AD}_o)+3\theta^2\mb A.
        \end{split}
    \end{equation}
\end{proof}

\subsection{Proof of Proposition \ref{prop:GradHatfUnionConcentrationBound}}

\begin{claim}[Union Tail Concentration Bound of $\frac{1}{np}\nabla\hat{f}(\cdot,\cdot)$]
    If $\mb X_o\in \bb R^{n\times p},x_{i,j}\sim_{iid}\text{BG}(\theta)$, for any $\mb A,\mb D_o \in \msf O(n;\bb R)$, and $\mb Y=\mb D_o\mb X_o$, the following inequality holds
    \begin{equation}
        \label{eq:GradHatfUnionConcentrationBound}
        \begin{split}
            &\bb P\Bigg(\sup_{\mb A\in \msf O(n;\bb R)}\frac{1}{4np}\norm{\nabla_{\mb A} \hat{f}(\mb A,\mb Y)-\bb E\big[\nabla_{\mb A} \hat{f}(\mb A,\mb Y)]}{F}\geq \delta\Bigg)\\
            \leq& 2n^2\exp\Bigg(-\frac{3p\delta^2}{c_1\theta+8n^{\frac{3}{2}}(\ln p)^4\delta}+n^2\ln\Big(\frac{48np(\ln p)^4}{\delta}\Big)\Bigg) + 2np\theta \exp\Bigg(-\frac{(\ln p)^2}{2}\Bigg),
        \end{split}
    \end{equation}
    for a constant $c_1>1.7\times 10^4$. Moreover 
    \begin{equation}
        \label{eq:GradHatfUnionConcentrationBoundDecent}
        \begin{split}
            2n^2\exp\Bigg(-\frac{3p\delta^2}{c_1\theta+8n^{\frac{3}{2}}(\ln p)^4\delta}+n^2\ln\Big(\frac{48np(\ln p)^4}{\delta}\Big)\Bigg) + 2np\theta \exp\Bigg(-\frac{(\ln p)^2}{2}\Bigg)\leq \frac{1}{p}
        \end{split}
    \end{equation}
    when $p=\Omega(\theta n^2\ln n/\delta^2)$.
\end{claim}

\begin{proof}
    \label{proof:GradHatfUnionConcentrationBound}
    By Proposition \ref{prop:L4MSPExp}, we have:
    \begin{equation}
        \bb E_{\mb X_o} \nabla_{\mb A}\hat{f}(\mb A, \mb Y)= 3p\theta(1-\theta)\nabla_{\mb A}g(\mb {AD}_o)+12p\theta^2\mb A,
    \end{equation}
    so 
    \begin{equation}
        \begin{split}
            & \bb P\Bigg(\sup_{\mb A\in \msf O(n;\bb R)}\frac{1}{4np}\norm{\nabla_{\mb A} \hat{f}(\mb A,\mb Y)-\bb E\big[\nabla_{\mb A} \hat{f}(\mb A,\mb Y)]}{F}\geq \delta\Bigg)\\
            = & \bb P\Bigg(\sup_{\mb W\in \msf O(n;\bb R)}\frac{1}{np}\norm{(\mb {WX}_o)^{\circ3}\mb X^*_o\mb D^*_o-\bb E\big[(\mb {WX}_o)^{\circ3}\mb X^*_o\mb D^*_o\big]}{F}\geq \delta\Bigg)\quad (\text{Assume } \mb W=\mb A\mb D_o)\\
            = & \bb P\Bigg(\sup_{\mb W\in \msf O(n;\bb R)}\frac{1}{np}\norm{(\mb {WX}_o)^{\circ3}\mb X^*_o-\bb E\big[(\mb {WX}_o)^{\circ3}\mb X^*_o\big]}{F}\geq \delta\Bigg)\\
            \leq& 2n^2\exp\Bigg(-\frac{3p\delta^2}{c_1\theta+8n^{\frac{3}{2}}(\ln p)^4\delta}+n^2\ln\Big(\frac{48np(\ln p)^4}{\delta}\Big)\Bigg)\\
            & + 2np\theta \exp\Bigg(-\frac{(\ln p)^2}{2}\Bigg) \quad\text{(By Lemma \ref{lemma:UnionConcentrationWXX})},
        \end{split}
    \end{equation}
    for a constant $c_1>1.7\times 10^4$, which completes the proof for \eqref{eq:GradHatfUnionConcentrationBound}. When $p=\Omega(\theta n^2 \ln n/\delta^2)$, we have 
    \begin{equation}
        \begin{split}
            &2n^2\exp\Bigg(-\frac{3p\delta^2}{c_1\theta+8n^{\frac{3}{2}}(\ln p)^4\delta}+n^2\ln\Big(\frac{48np(\ln p)^4}{\delta}\Big)\Bigg) + 2np\theta \exp\Bigg(-\frac{(\ln p)^2}{2}\Bigg)\\
            \leq& \frac{1}{2p}+\frac{1}{2p}=\frac{1}{p}, 
        \end{split}
    \end{equation}
    which completes the proof for \eqref{eq:GradHatfUnionConcentrationBoundDecent}.
\end{proof}

\section{Proofs of Section \ref{sec:Analysis}}

\subsection{Proof of Proposition \ref{prop:L4OrthCriticalPoints}}
\begin{claim}
    The critical points of $g(\mb W)$ on manifold $\msf O(n;\bb R)$ satisfies the following condition
    \begin{equation}
        (\mb W^{\circ3})^*\mb W =\mb W^* \mb W^{\circ3}.
    \end{equation}
\end{claim}

\begin{proof}
    \label{proof:L4OrthCriticalPoints}
    Notice that $\nabla_{\mb W}g(\mb W) = 4\mb W^{\circ3}$, and the critical points of $g(\mb W)$ on $\msf O(n;\bb R)$ satisfies 
    \begin{equation}
        \grad g(\mb W) = \mc P_{T_{\mb W}\msf O(n;\bb R)}(\nabla_{\mb A}g(\mb A)) = \frac{1}{2} (4\mb W^{\circ3} - \mb W (4\mb W^{\circ3})^*\mb W) = \mb 0,
    \end{equation}
    which yields
    \begin{equation}
        (\mb W^{\circ3})^*\mb W =\mb W^* \mb W^{\circ3}.
    \end{equation}
    Therefore, $\forall \mb W \in \msf O(n;\bb R)$, we can write critical points condition of $\ell^4$-norm over $\msf O(n;\bb R)$ as the following equations
\begin{equation}
    \begin{cases}
        (\mb W^{\circ3})^*\mb W = \mb W^*\mb W^{\circ3},\\
        \mb W^*\mb W = \mb I.
    \end{cases}
\end{equation}
\end{proof}

\subsection{Proof of Proposition \ref{prop:DiscreteCritPoints}}
\begin{claim}
    All global maximizers of $\ell^4$-norm over the orthogonal group are isolated critical points.
\end{claim}

\begin{proof}
    \label{proof:DiscreteCritPoints}
    As our objective function is invariant under signed permutation group, without loss of generality, we prove for the identity matrix $\mb I$. Suppose that the are not isolated. Then, there exists a $\mb W_0$ such that $(\mb W_0^{\circ3})^*\mb W_0 = \mb W_0^*\mb W_0^{\circ3}$ and $ \mb W_0^*\mb W_0 = \mb I$, and in every neighborhood of $\mb W_0$ there exists some $\mb W$ that is a critical point. This implies that there exists a path around $\mb W_0$ such that
		\begin{equation}
			\mb W(\cdot): (-\varepsilon, \varepsilon) \rightarrow \msf O(n, \bb R),\;  (\mb W(t)^{\circ3})^*\mb W(t) = \mb W(t)^*\mb W(t)^{\circ3}, \; \mb W(0) = \mb W_0.
		\end{equation}
		Then, we expand $\mb W(t)$ around $t=0$,
		\begin{equation}
			\mb W(t) = \mb W_0 + t\mb W_1+t^2 \mb W_2 + \dots
		\end{equation}
		Constraint $(\mb W(t)^{\circ3})^*\mb W(t) = \mb W(t)^*\mb W(t)^{\circ3}$ implies that
			\begin{equation}
			\label{eq:DiscreteCritPointsDerivation1}
			    \begin{split}
                    & (\mb W_0 + t\mb W_1)^* (\mb W_0 + t\mb W_1)^{\circ3} =[(\mb W_0 + t\mb W_1)^{\circ3}]^*(\mb W_0 + t\mb W_1)\\
                    \implies & 3\mb W_1 + \mb W_1^* = 3\mb W_1^* + \mb W_1\Rightarrow \mb W_1 = \mb W_1^*.
			    \end{split}
			\end{equation}
			Constraint $\mb W(t)^*\mb W(t)=\mb I$ implies that
			\begin{equation}
	    	\label{eq:DiscreteCritPointsDerivation2}
				\frac{d}{dt}\big[\mb W(t)^*\mb W(t)\big]\Big|_{t=0} = \frac{d}{dt}\mb I\Big|_{t=0} =\mb 0\implies \mb W_1+\mb W_1^*=\mb 0.
			\end{equation}
		The above computation is equivalent to plugging $\mb W(t)$ into the constraints, taking the derivative and evaluating at $\mb 0$. Combining \eqref{eq:DiscreteCritPointsDerivation1} and \eqref{eq:DiscreteCritPointsDerivation2}, we know $\mb W_1=\mb 0$. Since $\mb W(t)$ is any arbitrary path, this shows that the variety formed by all the critical points does not have a tangent space around $\mb W_0$. We conclude that the $\mb I$ is an isolated critical points. In fact, by calculating the Hessian at the maximizers, one can further show that all global maximizers, are nondegenerate critical points.  
\end{proof}

\subsection{Proof of Proposition \ref{prop:MSPFixedPoint}}

\begin{claim}[Fixed Point of the MSP Algorithm]
    Given $\mb W\in\msf {SO}(n;\bb R)$, $\mb W$ is a fix point of the MSP algorithm if and only if $\mb W$ is a critical point of the $\ell^4$-norm over $\msf {SO}(n;\bb R)$.
\end{claim}
\begin{proof}
    \label{proof:MSPFixedPoint}
    Let $\mb {U\Sigma V^}*=\mb W^{\circ3}$ denote the the \text{SVD} of $\mb W^{\circ3}$. 
    \begin{itemize}
        \item When $\mb W\in \msf O(n;\bb R)$ is a fixed point of MSP algorithm, we have:
    \begin{equation}
        \mb W = \mb {UV}^* \implies (\mb W^{\circ3})^*\mb W = \mb {V\Sigma U}^*\mb W = \mb V\mb \Sigma \mb V^*,
    \end{equation}
    which implies $(\mb W^{\circ3})^*\mb W$ is symmetric, hence we have:
    \begin{equation}
        (\mb W^{\circ3})^*\mb W = \mb W^*\mb W^{\circ3},
    \end{equation}
    and by Proposition \ref{prop:L4OrthCriticalPoints}, $\mb W$ is a critical point of $\ell^4$-norm over the orthogonal group.  
        \item When $\mb W$ is a critical point of $\ell^4$-norm over orthogonal group, we have $(\mb W^{\circ3})^*\mb W = \mb W^*\mb W^{\circ3}$. So the \text{SVD} of $(\mb W^{\circ3})^*\mb W$ should equal to the \text{SVD} of $\mb W^*\mb W^{\circ3}$. Also by the rotational invariant of \text{SVD}, we know:
        \begin{equation}
            \text{SVD}\big((\mb W^{\circ3})^*\mb W\big) = \mb V\mb \Sigma \mb U^*\mb W = \text{SVD}\big(\mb W^*\mb W^{\circ3}\big)=\mb W^*\mb U\mb \Sigma \mb V^*,
        \end{equation}
        and by the uniqueness polar decomposition, we have:
        \begin{equation}
            \mb V\mb U^*\mb W = \mb W^*\mb U \mb V^*\implies (\mb V\mb U^*\mb W)^2 = \mb I \implies \mb V\mb U^*\mb W = \mb I,
        \end{equation}
        where the last $\implies$ holds because we restricted $W\in \msf {SO}(n;\bb R)$, hence, we have $\mb W = \mb U\mb V^*$.
    \end{itemize}
\end{proof}

\subsection{Proof of Proposition \ref{prop:ConvergenceToSaddle}}
\begin{claim}[Convergence of PGA with Arbitrary Step Size]
Iterative PGA Algorithm \ref{algo:PGAL4MaxOrth} with any fixed step size $\alpha>0$ ($\alpha$ can be $+\infty$ and PGA is equivalent to MSP Algorithm \ref{algo:SPOrthD} when $\alpha = +\infty$) finds a saddle point of optimization problem \eqref{eq:L4MaxOrthClean} 
\begin{equation*}
    \max_{\mb A\in \msf O(n;\bb R)}\norm{\mb A}{4}^4.
\end{equation*}
\end{claim}
\begin{proof}
\label{proof:ConvergenceToSaddle}
Consider the following objective function $h(\cdot):\msf O(n;\bb R)\mapsto\bb R^+$:
\begin{equation}
    \label{eq:FunctionFDef}
    h(\mb A) =
    \begin{cases}
        \frac{\alpha}{4}\norm{\mb A}{4}^4 + \frac{1}{2}\norm{\mb A}{F}^2&\quad\text{when }\alpha < \infty\\
        \norm{\mb A}{4}^4&\quad\text{when }\alpha = +\infty
    \end{cases}.
\end{equation}
Note that $h(\mb A)$ is convex in both cases ($\alpha <+\infty$ or $\alpha = \infty$). Also note that the Stiefel manifold $\msf O(n;\bb R)$ is a compact manifold, so by Theorem 1 in \cite{journee2010generalized}, we know that the following iterative update
\begin{equation}
    \label{eq:ConvexCompactUpdate}
    \mb A_{k+1} \in \arg\max\{h(\mb A_k)+\innerprod{\partial h(\mb A_k)}{\mb W-\mb A_k} | \mb W\in \msf O(n; \bb R)\}
\end{equation}
will find a saddle point of $h(\mb A)$ with any initialization $\mb A_0\in \msf O(n;\bb R)$, where $\innerprod{\cdot}{\cdot}:\msf O(n;\bb R)\times\msf O(n;\bb R)\mapsto\bb R$ is defined as
\begin{equation}
    \label{eq:InnerProductDef}
    \innerprod{\mb W_1}{\mb W_2} = \trace(\mb W_1^*\mb W_2).
\end{equation}
Hence, by substituting \eqref{eq:InnerProductDef} into \eqref{eq:ConvexCompactUpdate}, yields
\begin{equation}
    \mb A_{k+1} = \begin{cases}
        \mc P_{\msf O(n;\bb R)}(\alpha\mb A_k^{\circ3}+\mb A_k)&\quad\text{when }\alpha < \infty\\
        \mc P_{\msf O(n;\bb R)}(\mb A_k^{\circ3})&\quad\text{when }\alpha = +\infty
    \end{cases},
\end{equation}
where $\mc P_{\msf O(n;\bb R)}(\cdot):\bb R^{n\times n}\mapsto \msf O(n;\bb R)$ is the projection onto $\msf O(n;\bb R)$ \citep{absil2012projection}:
\begin{equation}
    \mc P_{\msf O(n;\bb R)}(\mb A) = \mb U\mb V^*,\quad\st\quad \mb U\mb \Sigma \mb V^* = SVD(\mb A).
\end{equation}
Moreover, notice that $\norm{\mb A}{F}^2=n$ is a constant, so finding a critical point of $h(\mb A)$ is equivalent to finding a critical point of $\norm{\mb A}{4}^4$ over $\msf O(n;\bb R)$. Hence, we show that projected gradient ascent with any step size $\alpha>0$ (including $\alpha = +\infty$) finds a critical point of $\norm{\mb A}{4}^4$ over $\msf O(n;\bb R)$.\footnote{The proof can easily be generalized to any Stiefel manifolds.}
\end{proof}

\subsection{Proof of Theorem \ref{Thm:MSPLocalConvergence}}

\begin{claim}[Local Convergence of the MSP Algorithm]
Given an orthogonal matrix $\mb A\in \msf O(n;\bb R)$, let $\mb A^\prime$ denote the output of the MSP Algorithm \ref{algo:SPOrthD} after one iteration: $\mb A^\prime = \mb {UV}^*$, where $\mb {U\Sigma V}^*=\text{SVD}(\mb A^{\circ3})$. If $\norm{\mb A-\mb I}{F}^2=\eps$, for $\eps<0.579$, then we have $\norm{\mb A^\prime - \mb I}{F}^2 < \norm{\mb A - \mb I}{F}^2$ and $\norm{\mb A^\prime - \mb I}{F}^2 < O(\eps^3)$.
\end{claim}
\begin{proof}
    \label{proof:MSPLocalConvergence}
    Let $\mb A = \mb D + \mb N$, where $\mb D$ is the diagonal part of $\mb A$ and $\mb N$ is the off-diagonal part. Therefore, we have: 
    \begin{equation}
        \label{eq:NNormBound}
        \norm{\mb N^{\circ3}}{F} \leq \norm{\mb N}{F}^3\leq \norm{\mb A-\mb I}{F}^3=\eps^{3/2},
    \end{equation}
    where the first inequality is achieved through Cauchy-Schwarz inequality and the second inequality holds because $\mb N$ is the off-diagonal parts of $\mb A-\mb I$. We can view $\mb A^{\circ3}$ as a $\mb D^{\circ3}$ plus a small perturbation $\mb N^{\circ3}$ with norm as most $\eps^{3/2}$. By Lemma \ref{lemma:PolarPerturbation}, we have:
    \begin{equation}
        \label{eq:PerturbationBound}
        \norm{\mb {Q_A}-\mb {Q_D}}{F} \leq \frac{2\norm{\mb A^{\circ3}-\mb D^{\circ3}}{F}}{\sigma_n(\mb D^{\circ3})+\sigma_n(\mb A^{\circ3})}=\frac{2\norm{\mb N^{\circ3}}{F}}{\sigma_n(\mb D^{\circ3})+\sigma_n(\mb A^{\circ3})},
    \end{equation}
    where $\mb {U_A}\mb {\Sigma_A}\mb {V_A}^*=\text{SVD}(\mb A^{\circ3})$, $\mb {Q_A} = \mb {U_A}\mb {V_A}^*$ and $\mb {U_D}\mb {\Sigma_D}\mb {V_D}^*=\text{SVD}(\mb D^{\circ3})$, $\mb {Q_D} = \mb {U_DV_D}$. Notice that $\mb D^{\circ3}$ is a diagonal matrix, so $\mb {Q_D}=\mb I$, $\sigma_n(\mb D^3)=\min_{i}a_{i,i}^3$.
    Moreover
    \begin{equation}
        \begin{split}
            \eps = \norm{\mb A-\mb I}{F}^2 = \sum_{i=1}^n(a_{i,i}-1)^2 + \sum_{i\neq j}a_{i,j}^2 = \sum_{i,j}a_{i,j}^2 - 2\sum_{i=1}^na_{i,i} + n \iff  \sum_{i=1}^na_{i,i} = n - \frac{\eps}{2},
        \end{split}
    \end{equation}
    without loss of generality, we can assume $1\geq a_{1,1}\geq a_{2,2}\geq\dots\geq a_{n,n}>0$, so we have:
    \begin{equation}
        \begin{split}
            a_{n,n} = n-\frac{\eps}{2} - (a_{1,1}+a_{2,2}+\dots+a_{n-1,n-1}) \geq n-\frac{\eps}{2}-(n-1) = 1-\frac{\eps}{2},
        \end{split}
    \end{equation}
    hence we know that
    \begin{equation}
        \label{eq:LBsigmaD}
        \sigma_n(\mb D^{\circ3}) = \min_{i}a_{i,i}^3\geq \Big(1-\frac{\eps}{2}\Big)^3.
    \end{equation}
    Applying Lemma \ref{lemma:SingularValuePerturbation} by substituting
    \begin{equation}
        \mb G = \mb D^{\circ3}=\mb {BM},\quad  \delta \mb G = \mb A^{\circ3}-\mb D^{\circ3}=\mb N^{\circ3}= \delta \mb {BM},
    \end{equation}
    we know
    $\mb B = \mb I,\mb M = \mb D^{\circ3}, \delta \mb B=\mb N^{\circ3}(\mb D^{\circ3})^{-1}$, and thus: 
    \begin{equation}
        \frac{\abs{\sigma_n(\mb A^{\circ3})-\sigma_n(\mb D^{\circ3})}}{\sigma_n(\mb D^{\circ3})}\leq \frac{\norm{\delta \mb B}{2}}{\sigma_n(\mb B)},
    \end{equation}
    which implies
    \begin{equation}
        \begin{split}
            \abs{\sigma_n(\mb A^{\circ3})-\sigma_n(\mb D^{\circ3})}\leq & \norm{\mb N^{\circ3}(\mb D^{\circ3})^{-1}}{2} \sigma_n(\mb D^{\circ3})\leq\norm{\mb N^{\circ3}}{2}\norm{(\mb D^{\circ3})^{-1}}{2}\sigma_n(\mb D^{\circ3})\\
            \leq & \norm{\mb N^{\circ3}}{F}\Big(1-\frac{\eps}{2}\Big)^{-3}\sigma_n(\mb D^{\circ3})\leq \eps^{3/2}\Big(1-\frac{\eps}{2}\Big)^{-3}\sigma_n(\mb D^{\circ3}).
        \end{split}
    \end{equation}
    Therefore, using \eqref{eq:LBsigmaD}, we know that:
    \begin{equation}
        \label{eq:LBsigmaAsigmaD}
        \sigma_n(\mb A^{\circ3})+\sigma_n(\mb D^{\circ3})\geq \left[2-\eps^{3/2}\Big(1-\frac{\eps}{2}\Big)^{-3}\right]\sigma_n(\mb D^{\circ3})\geq 2\Big(1-\frac{\eps}{2}\Big)^3-\eps^{3/2}.
    \end{equation}
    Substituting \eqref{eq:NNormBound} and \eqref{eq:LBsigmaAsigmaD} into  \eqref{eq:PerturbationBound}, yields
    \begin{equation}
        \norm{\mb {Q_A}-\mb {Q_D}}{F} \leq  \frac{2\eps^{3/2}}{2\left(1-\frac{\eps}{2}\right)^3-\eps^{3/2}},
    \end{equation}
    hence, we have:
    \begin{equation}
        \norm{\mb A^\prime -\mb I}{F}^2=\norm{\mb {Q_A}-\mb {Q_D}}{F}^2\leq O(\eps^3). 
    \end{equation}
    Moreover, such operation is a contraction whenever:
    \begin{equation}
        \begin{split}
            &\frac{2\eps^{3/2}}{2\left(1-\frac{\eps}{2}\right)^3-\eps^{3/2}}< \norm{\mb A-\mb D}{F}=\eps^{1/2}\\
            \iff&2\eps - 2\left(1-\frac{\eps}{2}\right)^3-\eps^{3/2} < 0\\
            \impliedby & \eps<0.579,
        \end{split}
    \end{equation}
    which completes the proof.
\end{proof}

\subsection{Proof of Proposition \ref{prop:MSPGlobalConvn=2}}
\begin{claim}[Global Convergence of the MSP Algorithm on $\msf {SO}(2;\bb R)$]
When $n=2$, if we denote our $\mb A_t\in \msf {SO}(2,\bb R)$ as following form:
\begin{equation}
\mb A_t = \begin{pmatrix}
       \cos\theta_t  &  -\sin\theta_t \\
       \sin\theta_t  &  \cos\theta_t 
    \end{pmatrix},\quad \forall \theta_t \in \Big[-\frac{\pi}{2},\frac{\pi}{2}\Big],\\
\end{equation}
then: 1) $\mb A_{t+1}\in \msf {SO}(n;\bb R)$; 2) if we let 
\begin{equation}
\mb A_{t+1} = \begin{pmatrix}
       \cos\theta_{t+1}  &  -\sin\theta_{t+1} \\
       \sin\theta_{t+1}  &  \cos\theta_{t+1} 
    \end{pmatrix},\quad \forall \theta_{t+1} \in \Big[-\frac{\pi}{2},\frac{\pi}{2}\Big],\\
\end{equation}
$\theta_t$ and $\theta_{t+1}$ satisfies the following relation
\begin{equation}
\theta_{t+1} = \tan^{-1}\big(\tan^3\theta_{t}\big).
\end{equation}
\end{claim}

\begin{proof}
\label{proof:MSPGlobalConvn=2}
By the update of the MSP algorithm, we know that 
\begin{equation}
    \mb A_{t+1} = \underset{\mb M\in \msf O(n;\bb R)}{\arg\min} \norm{\mb M - \mb A_t^{\circ3}}{F}^2.
\end{equation}
$\forall \mb M\in O(n;\bb R)$, we can denote $\mb M$ as
\begin{equation}
    \mb M =
    \begin{cases}
    \begin{pmatrix}
        \cos\theta & -\sin\theta\\
        \sin\theta & \cos\theta
    \end{pmatrix} & \text{ if } \det(\mb M) = 1\quad\\ 
    \begin{pmatrix}
        \cos\theta & \sin\theta\\
        \sin\theta & -\cos\theta
    \end{pmatrix} & \text{ if } \det(\mb M) = -1,\quad
    \end{cases}, \theta\in (-\pi,\pi]
\end{equation}
Then if $\det(\mb M) = 1$, we have:
\begin{equation}
    \begin{split}
    \norm{\mb M-\mb A_t^{\circ3}}{F}^2 &= 2(\cos^3\theta_t-\cos\theta)^2+2(\sin^3\theta_t-\sin\theta)^2 \\
    &=2\cos^6\theta_t+2\sin^6\theta_t-4\cos\theta\cos^3\theta_t-4\sin\theta\sin^3\theta_t+2.
    \end{split}
\end{equation}
When $\det(\mb M) = -1$, we have:
\begin{equation}
    \begin{split}
    \norm{\mb M-\mb A_t^{\circ3}}{F}^2 &= (\cos^3\theta_t-\cos\theta)^2+(\cos^3\theta_t+\cos\theta)^2+(\sin^3\theta_t-\sin\theta)^2+(\sin^3\theta_t+\sin\theta)^2 \\
    &=2\cos^6\theta_t+2\sin^6\theta_t+2.
    \end{split}
\end{equation}
Notice that when $\det (\mb A) = -1$, $\norm{\mb M-\mb A_t^{\circ3}}{F}^2$ is a constant, so as long as we pick $\theta$ in the same quadrant with $\theta_t$, then
    \begin{equation}
        2\cos^6\theta_t+2\sin^6\theta_t-4\cos\theta\cos^3\theta_t-4\sin\theta\sin^3\theta_t + 2 \leq 2\cos^6\theta_t+2\sin^6\theta_t+2,
    \end{equation}
and therefore, we know $\mb A_{t+1}\in \msf {SO}(n;\bb R)$. So, we know that $\theta$ should satisfy the first order condition of critical point condition of the following optimization problem:
\begin{equation}
    \mb A_{t+1} = \underset{\mb M\in \msf {SO}(n;\bb R)}{\arg\min} \norm{\mb M - \mb A_t^{\circ3}}{F}^2,
\end{equation}
which implies 
\begin{equation}
    \label{ProjectionNormGrad}
    \begin{split}
        \nabla_\theta\big[2(\cos^3\theta_t-\cos\theta)^2+2(\sin^3\theta_t-\sin\theta)^2\big]=0\implies&\sin\theta\cos^3\theta_{t}-\cos\theta\sin^3\theta_t=0\\
    \end{split}
\end{equation}
which implies
\begin{equation}
    \begin{cases}
        \tan\theta = \tan^3\theta_t, \text{ when } \theta_t \neq \pm\pi/2  \\
        \cot\theta = \cot^3\theta_t, \text{ when } \theta_t \neq 0,\pi.
    \end{cases}
\end{equation}
Note that we distinguish the different cases $\theta_{t} = \pm\pi/2,0,\pi$ to avoid the division by zero in \eqref{ProjectionNormGrad}. One can also ignore this by taking the inverse on tangent, which yields 
\begin{equation}
    \theta_{t+1} = \tan^{-1}\big(\tan^3\theta_t\big),
\end{equation}
as stated.
\end{proof}

\section{Related Lemmas and Inequalities}

\subsection{Some Basic Inequalities}
\begin{lemma}[One-sided Bernstein's Inequality]
    \label{lemma:BernsteinInequality}
    Given $n$ random variables $x_1,x_2,\dots x_n$, if $\forall i\in [n],x_i\leq b$ almost surely, then 
    \begin{equation}
        \label{eq:BernsteinInequality}
        \bb P\Big(\sum_{i=1}^n\big(x_i-\bb E[x_i]\big)\geq nt \Big)\leq \exp\Bigg(-\frac{nt^2}{2(\frac{1}{n}\sum_{i=1}^n\bb E[x_i^2]+bt/3)} \Bigg).
    \end{equation}
\end{lemma}
\begin{proof}
See Proposition 2.14 in \cite{wainwright2019high}.
\end{proof}

\begin{lemma}[Some Useful Norm Matrix Norm Inequalities]
    \label{lemma:UsefulInequalities}
    Given two matrix $\mb A,\mb B$:
    \begin{enumerate}
        \item if $\mb A,\mb B \in \bb R^{n\times m}$, then $\norm{\mb A\circ \mb B}{F}\leq \norm{\mb A}{F}\norm{\mb B}{F}$
        \item if $\mb A\in \bb R^{n\times r},\mb B\in \bb R^{r\times m}$, then $\norm{\mb {AB}}{F}\leq \norm{\mb A}{2}\norm{\mb B}{F}$
        \item if $\mb A\in \bb R^{n\times r},\mb B\in \bb R^{r\times m}$, then $\norm{\mb {AB}}{F}\leq \norm{\mb A}{F}\norm{\mb B}{F}$
    \end{enumerate}
\end{lemma}
\begin{proof}
    \label{proof:UsefulInequalities}
    \begin{enumerate}
        \item $\norm{\mb A\circ \mb B}{F}^2 = \sum_{i,j}(a_{i,j}b_{i,j})^2\leq \Big(\sum_{i,j}a_{i,j}^2\Big)\Big(\sum_{i,j}b_{i,j}^2\Big)=\norm{\mb A}{F}^2\norm{\mb B}{F}^2$.
        \item $\norm{\mb {AB}}{F}^2=\sum_{j}\norm{\mb A\mb b_{j}}{F}^2\leq \norm{\mb A}{2}^2\sum_{j}\norm{\mb b_j}{2}^2=\norm{\mb A}{2}^2\norm{\mb B}{F}^2$.
        \item Let $\mb a_i,i\in[n]$ be the $i^\text{th}$ row vector of $\mb A$ and $\mb b_j,j\in [m]$ be the $j^\text{th}$ column vector of $\mb B$. So
        \begin{equation}
            \norm{\mb {AB}}{F}^2 = \sum_{i=1}^n\sum_{j=1}^m \abs{\mb a_i^*\mb b_j}^2\leq \Big(\sum_{i=1}^n\norm{\mb a_i}{2}^2\Big)\Big(\sum_{j=1}^m\norm{\mb b_i}{2}^2\Big)=\norm{\mb A}{F}^2\norm{\mb B}{F}^2.
        \end{equation}
    \end{enumerate}
\end{proof}

\subsection{Truncation of Bernoulli-Gaussian Matrix}

\begin{lemma}[Entry-wise Truncation of a Bernoulli Gaussian Matrix]
\label{lemma:BGMatrixTruncation}
    Let $\mb X\in \bb R^{n\times p}$, where $x_{i,j}\sim_{iid}\text{BG}(\theta)$ and let $\norm{\cdot}{\infty}$ denote the maximum element (in absolute value) of a matrix, then 
    \begin{equation}
        \bb P\Big( \norm{\mb X}{\infty}\geq t \Big)\leq 2np\theta \exp\Bigg(-\frac{t^2}{2}\Bigg).
    \end{equation}
\end{lemma}
\begin{proof}
    \label{proof:BGMatrixTruncation}
    A Bernoulli Gaussian variable $x_{i,j},\forall i\in [n],j\in [p]$ satisfies $x_{i,j}= b_{i,j} \cdot g_{i,j}$, where $b_{i,j}\sim_{iid}\text{Ber}(\theta)$, $g_{i,j}\sim_{iid}\mc N(0,1)$ and therefore
    \begin{equation}
        \bb P\big(\abs{x_{i,j}}\geq t \big)=\theta \cdot \bb P\big(\abs{g_{i,j}}\geq t \big)\leq2\theta\exp\Bigg(-\frac{t^2}{2}\Bigg).
    \end{equation}
    By union bound, we have:
    \begin{equation}
        \bb P\big( \norm{\mb X}{\infty}\geq t\big) \leq \sum_{i=1}^n\sum_{j=1}^p\bb P\big(\abs{x_{i,j}}\geq t\big)\leq 2np \theta\exp\Bigg(-\frac{t^2}{2}\Bigg).
    \end{equation}
\end{proof}

\subsection{$\eps-$covering of Stiefel Manifolds}

\begin{lemma}[$\eps-$Net Covering of Stiefel Manifolds]\footnote{A similar result can be found in Lemma 4.5 of \cite{recht2010guaranteed}.}
\label{lemma:epsCoveringStiefelManifold}
There is a covering $\eps-$net $\mc S_\eps$ for Stiefel manifold $\mc M = \{\mb W\in \bb R^{n\times r}|\mb W^*\mb W = \mb I\},(n\geq r)$ in operator norm
\begin{equation}
    \forall \mb W\in \mc M,\;\exists \mb W^\prime\in \mc S_\eps \quad \st \quad \norm{\mb W-\mb W^\prime}{2}\leq \eps,
\end{equation}
of size $|\mc S_\eps|\leq \big(\frac{6}{\eps}\big)^{nr}$. 
\end{lemma}
\begin{proof}
    \label{proof:epsCoveringStiefelManifold}
    Let $\mc S^\prime_{\eps/2}=\{\mb A_1,\mb A_2,\dots,\mb A_{|\mc S_{\eps/2}|}\}$ be an $\eps/2-$nets for the unit operator norm ball of $n\times r$ matrix $\{\mb A\in \bb R^{n\times r}|\norm{\mb A}{2}\leq 1\}$, $\eps-$net covering theorem shows that such construction of $\mc S^\prime_{\eps/2}$ exists and $|\mc S^\prime_{\eps/2}|\leq \big(\frac{6}{\eps}\big)^{nr}$. Next, let $\mc S^\prime$ be the subset of $\mc S^\prime_{\eps/2}$, which consists elements of $\mc S^\prime_{\eps/2}$ whose distance are within $\eps/2$ with $\mc M$:
    \begin{equation}
        \mc S^\prime = \{\mb A \in \mc S^\prime_{\eps/2}|\exists\mb W\in \mc M,\st \norm{\mb W-\mb A}{2}\leq \eps/2\}.
    \end{equation}
    Then, $\forall \mb A\in \mc S^\prime$, let $\hat{\mb W}(\mb A)$ be the nearest element of $\mb A$ in $\mc M$:
    \begin{equation}
        \hat{\mb W}(\mb A) = \underset{\mb W\in\mc M}{\arg\min} \norm{\mb W-\mb A}{2},
    \end{equation}
    and let $\mc S_\eps$ be the set of the nearest element of each $\mb A$ in $\mc S^\prime$:  $\mc S_\eps = \{\hat{\mb W}(\mb A)|\mb A\in \mc S^\prime\}$. Since $\mc S^\prime_{\eps/2}$ is an $\eps/2$-nets for $\mc M$. So $\forall\mb W\in \mc M$, there exists $\mb A_l\in \mc S^\prime_{\eps/2}$, such that
    \begin{equation}
        \norm{\mb W-\mb A_l}{2} \leq \frac{\eps}{2}, 
    \end{equation}
    so $\mb A_l\in \mc S^\prime$, and therefore there exists $\hat{\mb W}(\mb A_l)\in \mc S_\eps$, such that  
    \begin{equation}
        \norm{\hat{\mb W}(\mb A_l)-\mb A_l}{2} \leq \norm{\mb W-\mb A_l}{2} \leq \frac{\eps}{2}, 
    \end{equation}
    hence by triangle inequality
    \begin{equation}
        \norm{\mb W-\hat{\mb W}(\mb A_l)}{2} \leq \norm{\hat{\mb W}(\mb A_l)-\mb A_l}{2} + \norm{\mb W-\mb A_l}{2}\leq \eps.
    \end{equation}
    Thus, $\mc S_\eps$ is an $\eps-$net for $\mc M$ and $|\mc S_\eps|= |\mc S^\prime|\leq |\mc S^\prime_{\eps/2}|\leq\big(\frac{6}{\eps}\big)^{nr}$.
\end{proof}

\subsection{Convergence to Maxima of $\ell^4$-Norm over Unit Sphere}

\begin{lemma}[Single Vector Convergence over Unit Sphere]
    \label{lemma:L4ExtremaBoundSphere}
    Suppose $\mb q$ is a vector on the unit sphere: $\mb q\in \bb S^{n-1}$. $\forall \eps\in[0,1]$, if $\norm{\mb q}{4}^4 \geq 1-\eps$, then $\exists i\in[n]$, such that
    \begin{equation}
        \label{eq:qPMe1Bound}
            \norm{\mb q - \mb e_i}{2}^2 \leq 2\eps\quad \text{when } q_i>0,\quad
            \norm{\mb q + \mb e_i}{2}^2 \leq 2\eps\quad \text{when } q_i<0,
    \end{equation}
     where $\{\mb e_1,\mb e_2,\dots,\mb e_n\}$ is the canonical basis of $\bb R^n$. 
\end{lemma}



\begin{proof}
    \label{proof:L4ExtremaBoundSphere}
    Let $\mb q = [q_1,q_2,\dots,q_n]^*$, and without loss of generality, we can assume
    \begin{equation}
        \label{eq:qiSortedMagnitude}
        1\geq q_1^2 \geq q_2^2 \geq \dots \geq q_n^2 \geq 0.
    \end{equation}
    Also, from the assumption 
    \begin{equation}
        q_1^4+q_2^4+\dots + q_n^4\geq1-\eps,
    \end{equation}
    and along with \eqref{eq:qiSortedMagnitude} and $\mb q\in \bb S^{n-1}$, we have:
    \begin{equation}
        q_1^4+q_2^2(q_2^2+q_3^2+\dots + q_n^2)=q_1^4+q_2^2(1-q_1^2)\geq1-\eps,
    \end{equation}
    which implies:
    \begin{equation}
        \label{eq:q1PMe1EPS}
        \begin{split}
            q_1^4+q_1^2(1-q_1^2)\geq q_1^4+q_2^2(1-q_1^2)\geq 1-\eps\implies& q_1^2\geq 1-\eps.
        \end{split}
    \end{equation}
    Hence, we know
    \begin{equation}
        \label{eq:sumqiLeqEPS}
        q_2^2+\dots+q_n^2 \leq \eps.
    \end{equation}
    Moreover, \eqref{eq:q1PMe1EPS} also implies
    \begin{equation}
        \eps\geq(1-q_1)(1+q_1)\implies
        \begin{cases}
            1-q_1\leq \eps/(1+q_1)\leq \eps \quad \text{when } q_1>0, \\
            1+q_1\leq \eps/(1-q_1)\leq \eps \quad \text{when } q_1<0.
        \end{cases}
    \end{equation}
    When $q_1>0$, combine \eqref{eq:q1PMe1EPS} and \eqref{eq:sumqiLeqEPS}, we have 
    \begin{equation}
        \norm{\mb q-\mb e_1}{2}^2 = (q_1-1)^2 + q_2^2+\dots+q_n^2\leq\eps^2+\eps \leq 2\eps.
    \end{equation}
    And similar result for $\norm{\mb q+\mb e_1}{2}^2\leq 2\eps$ can be obtained through the same reasoning.
\end{proof}

\subsection{Multiple Vectors Convergence to Maxima of $\ell^4$-Norm over Unit Sphere}
\begin{lemma}[Multiple Vectors Convergence over Unit Sphere]
\label{lemma:L4MultipleVectorExtremaBoundSphere}
    Suppose $\mb q_1,\mb q_2\dots,\mb q_k$ are $k$ vectors on the unit sphere: $\mb q_i\in \bb S^{n-1},\forall i\in[k]$. $\forall \eps\in[0,1]$, if
    \begin{equation}
        \label{eq:MultipleSphericalVectorsCondition}
        \frac{1}{k}\sum_{i=1}^k\norm{\mb q_i}{4}^4 \geq 1-\eps,
    \end{equation}
    then $\exists j_1,j_2,\dots j_k\in[n]$, such that 
    \begin{equation}
        \frac{1}{k}\sum_{i=1}^k\norm{\mb q_i - s_i\mb e_{j_i}}{2}^2 \leq 2\eps,
    \end{equation}
    where $\mb e_{j_i}$ are one vector in canonical basis and $s_i\in \{1,-1\}$ indicates the sign of $\mb e_{j_i}$, $\forall i\in [k]$.
\end{lemma}
\begin{proof}
    \label{proof:L4MultipleVectorExtremaBoundSphere}
    Note that we can reformulate the condition \eqref{eq:MultipleSphericalVectorsCondition} as
    \begin{equation}
        \label{eq:MultipleSphericalVectorsConditionNew}
        \sum_{i=1}^k\norm{\mb q_i}{4}^4 \geq k-k\eps.
    \end{equation}
    Since $\forall i\in [k]$, $0\leq \norm{\mb q_i}{4}^4\leq \norm{\mb q_i}{2}^4=1$, we can assume
    \begin{equation}
        \norm{\mb q_i}{4}^4 = 1-\alpha_i\eps, \quad \forall i\in [k],
    \end{equation}
    where $\alpha_i$ satisfies $\alpha_i\eps\in[0,1]$, for all $i\in [k]$ and 
    \begin{equation}
        \label{eq:SumAlphaCondition}
        \sum_{i=1}^k\alpha_i\leq k.        
    \end{equation}
    By Lemma \ref{lemma:L4ExtremaBoundSphere}, we know there exists $s_i\in \{1,-1\}$ and $j_i$, such that  
    \begin{equation}
        \norm{\mb q_i-s_i\mb e_{j_i}}{2}^2\leq 2\alpha_i\eps,\quad \forall i\in [k].
    \end{equation}
    Along with \eqref{eq:SumAlphaCondition}, we have:
    \begin{equation}
        \begin{split}
            \frac{1}{k}\sum_{i=1}^k\norm{\mb q_i-s_i\mb e_{j_i}}{2}^2\leq \frac{2}{k
            }\eps\sum_{i=1}^k\alpha_i\leq2\eps,
        \end{split}
    \end{equation}
    which completes the proof.
\end{proof}

\subsection{Related Lipschitz Constants}

\begin{lemma}[Lipschitz Constant of $\frac{1}{np}\hat{f}(\cdot,\cdot)$ over $\msf O(n;\bb R)$]
    \label{lemma:HatfLipschitzBound}
    If $\mb X\in \bb R^{n\times p}$, $x_{i,j}\sim_{iid}\text{BG}(\theta)$, and let $\bar{\mb X}$ be the truncation of $\mb X$ by bound $B$ 
    \begin{equation}
        \bar{x}_{i,j}=
        \begin{cases}
            x_{i,j}&\textrm{if}\quad \abs{x_{i,j}}\leq B\\
            0&\textrm{else}
        \end{cases},
    \end{equation}
    then $\forall\mb W_1,\mb W_2 \in \msf O(n;\bb R)$, we have
    \begin{equation}
        \frac{1}{np}\abs{\norm{\mb W_1\bar{\mb X}}{4}^4-\norm{\mb W_2\bar{\mb X}}{4}^4}\leq L_1\norm{\mb W_1-\mb W_2}{2},
    \end{equation}
    for a constant $L_1\leq 4npB^4$. 
\end{lemma}
\begin{proof}
    \label{proof:HatfLipschitzBound}
    Notice that
    \begin{equation}
        \begin{split}
            &\abs{\norm{\mb W_1\bar{\mb X}}{4}^4 - \norm{\mb W_2\bar{\mb X}}{4}^4} =\abs{\sum_{i,j}\Big[\big(\mb W_1\bar{\mb X}\big)^4_{i,j}-\big(\mb W_2\bar{\mb X}\big)^4_{i,j}\Big]}\\
            =&\abs{\sum_{i,j}\Bigg\{\Big[\big(\mb W_1\bar{\mb X}\big)^2_{i,j}-\big(\mb W_2\bar{\mb X}\big)^2_{i,j}\Big]\Big[\big(\mb W_1\bar{\mb X}\big)^2_{i,j}+\big(\mb W_2\bar{\mb X}\big)^2_{i,j}\Big]\Bigg\}}\\
            \leq &\Bigg\{\sum_{i,j}\Big[\big(\mb W_1\bar{\mb X}\big)^2_{i,j}-\big(\mb W_2\bar{\mb X}\big)^2_{i,j}\Big]^2\Bigg\}^{1/2}\Bigg\{\sum_{i,j}\Big[\big(\mb W_1\bar{\mb X}\big)^2_{i,j}+\big(\mb W_2\bar{\mb X}\big)^2_{i,j}\Big]^2\Bigg\}^{1/2}\\
            =&\underbrace{\norm{\big(\mb W_1\bar{\mb X}\big)^{\circ2}-\big(\mb W_2\bar{\mb X}\big)^{\circ2}}{F}}_{\Gamma_1}\underbrace{\norm{\big(\mb W_1\bar{\mb X}\big)^{\circ2}+\big(\mb W_2\bar{\mb X}\big)^{\circ2}}{F}}_{\Gamma_2},
        \end{split}
    \end{equation}
    the only inequality is obtained through Cauchy$-$Schwarz inequality. For $\Gamma_1$, we have
    \begin{equation}
        \begin{split}
            \Gamma_1 =& \norm{\big(\mb W_1\bar{\mb X}\big)^{\circ2}-\big(\mb W_2\bar{\mb X}\big)^{\circ2}}{F} = \norm{\big(\mb W_1\bar{\mb X}-\mb W_2\bar{\mb X}\big)\circ\big(\mb W_1\bar{\mb X}+\mb W_2\bar{\mb X}\big)}{F}\\
            \leq & \norm{\big(\mb W_1-\mb W_2\big)\bar{\mb X}}{F}\norm{\big(\mb W_1+\mb W_2\big)\bar{\mb X}}{F}\quad \text{(By inequality 1 in Lemma \ref{lemma:UsefulInequalities})}\\
            \leq &\norm{\mb W_1-\mb W_2}{2}\norm{\mb W_1+\mb W_2}{2}\norm{\bar{\mb X}}{F}^2 \quad \text{(By inequality 2 in Lemma \ref{lemma:UsefulInequalities})}\\
            \leq & \norm{\mb W_1-\mb W_2}{2}\big(\norm{\mb W_1}{2}+\norm{\mb W_2}{2}\big)\norm{\bar{\mb X}}{F}^2\\
            = & 2 \norm{\mb W_1-\mb W_2}{2}\norm{\bar{\mb X}}{F}^2 \quad \text{($\mb W_1,\mb W_2$ are orthogonal, $\norm{\mb W_1}{2}=\norm{\mb W_2}{2}=1$)}\\
            \leq & 2npB^2\norm{\mb W_1-\mb W_2}{2} \quad \text{($\abs{\bar{x}_{i,j}}\leq B$)}. 
        \end{split}
    \end{equation}
    For $\Gamma_2$, we have
    \begin{equation}
        \begin{split}
            \Gamma_2 = &\norm{\big(\mb W_1\bar{\mb X}\big)^{\circ2}+\big(\mb W_2\bar{\mb X}\big)^{\circ2}}{F}\\
            \leq & \norm{\big(\mb W_1\bar{\mb X}\big)^{\circ2}}{F} + \norm{\big(\mb W_1\bar{\mb X}\big)^{\circ2}}{F}\\
            \leq & \norm{\mb W_1\bar{\mb X}}{F}^2+\norm{\mb W_2\bar{\mb X}}{F}^2 \quad \text{(By inequality 1 in Lemma \ref{lemma:UsefulInequalities})}\\
            \leq & \norm{\mb W_1}{2}^2\norm{\bar{\mb X}}{F}^2 + \norm{\mb W_2}{2}^2\norm{\bar{\mb X}}{F}^2\\
            = & 2\norm{\bar{\mb X}}{F}^2\quad \text{($\mb W_1,\mb W_2$ are orthogonal matrices, $\norm{\mb W_1}{2}=\norm{\mb W_2}{2}=1$)}\\
            \leq & 2npB^2 \quad \text{($\abs{\bar{x}_{i,j}}\leq B$)}.
        \end{split}
    \end{equation}
    Thus, we know that
    \begin{equation}
        \frac{1}{np}\abs{\norm{\mb W_1\bar{\mb X}}{4}^4-\norm{\mb W_2\bar{\mb X}}{4}^4}\leq \frac{1}{np}\Gamma_1\Gamma_2\leq 4npB^4\norm{\mb W_1-\mb W_2}{2}.
    \end{equation}
\end{proof}

\begin{lemma}[Lipschitz Constant of $\frac{1}{np}f(\cdot)$ over $\msf O(n;\bb R)$]
    \label{lemma:fLipschitzBound}
    If $\mb X\in \bb R^{n\times p},x_{i,j}\sim_{iid}\text{BG}(\theta)$, then $\forall\mb W_1,\mb W_2 \in \msf O(n;\bb R)$, we have
    \begin{equation}
        \frac{1}{np}\abs{\bb E\norm{\mb W_1\mb X}{4}^4-\bb E\norm{\mb W_2\mb X}{4}^4}\leq L_2\norm{\mb W_1-\mb W_2}{2},
    \end{equation}
    with a constant $L_2\leq 12n\theta(1-\theta)$.
\end{lemma}
\begin{proof}
    \label{proof:fLipschitzBound}
    According to Lemma \ref{lemma:orthproperty}, we have
    \begin{equation}
        \bb E\norm{\mb {WX}}{4}^4 = 3p\theta(1-\theta)\norm{\mb W}{4}^4+3\theta^2np,\quad \forall \mb W\in \msf O(n;\bb R),
    \end{equation}
    so
    \begin{equation}
        \begin{split}
            \abs{\bb E\norm{\mb W_1\mb X}{4}^4 - \bb E\norm{\mb W_2\mb X}{4}^4} =  3p\theta(1-\theta)\abs{\norm{\mb W_1}{4}^4-\norm{\mb W_2}{4}^4}.
        \end{split}
    \end{equation}
    Notice that
    \begin{equation}
        \begin{split}
            &\abs{\norm{\mb W_1}{4}^4-\norm{\mb W_2}{4}^4}=\abs{\sum_{i,j}\Big[\big(\mb W_1\big)_{i,j}^4-\big(\mb W_2\big)_{i,j}^4\Big]}\\
            =&\abs{\sum_{i,j}\Bigg\{\Big[\big(\mb W_1\big)_{i,j}^2-\big(\mb W_2\big)_{i,j}^2\Big]\Big[\big(\mb W_1\big)_{i,j}^2+\big(\mb W_2\big)_{i,j}^2\Big]\Bigg\}}\\
            =&\Bigg\{\sum_{i,j}\Big[\big(\mb W_1\big)_{i,j}^2-\big(\mb W_2\big)_{i,j}^2\Big]^2\Bigg\}^{1/2}\Bigg\{\sum_{i,j}\Big[\big(\mb W_1\big)_{i,j}^2+\big(\mb W_2\big)_{i,j}^2\Big]^2\Bigg\}^{1/2} \quad \text{(Cauchy$-$Schwarz)}\\
            =&\norm{\mb W_1^{\circ2}-\mb W_2^{\circ2}}{F}\norm{\mb W_1^{\circ2}+\mb W_2^{\circ2}}{F}=\norm{\big(\mb W_1-\mb W_2\big)\circ\big(\mb W_1+\mb W_2\big)}{F}\norm{\mb W_1^{\circ2}+\mb W_2^{\circ2}}{F}\\
            \leq & \norm{\mb W_1-\mb W_2}{F}\norm{\mb W_1+\mb W_2}{F}\big(\norm{\mb W_1^{\circ2}}{F}+\norm{\mb W_2^{\circ2}}{F}\big)\quad \text{(By Lemma \ref{lemma:UsefulInequalities})}\\
            \leq & n\norm{\mb W_1-\mb W_2}{2}\norm{\mb W_1+\mb W_2}{2}\big(\norm{\mb W_1}{F}^2+\norm{\mb W_2}{F}^2\big)\quad \text{(By Lemma \ref{lemma:UsefulInequalities}, $\norm{\mb W}{F}\leq\sqrt{n} \norm{\mb W}{2}$)}\\
            \leq & 4n^2\norm{\mb W_1-\mb W_2}{2} \quad \text{($\norm{\mb W_1+\mb W_2}{2}\leq \norm{\mb W_1}{2}+\norm{\mb W_2}{2}$ and $\norm{\mb W_1}{2}=\norm{\mb W_2}{2}=1$)}.
        \end{split}
    \end{equation}
    Hence, 
    \begin{equation}
        \frac{1}{np}\abs{\bb E\norm{\mb W_1\mb X}{4}^4-\bb E\norm{\mb W_2\mb X}{4}^4}\leq 12n\theta(1-\theta)\norm{\mb W_1-\mb W_2}{2},
    \end{equation}
    which completes the proof.
\end{proof}

\begin{lemma}[Lipschitz Constant of $\nabla\hat{f}(\cdot,\cdot)$]
\label{lemma:GradHatfLipschitzBound}

    If $\mb X\in \bb R^{n\times p},x_{i,j}\sim_{iid}\text{BG}(\theta)$, and let $\bar{\mb X}$ be the truncation of $\mb X$ by bound $B$
    \begin{equation}
        \bar{x}_{i,j}=
        \begin{cases}
            x_{i,j}&\textrm{if}\quad \abs{x_{i,j}}\leq B\\
            0&\textrm{else}
        \end{cases},
    \end{equation}
    then $\forall\mb W_1,\mb W_2 \in \msf O(n;\bb R)$, we have
    \begin{equation}
        \frac{1}{np}\norm{(\mb W_1\bar{\mb X})^{\circ3}\bar{\mb X}^*-(\mb W_2\bar{\mb X})^{\circ3}\bar{\mb X}^*}{F}\leq L_1\norm{\mb W_1-\mb W_2}{2},
    \end{equation}
    with a constant $L_1\leq 3npB^4$.  
\end{lemma}

\begin{proof}
    \label{proof:GradHatfLipschitzBound}
    Notice that 
    \begin{equation}
        \begin{split}
        &\frac{1}{np}\norm{(\mb W_1\bar{\mb X})^{\circ3}\bar{\mb X}^*-(\mb W_2\bar{\mb X})^{\circ3}\bar{\mb X}^*}{F}=\frac{1}{np}\norm{\Big[(\mb W_1\bar{\mb X})^{\circ3}-(\mb W_2\bar{\mb X})^{\circ3}\Big]\bar{\mb X}^*}{F}\\
        \leq & \frac{1}{np}\norm{(\mb W_1\bar{\mb X}-\mb W_2\bar{\mb X})\circ\Big[ (\mb W_1\bar{\mb X})^{\circ2}+(\mb W_2\bar{\mb X})^{\circ2}+(\mb W_1\bar{\mb X})\circ(\mb W_2\bar{\mb X})\Big]}{F}\norm{\bar{\mb X}^*}{F}\\
        \leq &\frac{1}{np}\underbrace{\norm{\mb W_1\bar{\mb X}-\mb W_2\bar{\mb X}}{F}}_{\Gamma_1}\underbrace{\norm{ (\mb W_1\bar{\mb X})^{\circ2}+(\mb W_2\bar{\mb X})^{\circ2}+(\mb W_1\bar{\mb X})\circ(\mb W_2\bar{\mb X})}{F}}_{\Gamma_2}\norm{\bar{\mb X}^*}{F},
        \end{split}
    \end{equation}
    where the $\leq$ is achieved by inequality 1 in Lemma \ref{lemma:UsefulInequalities}. For $\Gamma_1$, using inequality 2 in Lemma \ref{lemma:UsefulInequalities}, we have
    \begin{equation}
        \begin{split}
            \Gamma_1 = \norm{(\mb W_1-\mb W_2)\bar{\mb X}}{F}\leq \norm{\mb W_1-\mb W_2}{2}\norm{\bar{\mb X}}{F}.
        \end{split}
    \end{equation}
    For $\Gamma_2$, we have
    \begin{equation}
        \begin{split}
            \Gamma_2 & = \norm{ (\mb W_1\bar{\mb X})^{\circ2}+(\mb W_2\bar{\mb X})^{\circ2}+(\mb W_1\bar{\mb X})\circ(\mb W_2\bar{\mb X})}{F}\\
            & \leq \Big(\norm{(\mb W_1\bar{\mb X})^{\circ2}}{F}+\norm{(\mb W_2\bar{\mb X})^{\circ2}}{F}+\norm{(\mb W_1\bar{\mb X})\circ(\mb W_2\bar{\mb X})}{F}\Big)\\
            &\leq \Big(\norm{\mb W_1\bar{\mb X}}{F}^2+\norm{\mb W_2\bar{\mb X}}{F}^2+\norm{\mb W_1\bar{\mb X}}{F}\norm{\mb W_2\bar{\mb X}}{F}\Big)\quad \text{(By inequality 1 in Lemma \ref{lemma:UsefulInequalities})}\\
            & = 3\norm{\bar{\mb X}}{F}^2\quad \text{($\norm{\cdot}{F}$ is rotation invariant)}.
        \end{split}
    \end{equation}
    Hence, we have
    \begin{equation}
        \begin{split}
            &\frac{1}{np}\norm{(\mb W_1\bar{\mb X})^{\circ3}\bar{\mb X}^*-(\mb W_2\bar{\mb X})^{\circ3}\bar{\mb X}^*}{F}\leq \frac{1}{np}\Gamma_1\Gamma_2\norm{\bar{\mb X}}{F}\\
            \leq &\frac{3}{np}\norm{\bar{\mb X}}{F}^4\norm{\mb W_1-\mb W_2}{2}=3npB^4\norm{\mb W_1-\mb W_2}{2},
        \end{split}
    \end{equation}
    which completes the proof.
\end{proof}

\begin{lemma}[Lipschitz Constant of $\bb E\frac{1}{p}\nabla\hat{f}(\cdot,\cdot)$]
\label{lemma:ExpGradHatfLipschitzBound}
If $\mb X\in \bb R^{n\times p},x_{i,j}\sim_{iid}\text{BG}(\theta)$, then $\forall\mb W_1$, $\mb W_2 \in \msf O(n;\bb R)$, we have
    \begin{equation}
        \frac{1}{np}\norm{\bb E\big[(\mb W_1\bar{\mb X}^*)^{\circ3}\bar{\mb X}\big]-\bb E\big[(\mb W_2\bar{\mb X})^{\circ3}\bar{\mb X}^*\big]}{F}\leq L_2\norm{\mb W_1-\mb W_2}{2},
    \end{equation}
    with a constant $L_2\leq 9\theta(1-\theta)+\frac{3\theta^2}{\sqrt{n}}$.
\end{lemma}
\begin{proof}
\label{proof:ExpGradHatfLipschitzBound}
    By \eqref{eq:GradHatfExpCalculation} in proof \ref{proof:L4MSPExp} of Proposition \ref{prop:L4MSPExp}, we have
    \begin{equation}
        \bb E\big[(\mb W\bar{\mb X})^{\circ3}\bar{\mb X}\big] = 3p\theta(1-\theta)\mb W^{\circ3}+3p\theta^2\mb W,\quad \forall \mb W\in \msf O(n;\bb R).
    \end{equation}
    Hence, we have:
    \begin{equation}
        \begin{split}
            &\frac{1}{np}\norm{\bb E\big[(\mb W_1\bar{\mb X}^*)^{\circ3}\bar{\mb X}\big]-\bb E\big[(\mb W_2\bar{\mb X})^{\circ3}\bar{\mb X}^*\big]}{F}\\
            =&\frac{1}{n}\norm{3\theta(1-\theta)(\mb W_1^{\circ3}-\mb W_2^{\circ3})+3\theta^2(\mb W_1-\mb W_2))}{F}\\
            \leq&\frac{3\theta(1-\theta)}{n}\norm{\mb W_1^{\circ3}-\mb W_2^{\circ3}}{F}+\frac{3\theta^2}{n}\norm{\mb W_1-\mb W_2}{F}\\
            \leq& \frac{3\theta(1-\theta)}{n}\underbrace{\norm{\mb W_1^{\circ3}-\mb W_2^{\circ3}}{F}}_{\Gamma}+\frac{3\theta^2}{\sqrt{n}}\norm{\mb W_1-\mb W_2}{2}.
        \end{split}
    \end{equation}
    Note that 
    \begin{equation}
        \begin{split}
            \Gamma = & \norm{\mb W_1^{\circ3}-\mb W_2^{\circ3}}{F}= \norm{(\mb W_1-\mb W_2)(\mb W_1^{\circ2}+\mb W_2^{\circ2}+\mb W_1\circ\mb W_2)}{F}\\
            \leq & \norm{\mb W_1-\mb W_2}{2}\norm{\mb W_1^{\circ2}+\mb W_2^{\circ2}+\mb W_1\circ\mb W_2}{F} \quad \text{(By inequality 2 in Lemma \ref{lemma:UsefulInequalities})}\\
            \leq & (\norm{\mb W_1}{F}^2+\norm{\mb W_2}{F}^2+\norm{\mb W_1}{F}\norm{\mb W_2}{F})\norm{\mb W_1-\mb W_2}{2}= 3n\norm{\mb W_1-\mb W_2}{2},
        \end{split}
    \end{equation}
    therefore, we have:
    \begin{equation}
        \begin{split}
            \frac{1}{np}\norm{\bb E\big[(\mb W_1\bar{\mb X}^*)^{\circ3}\bar{\mb X}\big]-\bb E\big[(\mb W_2\bar{\mb X})^{\circ3}\bar{\mb X}^*\big]}{F}\leq \Bigg(9\theta(1-\theta)+\frac{3\theta^2}{\sqrt{n}}\Bigg)\norm{\mb W_1-\mb W_2}{2}, 
        \end{split}
    \end{equation}
    which completes the proof.
\end{proof}

\subsection{High Order Moment Bound of Bernoulli Gaussian Random Variables}

\begin{lemma}[Second Moment of $\norm{\cdot}{4}^4$]
    \label{lemma:HatfEightMoment}
    Assume that $\mb W\in\msf O(n;\bb R)$,$\mb x\in\bb R^{n}$, $ x_{i}\sim_{iid}\text{BG}(\theta)$, $\forall i\in [n]$, then the second order moment of $\norm{\mb {Wx}}{4}^4$ satisfies
\begin{equation}
    \label{eq:ExpWXPowerEight}
    \bb E \norm{\mb {Wx}}{4}^8 \leq Cn^2\theta,
\end{equation}
for a constant $C>105$.
\end{lemma}

\begin{proof}
    \label{proof:HatfEightMoment}
Assume $\mb v$ is a Gaussian vector and the support for Bernoulli-Gaussian vector $\mb x$ is $\mc S$, that is, $\forall i\in [n]$,
\begin{equation}
x_i = 
    \begin{cases}
        v,\; v\sim \mc N(0,1) & \quad \text{if } i \in \mc S,\\
        0 & \quad \text{otherwise}.
    \end{cases}
\end{equation}
Let $\mb P_{\mc S}:\bb R^n\mapsto \bb R^n $ be the projection onto set $\mc S$, that is, $\forall \mb q \in \bb R^n$
\begin{equation}
\big(\mb P_{\mc S} \mb q\big)_i = 
    \begin{cases}
        q_i\; & \quad \text{if } i \in \mc S,\\
        0 & \quad \text{otherwise}.
    \end{cases}
\end{equation}
Let $\mb w_i$ denote the $i^\text{th}$ row vector of $\mb W$, so
\begin{equation}
    \label{eq:WXEightOrderDerivation1}
    \begin{split}
        \bb E \norm{\mb {Wx}}{4}^8 & = \bb E \Big[\big( \norm{\mb W\mb x}{4}^4\big)^2 \Big]= \bb E \Bigg[\sum_{i=1}^n\sum_{j=1}^n\innerprod{\mb w_i}{\mb x}^4\innerprod{\mb w_j}{\mb x}^4\Bigg]\\
        & = \bb E_{\mc S}\sum_{i=1}^n\sum_{j=1}^n\bb E\Big[\innerprod{\mb P_{\mc S}\mb w_i}{\mb v}^4\innerprod{\mb P_{\mc S}\mb w_j}{\mb v}^4\Big] \\
        & \leq \bb E_{\mc S}\sum_{i=1}^n\sum_{j=1}^n\Big[\bb E\innerprod{\mb P_{\mc S}\mb w_i}{\mb v}^8\bb E\innerprod{\mb P_{\mc S}\mb w_j}{\mb v}^8\Big]^{\frac{1}{2}}\quad \text{(Cauchy$-$Schwarz)}.
    \end{split}
\end{equation}
Since $\mb v\sim \mc N (\mb 0,\mb I)$, so 
\begin{equation}
    \innerprod{\mb P_{\mc S}\mb w_i}{\mb v} = \sum_{k=1}^n \big(\mb P_{\mc S}\mb w_i\big)_kv_k \sim \mc N(0,\norm{\mb P_{\mc S}\mb w_i}{2}^2).
\end{equation}
Therefore, 
\begin{equation}
    \bb E \innerprod{\mb P_{\mc S}\mb w_i}{\mb v}^8 = 105\norm{\mb P_{\mc S}\mb w_i}{2}^8.
\end{equation}
Hence, combine \eqref{eq:WXEightOrderDerivation1}, we have 
\begin{equation}
    \label{eq:WXEightOrderDerivation2}
    \begin{split}
        \bb E \norm{\mb {Wx}}{4}^8 & \leq \bb E_{\mc S}\sum_{i=1}^n\sum_{j=1}^n\Big[\bb E\innerprod{\mb P_{\mc S}\mb w_i}{\mb v}^8\bb E\innerprod{\mb P_{\mc S}\mb w_j}{\mb v}^8\Big]^{\frac{1}{2}}\\
        & = 105 \bb E_{\mc S} \sum_{i=1}^n\sum_{j=1}^n \norm{\mb P_{\mc S}\mb w_i}{2}^4\norm{\mb P_{\mc S}\mb w_j}{2}^4\\
        & = 105 \sum_{i=1}^n\sum_{j=1}^n\sum_{k_1,k_2,k_3,k_4} \bb E_{\mc S} \Big[w_{i,k_1}^2\indicator{k_1\in \mc S}w_{i,k_2}^2\indicator{k_2\in \mc S}w_{j,k_3}^2\indicator{k_3\in \mc S}w_{j,k_4}^2\indicator{k_4\in \mc S}\Big].
    \end{split}
\end{equation}
Now we discuss these four different cases separately:
\begin{itemize}
    \item With probability $c_1\theta^4$ $(c_1\leq 1)$, all $k_1,k_2,k_3,k_4\in \mc S$, in this case, we have:
    \begin{equation}
        \label{eq:WXEightOrderDerivation3}
        \begin{split}
            &\sum_{i=1}^n\sum_{j=1}^n\sum_{k_1,k_2,k_3,k_4} \bb E_{\mc S} \Big[w_{i,k_1}^2\indicator{k_1\in \mc S}w_{i,k_2}^2\indicator{k_2\in \mc S}w_{j,k_3}^2\indicator{k_3\in \mc S}w_{j,k_4}^2\indicator{k_4\in \mc S}\Big]\\
            =&\sum_{i=1}^n\sum_{j=1}^n\sum_{k_1,k_2,k_3,k_4} \bb E_{\mc S} \Big[w_{i,k_1}^2w_{i,k_2}^2w_{j,k_3}^2w_{j,k_4}^2\Big]=\sum_{i=1}^n\sum_{j=1}^n\sum_{k_1,k_2,k_3,k_4}  w_{i,k_1}^2w_{i,k_2}^2w_{j,k_3}^2w_{j,k_4}^2\\
            =& \sum_{i=1}^n\sum_{j=1}^n\sum_{k_1,k_2,k_3}w_{i,k_1}^2w_{i,k_2}^2w_{j,k_3}^2=\sum_{i=1}^n\sum_{j=1}^n\sum_{k_1,k_2}w_{i,k_1}^2w_{i,k_2}^2 =n^2.
        \end{split}
    \end{equation}
    \item With probability $c_2\theta^3$ $(c_2\leq 1)$, only three among $k_1,k_2,k_3,k_4\in \mc S$, in this case, we have:
    \begin{equation}
        \label{eq:WXEightOrderDerivation4}
        \begin{split}
            &\sum_{i=1}^n\sum_{j=1}^n\sum_{k_1,k_2,k_3,k_4} \bb E_{\mc S} \Big[w_{i,k_1}^2\indicator{k_1\in \mc S}w_{i,k_2}^2\indicator{k_2\in \mc S}w_{j,k_3}^2\indicator{k_3\in \mc S}w_{j,k_4}^2\indicator{k_4\in \mc S}\Big]\\
            =& \sum_{i=1}^n\sum_{j=1}^n\sum_{k_1,k_2,k_3} \Big[w_{i,k_1}^2w_{i,k_2}^2w_{j,k_3}^4\Big]+\sum_{i=1}^n\sum_{j=1}^n\sum_{k_1,k_2,k_3}\Big[w_{i,k_1}^2w_{i,k_2}^2w_{j,k_2}^2w_{j,k_3}^2\Big]=n\norm{\mb W}{4}^4 + 1.
        \end{split}
    \end{equation}
    \item With probability $c_3\theta^2$ $(c_3\leq 1)$, only two among $k_1,k_2,k_3,k_4\in \mc S$, in this case, we have:
    \begin{equation}
        \label{eq:WXEightOrderDerivation5}
        \begin{split}
            &\sum_{i=1}^n\sum_{j=1}^n\sum_{k_1,k_2,k_3,k_4} \bb E_{\mc S} \Big[w_{i,k_1}^2\indicator{k_1\in \mc S}w_{i,k_2}^2\indicator{k_2\in \mc S}w_{j,k_3}^2\indicator{k_3\in \mc S}w_{j,k_4}^2\indicator{k_4\in \mc S}\Big]\\
            =& \sum_{i=1}^n\sum_{j=1}^n\sum_{k_1,k_2}w_{i,k_1}^4w_{i,k_2}^4+\sum_{i=1}^n\sum_{j=1}^n\sum_{k_1,k_2}w_{i,k_1}^4w_{j,k_1}^2w_{j,k_2}^2+\sum_{i=1}^n\sum_{j=1}^n\sum_{k_1,k_2}w_{i,k_1}^2w_{i,k_2}^2w_{j,k_1}^2w_{j,k_2}^2\\
            \leq & 3 \sum_{i=1}^n\sum_{j=1}^n\sum_{k_1,k_2}w_{i,k_1}^4w_{i,k_2}^4 = 3\norm{\mb W}{4}^8.
        \end{split}
    \end{equation}
    The only inequality above is achieved by Rearrangement inequality.
    \item With probability $c_4\theta$ $(c_4\leq 1)$, only one among $k_1,k_2,k_3,k_4\in \mc S$, in this case, we have:
    \begin{equation}
        \label{eq:WXEightOrderDerivation6}
        \begin{split}
            &\sum_{i=1}^n\sum_{j=1}^n\sum_{k_1,k_2,k_3,k_4} \bb E_{\mc S} \Big[w_{i,k_1}^2\indicator{k_1\in \mc S}w_{i,k_2}^2\indicator{k_2\in \mc S}w_{j,k_3}^2\indicator{k_3\in \mc S}w_{j,k_4}^2\indicator{k_4\in \mc S}\Big]\\
            = & \sum_{i=1}^n\sum_{j=1}^n\sum_{k_1} w_{i,k_1}^4w_{j,k_1}^4 \leq \sum_{i=1}^n\sum_{j=1}^n\sum_{k_1} w_{i,k_1}^8=n\norm{\mb W}{8}^8.
        \end{split}
    \end{equation}
    The only inequality above is achieved by Rearrangement inequality.
\end{itemize}
Substitute \eqref{eq:WXEightOrderDerivation3}, \eqref{eq:WXEightOrderDerivation4}, \eqref{eq:WXEightOrderDerivation5}, and \eqref{eq:WXEightOrderDerivation6} into \eqref{eq:WXEightOrderDerivation2}, yields
\begin{equation}
    \label{eq:WXEightOrderDerivation7}
    \begin{split}
        \bb E \norm{\mb {Wx}}{4}^8 & \leq 105 \sum_{i=1}^n\sum_{j=1}^n\sum_{k_1,k_2,k_3,k_4} \bb E_{\mc S} \Big[w_{i,k_1}^2\indicator{k_1\in \mc S}w_{i,k_2}^2\indicator{k_2\in \mc S}w_{j,k_3}^2\indicator{k_3\in \mc S}w_{j,k_4}^2\indicator{k_4\in \mc S}\Big] \\
        & \leq C(\theta^4n^2+\theta^3n\norm{\mb W}{4}^4+\theta^3+3\theta^2\norm{\mb W}{4}^8+\theta n\norm{\mb W}{8}^8) \leq Cn^2\theta,
    \end{split}
\end{equation}
for a constant $C>105$, which completes the proof.
\end{proof}

\begin{lemma}[Second Moment of $\nabla_{\mb A}\hat{f}(\mb A,\mb Y)$]
\label{lemma:GradHatfSecondMoment}
Assume that $\mb W\in\msf O(n;\bb R)$,$\mb X\in\bb R^{n\times p}$, $ x_{i,j}\sim_{iid}\text{BG}(\theta)$, $\forall i\in [n],j\in [p]$, then each element of $\{(\mb {WX}_o)^{\circ3}\mb {X}_o^*\}_{i,j^\prime}$ satisfies
\begin{enumerate}
    \item $\{(\mb {WX}_o)^{\circ3}\mb {X}_o^*\}_{i,j^\prime}$ can be represented as sum of $p$ i.i.d random variables $z_j$
    \begin{equation}
        \{(\mb {WX}_o)^{\circ3}\mb {X}_o^*\}_{i,j^\prime} = \sum_{j=1}^p z_j,\quad \forall i,j^\prime \in [n]. 
    \end{equation} 
    \item The second moment of $z_j$ is bounded by $C\theta$
    \begin{equation}
        \bb E z_j^2\leq C,\quad \forall j \in [p].
    \end{equation}
\end{enumerate}
for a constant $C>177$.
\end{lemma}

\begin{proof}
    \label{proof:GradHatfSecondMoment}
    Assume $\mb v$ is a Gaussian vector and the support for Bernoulli-Gaussian vector $\mb x$ is $\mc S$, that is, $\forall i\in [n]$,
\begin{equation}
x_i = 
    \begin{cases}
        v,\; v\sim \mc N(0,1) & \quad \text{if } i \in \mc S,\\
        0 & \quad \text{otherwise}.
    \end{cases}
\end{equation}
Let $\mb P_{\mc S}:\bb R^n\mapsto \bb R^n $ be the projection onto set $\mc S$, that is, $\forall \mb q \in \bb R^n$
\begin{equation}
\big(\mb P_{\mc S} \mb q\big)_i = 
    \begin{cases}
        q_i\; & \quad \text{if } i \in \mc S,\\
        0 & \quad \text{otherwise}.
    \end{cases}
\end{equation}
By \eqref{eq:L4MSPExpDerivation1} and \eqref{eq:L4MSPExpDerivation2} in proof \ref{proof:L4MSPExp} of Proposition \ref{prop:L4MSPExp}, we know that 
    \begin{equation}
        \frac{1}{4}\nabla_{\mb A}\hat{f}(\mb A, \mb Y) = (\mb {AY})^{\circ3}\mb Y^* = (\mb {WX}_o)^{\circ3}\mb X_o^*\mb D_o^*,
    \end{equation}
    and
    \begin{equation}
        \begin{split}
            \{(\mb {WX}_o)^{\circ3}\mb {X}_o^*\}_{i,j^\prime} = 
            \sum_{j=1}^p\Big[x_{j^\prime,j}\Big(\sum_{k=1}^nw_{i,k}x_{k,j}\Big)^3\Big]=\sum_{j=1}^p\Big[x_{j^\prime,j}\innerprod{\mb w_i}{\mb x_j}^3\Big],
        \end{split}
    \end{equation}
    where $\mb w_i$ is the $i^{\text{th}}$ row vector of $\mb W$. Notice that $x_{i,j}$ are independent with each other, if we define $p$ random variables $z_1,z_2,\dots,z_p$ as the following
\begin{equation}
    z_j = x_{j^\prime,j}\innerprod{\mb w_i}{\mb x_j},\forall j\in [p],
\end{equation}
then we can view $\{(\mb {WX})^{\circ3}\mb {X}^*\}_{i,j^\prime}$ as the mean of $p$ i.i.d. random variables $z_j$. Notice that 
\begin{equation}
    \label{eq:GradHatfSecondMomentDerivation1}
    \bb E(z_j^2) = \bb E \Big[x_{j^\prime,j}^2\Big(\sum_{k=1}^nw_{i,k}x_{k,j}\Big)^6\Big] \leq \Big(\bb Ex_{j^\prime,j}^4\Big)^{\frac{1}{2}}\Big(\bb E\innerprod{\mb w_i}{\mb x_j}^{12} \Big)^{\frac{1}{2}}= \sqrt{3}\theta^{\frac{1}{2}} \Big(\bb E_{\mc S}\bb E_{\mb v}\innerprod{\mb P_{\mc S}\mb w_i}{\mb v}^{12} \Big)^{\frac{1}{2}},
\end{equation}
and
\begin{equation}
    \innerprod{\mb P_{\mc S}\mb w_i}{\mb v} = \sum_{k=1}^n \big(\mb P_{\mc S}\mb w_i\big)_kv_k \sim \mc N(0,\norm{\mb P_{\mc S}\mb w_i}{2}^2). 
\end{equation}
So we have 
\begin{equation}
    \label{eq:GradHatfSecondMomentDerivation2}
    \Big(\bb E_{\mc S}\bb E_{\mb v}\innerprod{\mb P_{\mc S}\mb w_i}{\mb v}^{12} \Big)^{\frac{1}{2}}=\Big(11!!\bb E_{\mc S}\norm{\mb P_{\mc S}\mb w_i}{2}^{12}\Big)^{\frac{1}{2}}=\sqrt{11!!}\Big(\bb E_{\mc S}\norm{\mb P_{\mc S}\mb w_i}{2}^{12}\Big)^{\frac{1}{2}}.
\end{equation}
Follow the same pipe line \eqref{eq:WXEightOrderDerivation3}, \eqref{eq:WXEightOrderDerivation4}, \eqref{eq:WXEightOrderDerivation5}, and \eqref{eq:WXEightOrderDerivation6} in Lemma \ref{lemma:HatfEightMoment}, one can show that 
\begin{equation}
    \label{eq:GradHatfSecondMomentDerivation3}
    \begin{split}
        \bb E_{\mc S}\norm{\mb P_{\mc S}\mb w_i}{2}^{12}& = \sum_{k_1,k_2,\dots,k_6} w_{i,k_1}^2\indicator{k_1\in \mc S}w_{i,k_2}^2\indicator{k_2\in \mc S} w_{i,k_3}^2\indicator{k_3\in \mc S}w_{i,k_4}^2\indicator{k_4\in \mc S}w_{i,k_5}^2\indicator{k_5\in \mc S} w_{i,k_6}^2\indicator{k_6\in \mc S}\\
        &\leq C^\prime (\theta^6+\theta^5+\theta^4+\theta^3+\theta^2+\theta)\leq C^{\prime\prime} \theta,
    \end{split}
\end{equation}
for some constant $C^\prime, C^{\prime\prime}>1$. Therefore, combine \eqref{eq:GradHatfSecondMomentDerivation1}, \eqref{eq:GradHatfSecondMomentDerivation2}, and \eqref{eq:GradHatfSecondMomentDerivation3}, we have
\begin{equation}
    \label{eq:ElementwiseGrad2Moment}
    \begin{split}
         \bb E(z_j^2)\leq \sqrt{3}\theta^{\frac{1}{2}}\Big(\bb E_{\mc S}\bb E_{\mb v}\innerprod{\mb P_{\mc S}\mb w_i}{\mb v}^{12} \Big)^{\frac{1}{2}}\leq \sqrt{3\times 11!!\theta^2}\leq C\theta,
    \end{split}
\end{equation}
for a constant $C>177$, which completes the proof.
\end{proof}

\subsection{Union Tail Concentration Bound}

\begin{lemma}[Union Tail Concentration Bound of $(\mb {WX})^{\circ3}\mb X^*$]
    \label{lemma:UnionConcentrationWXX}
    If $\mb X\in \bb R^{n\times p},x_{i,j}\sim_{iid}\text{BG}(\theta)$, the following inequality holds
    \begin{equation}
        \begin{split}
            &\bb P\Bigg(\sup_{\mb W\in \msf O(n;\bb R)}\frac{1}{np}\norm{(\mb {WX})^{\circ3}\mb X^*-\bb E\big[(\mb {WX})^{\circ3}\mb X^*\big]}{F}\geq \delta\Bigg)\\
            \leq& 2n^2\exp\Bigg(-\frac{3p\delta^2}{c_1\theta+8n^{\frac{3}{2}}(\ln p)^4\delta}+n^2\ln\Big(\frac{48np(\ln p)^4}{\delta}\Big)\Bigg) + 2np\theta \exp\Bigg(-\frac{(\ln p)^2}{2}\Bigg),
        \end{split}
    \end{equation}
    for a constant $c_1>1.7\times 10^4$.
\end{lemma}

\begin{proof}
    \label{proof:UnionConcentrationWXX}
    Let $\bar{\mb X}\in \bb R^{n\times p}$ denote the truncated $\mb X$ by bound $B$
    \begin{equation}
        \bar{x}_{i,j}=
        \begin{cases}
            x_{i,j}&\textrm{if}\quad \abs{x_{i,j}}\leq B,\\
            0&\textrm{else}.
        \end{cases}
    \end{equation}
    Note that $\bar{\mb X} = \mb X$ holds whenever $\norm{\mb X}{\infty}\leq B$, and by Lemma \ref{lemma:BGMatrixTruncation}, we know that with probability $\norm{\mb X}{\infty}\leq B$ happens with probability at least $1-2np\theta\exp(-B^2/2)$. So we know that $\bar{\mb X} \neq \mb X$ holds with probability at most $2np\theta\exp(-B^2/2)$, and thus
    \begin{equation}
        \label{eq:GradUnionBoundWithAndWithoutTruncation}
        \begin{split}
            &\bb P\Bigg(\sup_{\mb W\in \msf O(n;\bb R)}\frac{1}{np}\norm{(\mb {WX})^{\circ3}\mb X^*-\bb E\big[(\mb {WX})^{\circ3}\mb X^*\big]}{F}>\delta\Bigg)\\
            \leq&\bb P\Bigg(\sup_{\mb W\in \msf O(n;\bb R)}\frac{1}{np}\norm{(\mb {WX})^{\circ3}\mb X^*-\bb E\big[(\mb {WX})^{\circ3}\mb X^*\big]}{F}>\delta,\mb X =\bar{\mb X}\Bigg)+\bb P\Big(\mb X \neq \bar{\mb X}\Big)\\
            \leq & \bb P\Bigg(\sup_{\mb W\in \msf O(n;\bb R)}\frac{1}{np}\norm{(\mb {W}\bar{\mb X})^{\circ3}\bar{\mb X}^*-\bb E\big[(\mb {WX})^{\circ3}\mb X^*\big]}{F}>\delta\Bigg) + 2np\theta e^{-\frac{B^2}{2}}
        \end{split}
    \end{equation}
    \textbf{$\eps-$net Covering.} For any positive $\eps$ satisfy
    \begin{equation}
        \label{eq:GradEpsCondition}
        \eps\leq \frac{\delta}{8npB^4},
    \end{equation}
     Lemma \ref{lemma:epsCoveringStiefelManifold} shows there exists an $\eps-$nets 
    \begin{equation}
        \mc S_\eps = \{\mb W_1,\mb W_2,\dots,\mb W_{|\mc S_\eps|}\},
    \end{equation} 
    which covers $\msf O(n;\bb R)$
    \begin{equation}
        \label{eq:GradEpsBallCoveringOrthogonalGroup}
        \msf O(n;\bb R)\subset \bigcup_{l=1}^{|\mc S_\eps|} \bb B(\mb W_l,\eps),
    \end{equation}
    in operator norm $\norm{\cdot}{2}$. Moreover, we have 
    \begin{equation}
        \label{eq:GradEpsBallMaxNumber}
        |\mc S_\eps|\leq \Big(\frac{6}{\eps}\Big)^{n^2}.    
    \end{equation}
    So $\forall \mb W\in \msf O(n;\bb R)$, there exists $l\in [|\mc S_\eps|]$, such that $\norm{\mb W-\mb W_l}{2}\leq \eps$. Thus, we have:
    \begin{equation}
        \label{eq:GradProbabilityTruncationEpsNet}
        \begin{split}
            & \bb P\Bigg(\frac{1}{np}\norm{(\mb W\bar{\mb X})^{\circ3}\bar{\mb X}^*-\bb E[(\mb W\mb X)^{\circ3}\mb X^*]}{F}\leq \delta \Bigg)\\
            \leq & \bb P\Bigg( \frac{1}{np}\Big( \norm{(\mb W\bar{\mb X})^{\circ3}\bar{\mb X}^*-(\mb W_l\bar{\mb X})^{\circ3}\bar{\mb X}^*}{F} + \norm{(\mb W_l\bar{\mb X})^{\circ3}\bar{\mb X}^*-\bb E[(\mb W_l\mb X)^{\circ3}\mb X^*]}{F} \\
            & + \norm{\bb E[(\mb W_l\mb X)^{\circ3}\mb X^*]-\bb E[(\mb W\mb X)^{\circ3}\mb X^*]}{F}\Big)\geq \delta \Bigg)\\
            \leq & \bb P \Bigg(\frac{1}{np}\norm{(\mb W_l\bar{\mb X})^{\circ3}\bar{\mb X}^*-\bb E[(\mb W_l\mb X)^{\circ3}\mb X^*]}{F}\\
            &+ \Big(3npB^4+9\theta(1-\theta)+\frac{3\theta^2}{\sqrt{n}}\Big)\norm{\mb W-\mb W_l}{2}\geq \delta \Bigg) \quad \text{(By Lemma \ref{lemma:GradHatfLipschitzBound}, \ref{lemma:ExpGradHatfLipschitzBound})}\\
            \leq &\bb P\Bigg(\frac{1}{np}\norm{(\mb W_l\bar{\mb X})^{\circ3}\bar{\mb X}^*-\bb E[(\mb W_l\mb X)^{\circ3}\mb X^*]}{F}+4npB^4\eps\geq \delta \Bigg)\quad \text{($p,B$ are large, $\norm{\mb W-\mb W_l}{2}\leq \eps$)}\\
            \leq &\bb P \Bigg(\frac{1}{np}\norm{(\mb W_l\bar{\mb X})^{\circ3}\bar{\mb X}^*-\bb E[(\mb W_l\mb X)^{\circ3}\mb X^*]}{F}+\frac{\delta}{2}\geq \delta \Bigg)\quad \text{(We assume $\eps\leq \frac{\delta}{8npB^4}$)}\\
            = &P\Bigg(\frac{1}{np}\norm{(\mb W_l\bar{\mb X})^{\circ3}\bar{\mb X}-\bb E[(\mb W_l\mb X)^{\circ3}\mb X^*]}{F}\geq\frac{\delta}{2}\Bigg). 
        \end{split}
    \end{equation}
    \textbf{Analysis.} For random variable $\bar{\mb X}$, we have
    \begin{equation}
        \label{eq:BoundGradMatrixTruncationEXP}
        \begin{split}
            &\norm{\bb E\big[(\mb W_l\mb X)^{\circ3}\mb X^*\big]-\bb E\big[(\mb W_l\bar{\mb X})^{\circ3}\bar{\mb X}^*\big]}{F}=\norm{\bb E\big[(\mb W_l\mb X)^{\circ3}\mb X^*\big]-\bb E\big[(\mb W_l\mb X)^{\circ3}\mb X^*\cdot\indicator{\{\norm{\mb X}{\infty}\leq B\}}\big]}{F}\\
            =&\norm{\bb E\big[(\mb W_l\mb X)^{\circ3}\mb X^*\cdot\indicator{\{\norm{\mb X}{\infty}>B\}}\big]}{F}=\sqrt{\sum_{\substack{1\leq i\leq n\\1\leq j^\prime\leq n}}\Bigg(\bb E\Big\{[(\mb W_l\mb X)^{\circ3}\mb X^*]_{i,j^\prime}\cdot\indicator{\{\norm{\mb X}{\infty}>B\}}\Big\}\Bigg)^2}\\
            \leq & \sqrt{\sum_{\substack{1\leq i\leq n\\1\leq j^\prime\leq n}}\Bigg(\Big\{\bb E \big[(\mb W_l\mb X)^{\circ3}\mb X^*\big]^2_{i,j^\prime}\Big\}\Big\{\bb E\indicator{\{\norm{\mb X}{\infty}>B\}}\Big\} \Bigg)}= \norm{\bb E\big[(\mb W_l\mb X)^{\circ3}\mb X^*\big]}{F}\cdot\sqrt{\bb E\indicator{\{\norm{\mb X}{\infty}>B\}}}.
        \end{split}
    \end{equation}
    Note that in \eqref{eq:GradHatfExpCalculation} of Lemma \ref{proof:L4MSPExp}, we know that \begin{equation}
        \bb E\big[(\mb W_l\mb X)^{\circ3}\mb X^*\big] = 3p\theta(1-\theta)\mb W_l^{\circ3}+3p\theta^2\mb W_l,
    \end{equation}
    hence
    \begin{equation}
        \label{eq:BoundGradExp}
        \begin{split}
            \norm{\bb E\big[(\mb W_l\mb X)^{\circ3}\mb X^*\big]}{F} = & \norm{3p\theta(1-\theta)\mb W_l^{\circ3}+3p\theta^2\mb W_l}{F}\leq 3p\theta(1-\theta)\norm{\mb W_l^{\circ3}}{F} + 3p\theta^2\norm{\mb W_l}{F}\\
            \leq & 3p\theta(1-\theta)\norm{\mb W_l}{F}^3+3p\theta^2\norm{\mb W_l}{F} \quad \text{(By inequality 1 in Lemma \ref{lemma:UsefulInequalities})}\\
            = & 3p\theta(1-\theta)n^{\frac{3}{2}}+3p\theta^2n^{\frac{1}{2}}<4n^{\frac{3}{2}}p\theta\quad \text{($n$ is a large number)}.
        \end{split}
    \end{equation}
    Moreover, Lemma \ref{lemma:BGMatrixTruncation} shows that 
    \begin{equation}
        \label{eq:GradIndicatorBound}
        \bb E\indicator{\{\norm{\mb X}{\infty}>B\}} \leq 2np\theta^{-\frac{B^2}{2}}.
    \end{equation}
    Substitute \eqref{eq:GradIndicatorBound} and \eqref{eq:BoundGradExp} into \eqref{eq:BoundGradMatrixTruncationEXP}, yield
    \begin{equation}
        \frac{1}{np}\norm{\bb E\big[(\mb W_l\mb X)^{\circ3}\mb X^*\big]-\bb E\big[(\mb W_l\bar{\mb X})^{\circ3}\bar{\mb X}^*\big]}{F}\leq 4\sqrt{2}np^{\frac{1}{2}}\theta^{\frac{3}{2}}e^{-\frac{B^2}{4}}.
    \end{equation}
    Hence, when 
    \begin{equation}
        \label{eq:GradBoundB}
        B\geq 2\sqrt{\ln\Bigg(\frac{16\sqrt{2}np^{\frac{1}{2}}\theta^{\frac{3}{2}}}{\delta}\Bigg)},
    \end{equation}
    we have 
    \begin{equation}
        \label{eq:GradXAndBarXBound}
        \frac{1}{np}\norm{\bb E\big[(\mb W_l\bar{\mb X})^{\circ3}\bar{\mb X}^*\big]-\bb E\big[(\mb W_l\mb X)^{\circ3}\mb X^*\big]}{F}\leq 4\sqrt{2}np^{\frac{1}{2}}\theta^{\frac{3}{2}}e^{-\frac{B^2}{4}}\leq \frac{\delta}{4}.
    \end{equation}
    Therefore, combine \eqref{eq:GradProbabilityTruncationEpsNet}, we have
    \begin{equation}
        \label{eq:GradProbabilityNormalToTruncationEXP1}
        \begin{split}
            & \bb P\Bigg(\frac{1}{np}\norm{(\mb W\bar{\mb X})^{\circ3}\bar{\mb X}^*-\bb E[(\mb W\mb X)^{\circ3}\mb X^*]}{F}\leq\delta\Bigg)\\
            \leq& \bb P\Bigg(\frac{1}{np}\norm{(\mb W_l\bar{\mb X})^{\circ3}\bar{\mb X}^*-\bb E[(\mb W_l\mb X)^{\circ3}\mb X^*]}{F}\geq\frac{\delta}{2}\Bigg)\\
            \leq &\bb P\Bigg(\frac{1}{np}\norm{(\mb W\bar{\mb X})^{\circ3}\bar{\mb X}^*-\bb E\big[(\mb W_l\bar{\mb X})^{\circ3}\bar{\mb X}^*\big]}{F}+\frac{1}{np}\norm{\bb E\big[(\mb W_l\bar{\mb X})^{\circ3}\bar{\mb X}^*\big]-\bb E\big[(\mb W_l\mb X)^{\circ3}\mb X^*\big]}{F}\geq\frac{\delta}{2}\Bigg)\\
            \leq &\bb P\Bigg(\frac{1}{np}\norm{(\mb W\bar{\mb X})^{\circ3}\bar{\mb X}^*-\bb E\big[(\mb W_l\bar{\mb X})^{\circ3}\bar{\mb X}^*\big]}{F}\geq\frac{\delta}{4}\Bigg)\quad\text{(By \eqref{eq:GradXAndBarXBound})}\\
            \leq & n^2 \bb P \Bigg(\abs{\Big[(\mb W_l\bar{\mb X})^{\circ3}\bar{\mb X}^*\Big]_{i,j^\prime}-\bb E\Big[(\mb W_l\bar{\mb X})^{\circ3}\bar{\mb X}^*\Big]_{i,j^\prime}}\geq p\cdot\frac{\delta}{4} \Bigg)\quad\text{(By union bound, $\forall i,j^\prime \in [n]$)}.
        \end{split}
    \end{equation}
    \textbf{Point-wise Bernstein's Inequality.} Next, we apply Bernstein's inequality on $\big[(\mb W_l\bar{\mb X})^{\circ3}\bar{\mb X}^*\big]_{i,j^\prime}$. Note that for each $i,j^\prime\in[n]$ 
    \begin{equation}
        \big[(\mb W_l\bar{\mb X})^{\circ3}\bar{\mb X}^*\big]_{i,j^\prime}=\sum_{j=1}^p\Bigg[\bar{x}_{j^\prime,j}\Big(\sum_{k=1}^nw_{i,k}\bar{x}_{k,j}\Big)^3\Bigg],
    \end{equation}
    can be viewed as sum of $p$ independent variables
    \begin{equation}
        \bar{z}_j = \bar{x}_{j^\prime,j}\Big(\sum_{k=1}^nw_{i,j}\bar{x}_{k,j}\Big)^3=\bar{x}_{j^\prime,j}\innerprod{\mb w_i}{\bar{\mb x}_j}^3,\quad \forall j\in [p],
    \end{equation}
    where $\mb w_i$ is the $i^{\text{th}}$ row vector of $\mb W$. Note that each $\bar{z}_j$ are bounded by
    \begin{equation}
        \label{eq:GradElementWiseBound}
        \begin{split}
            \abs{\bar{x}_{j^\prime,j}\innerprod{\mb w_i}{\bar{\mb x}_j}^3}\leq B\abs{\innerprod{\mb w_i}{\bar{\mb x}_j}^3}= B\norm{\bar{\mb x}_j}{2}^3\abs{\innerprod{\mb w_i}{\frac{\bar{\mb x}_j}{\norm{\bar{\mb x}_j}{2}}}} \leq B\cdot (nB^2)^{\frac{3}{2}}=n^{\frac{3}{2}}B^4, 
        \end{split}
    \end{equation}
    also for each $\bar{x}_{j^\prime,j}\Big(\sum_{k=1}^nw_{i,j}\bar{x}_{k,j}\Big)^3$, we have
    \begin{equation}
        \label{eq:GradElementWise2Moment}
        \begin{split}
            &\bb E\Big[x_{j^\prime,j}^2\Big(\sum_{k=1}^nw_{i,j}x_{k,j}\Big)^6\Big]-\bb E\Big[\bar{x}_{j^\prime,j}^2\Big(\sum_{k=1}^nw_{i,j}\bar{x}_{k,j}\Big)^6\Big]\\
            = &\bb E\Big[x_{j^\prime,j}^2\Big(\sum_{k=1}^nw_{i,j}x_{k,j}\Big)^6-\bar{x}_{j^\prime,j}^2\Big(\sum_{k=1}^nw_{i,j}\bar{x}_{k,j}\Big)^6\Big]\\
            = & \bb E\Big[x_{j^\prime,j}^2\Big(\sum_{k=1}^nw_{i,j}x_{k,j}\Big)^6\cdot \indicator{\{\norm{\mb X}{\infty}>B\}}\Big]\geq 0,
        \end{split}
    \end{equation}
    and \eqref{eq:ElementwiseGrad2Moment} in Lemma \ref{lemma:GradHatfSecondMoment} shows that
    \begin{equation}
        \bb E\Big[\bar{x}_{j^\prime,j}^2\Big(\sum_{k=1}^nw_{i,j}\bar{x}_{k,j}\Big)^6\Big] \leq \bb E\Big[x_{j^\prime,j}^2\Big(\sum_{k=1}^nw_{i,j}x_{k,j}\Big)^6\Big] \leq C,
    \end{equation}
    for a constant $C>177$. Hence, by Bernstein's inequality, $\forall i,j^\prime \in [n]$, we have 
    \begin{equation}
        \label{eq:GradProbabilityNormalToTruncationEXP2}
        \begin{split}
            &\bb P \Bigg(\abs{\Big[(\mb W_l\bar{\mb X})^{\circ3}\bar{\mb X}^*\Big]_{i,j^\prime}-\bb E\Big[(\mb W_l\bar{\mb X})^{\circ3}\bar{\mb X}^*\Big]_{i,j^\prime}}\geq p\cdot\frac{\delta}{4} \Bigg)\\
            =&\bb P\Bigg( \abs{\sum_{j=1}^p\Big[\bar{x}_{j^\prime,j}(\sum_{k=1}^nw_{i,k}\bar{x}_{k,j})^3\Big]-\sum_{j=1}^p\bb E\Big[\bar{x}_{j^\prime,j}(\sum_{k=1}^nw_{i,k}\bar{x}_{k,j})^3\Big]}\geq p\cdot \frac{\delta}{4}\Bigg) \\
            = & \bb P\Bigg( \sum_{j=1}^p\Big[\bar{x}_{j^\prime,j}(\sum_{k=1}^nw_{i,k}\bar{x}_{k,j})^3\Big]-\sum_{j=1}^p\bb E\Big[\bar{x}_{j^\prime,j}(\sum_{k=1}^nw_{i,k}\bar{x}_{k,j})^3\Big]\geq p\cdot \frac{\delta}{4}\Bigg)\\
            &+\bb P\Bigg( \sum_{j=1}^p\Big[\bar{x}_{j^\prime,j}(\sum_{k=1}^nw_{i,k}\bar{x}_{k,j})^3\Big]-\sum_{j=1}^p\bb E\Big[\bar{x}_{j^\prime,j}(\sum_{k=1}^nw_{i,k}\bar{x}_{k,j})^3\Big]\leq -p\cdot \frac{\delta}{4}\Bigg)\\
            \leq & 2 \exp\Bigg(-\frac{p\delta^2/16}{2\big[\frac{1}{p}\sum_{j=1}^p\bar{x}_{j^\prime,j}^2(\sum_{k=1}^nw_{i,j}\bar{x}_{k,j})^6+\frac{ n^{\frac{3}{2}}B^4\delta}{12}\big]}\Bigg)\quad \text{(By \eqref{eq:GradElementWiseBound})}\\
            \leq & 2 \exp\Bigg(-\frac{p\delta^2}{\frac{96}{p}\sum_{j=1}^pC\theta+8\delta n^{\frac{3}{2}}B^4\delta}\Bigg)\quad \text{(By \eqref{eq:GradElementWise2Moment})}\\
            =&2\exp\Bigg(-\frac{3p\delta^2}{c_1\theta+8n^{\frac{3}{2}}B^4\delta}\Bigg),
        \end{split}
    \end{equation}
    for a constant $c_1>1.7\times 10^4$. Combine \eqref{eq:GradProbabilityNormalToTruncationEXP1}, we have
    \begin{equation}
        \label{eq:GradProbabilityNormalToTruncationEXP3}
        \begin{split}
            &\bb P\Bigg(\frac{1}{np}\norm{(\mb W\bar{\mb X})^{\circ3}\bar{\mb X}^*-\bb E[(\mb W\mb X)^{\circ3}\mb X^*]}{F}\leq\delta\Bigg)\\
            \leq& n^2 \bb P \Bigg(\abs{\Big[(\mb W_l\bar{\mb X})^{\circ3}\bar{\mb X}^*\Big]_{i,j^\prime}-\bb E\Big[(\mb W_l\bar{\mb X})^{\circ3}\bar{\mb X}^*\Big]_{i,j^\prime}}\geq p\cdot\frac{\delta}{4} \Bigg).
        \end{split}
    \end{equation}
    \textbf{Union Bound.} Now, we will give a union bound for 
    \begin{equation}
        \bb P\Bigg(\sup_{\mb W\in \msf O(n;\bb R)}\frac{1}{np}\norm{(\mb {W}\bar{\mb X})^{\circ3}\bar{\mb X}^*-\bb E\big[(\mb {WX})^{\circ3}\mb X^*\big]}{F}>\delta\Bigg).
    \end{equation}
    Notice that 
    \begin{equation}
        \label{eq:GradUnionConcentrationBoundFinal}
        \begin{split}
            & \bb P\Bigg(\sup_{\mb W\in \msf O(n;\bb R)}\frac{1}{np}\norm{(\mb {W}\bar{\mb X})^{\circ3}\bar{\mb X}^*-\bb E\big[(\mb {WX})^{\circ3}\mb X^*\big]}{F}>\delta\Bigg)\\
            \leq & \sum_{l=1}^{|\mc S_\eps|}\bb P\Bigg(\sup_{\mb W\in \bb B(\mb W_l,\eps)}\frac{1}{np}\norm{(\mb {W}\bar{\mb X})^{\circ3}\bar{\mb X}^*-\bb E\big[(\mb {WX})^{\circ3}\mb X^*\big]}{F}>\delta\Bigg)\quad \text{(By $\eps-$covering in  \eqref{eq:GradEpsBallCoveringOrthogonalGroup})}\\
            \leq & \sum_{l=1}^{|\mc S_\eps|}n^2 \bb P \Bigg(\abs{\Big[(\mb W_l\bar{\mb X})^{\circ3}\bar{\mb X}^*\Big]_{i,j^\prime}-\bb E\Big[(\mb W_l\bar{\mb X})^{\circ3}\bar{\mb X}^*\Big]_{i,j^\prime}}\geq p\cdot\frac{\delta}{4} \Bigg)\quad \text{(By \eqref{eq:GradProbabilityNormalToTruncationEXP3})}\\
            \leq & \sum_{l=1}^{|\mc S_\eps|}\Bigg[2n^2\exp\Bigg(-\frac{3p\delta^2}{c_1\theta+8n^{\frac{3}{2}}B^4\delta}\Bigg)\Bigg]\quad \text{(By \eqref{eq:GradProbabilityNormalToTruncationEXP2})}\\
            \leq & \Big(\frac{6}{\eps}\Big)^{n^2}\Bigg[2n^2\exp\Bigg(-\frac{3p\delta^2}{c_1\theta+8n^{\frac{3}{2}}B^4\delta}\Bigg)\Bigg] \quad \text{(By \eqref{eq:GradEpsBallMaxNumber})}\\
            = & \exp\Bigg(n^2\ln \Big(\frac{48npB^4}{\delta}\Big)\Bigg)\Bigg[2n^2\exp\Bigg(-\frac{3p\delta^2}{c_1\theta+8n^{\frac{3}{2}}B^4\delta}\Bigg)\Bigg] \quad \text{(Let $\eps = \frac{\delta}{8npB^4}$)}\\
            =&2n^2\exp\Bigg(-\frac{3p\delta^2}{c_1\theta+8n^{\frac{3}{2}}B^4\delta}+n^2\ln\Big(\frac{48npB^4}{\delta}\Big)\Bigg),
        \end{split}
    \end{equation}
    for a constant $c_1>1.4\times 10^4$. Note that \eqref{eq:GradBoundB} requires a lower bound on $B$, here we can choose $B = \ln p$, which satisfies \eqref{eq:GradBoundB} when $p$ is large enough (say $p = \Omega(n)$). Combine \eqref{eq:GradUnionBoundWithAndWithoutTruncation} and substitute $B = \ln p$, we have 
    \begin{equation}
        \begin{split}
            &\bb P\Bigg(\sup_{\mb W\in \msf O(n;\bb R)}\frac{1}{np}\norm{(\mb {WX})^{\circ3}\mb X^*-\bb E\big[(\mb {WX})^{\circ3}\mb X^*\big]}{F}>\delta\Bigg)\\
            \leq & \bb P\Bigg(\sup_{\mb W\in \msf O(n;\bb R)}\frac{1}{np}\norm{(\mb {W}\bar{\mb X})^{\circ3}\bar{\mb X}^*-\bb E\big[(\mb {WX})^{\circ3}\mb X^*\big]}{F}>\delta\Bigg) + 2np\theta e^{-\frac{B^2}{2}}\\
            \leq & 2n^2\exp\Bigg(-\frac{3p\delta^2}{c_1\theta+8n^{\frac{3}{2}}(\ln p)^4\delta}+n^2\ln\Big(\frac{48np(\ln p)^4}{\delta}\Big)\Bigg) + 2np\theta \exp\Bigg(-\frac{(\ln p)^2}{2}\Bigg),
        \end{split}
    \end{equation}
    for a constant $c_1>1.7\times 10^4$, which completes the proof.
\end{proof}

\subsection{Perturbation Bound for Unitary Polar Factor}

\begin{lemma}[Perturbation Bound for Unitary Polar Factor]
\label{lemma:PolarPerturbation}
Let $\mb A_1,\mb A_2\in \bb R^{n\times n}$ be two nonsigular matrices, and let $\mb Q_1 = \mb U_1\mb V_1^*$, $\mb Q_2 = \mb U_2\mb V_2^*$, where 
\begin{equation*}
    \mb U_1\mb \Sigma_1 \mb V_1^* = \text{SVD}(\mb A_1),\quad \mb U_2\mb \Sigma_2 \mb V_2^* = \text{SVD}(\mb A_2).
\end{equation*}
Let $\sigma_n(\mb A_1)$, $\sigma_n(\mb A_2)$ denote the smallest singular value of $\mb A_1$,$\mb A_2$ respectively. Then, for any unitary invariant norm $\norm{\cdot}{\diamond}$ we have
\begin{equation}
    \norm{\mb Q_1-\mb Q_2}{\diamond}\leq \frac{2}{\sigma_n(\mb A_1)+\sigma_n(\mb A_2)}\norm{\mb A_1-\mb A_2}{\diamond}.
\end{equation}
\end{lemma}
\begin{proof}
    \label{proof:PolarPerturbation}
    See theorem 1 in \cite{li1995new}.
\end{proof}

\subsection{Perturbation Bound for Singular Values}

\begin{lemma}[Singular Value Perturbation]
\label{lemma:SingularValuePerturbation}
Let $\mb G = \mb B \mb M$ be a general full rank matrix, where $\mb M$ ($M_{i,i}$ equals to the $\ell^2-$norm of the $i^\text{th}$ column of $\mb G$) is a chosen diagonal matrix so $\mb B$ has unit matrix 2 norm $(\norm{\mb B}{2}=1)$. Let $\delta \mb G= \delta\mb B \mb M$ be a perturbation of $\mb G$ such that $\norm{\delta \mb B}{2}\leq \sigma_{\min} (\mb B)$ and $\sigma_i$ and $\sigma_i^\prime$ be the $i^{\text{th}}$ singular value of $\mb G$ and $\mb G + \delta \mb G$ respectively. Then
\begin{equation}
    \frac{\abs{\sigma_i-\sigma_i^\prime}}{\sigma_i}\leq \frac{\norm{\delta\mb B}{2}}{\sigma_{\min}(\mb B)}.
\end{equation}
\end{lemma}
\begin{proof}
    \label{proof:SingularValuePerturbation}
    See theorem 2.17 in \cite{demmel1992jacobi}.
\end{proof}

\addcontentsline{toc}{section}{References}
\bibliography{reference.bib}

\end{document}